\documentclass[11pt]{article}

\usepackage{geometry}
\usepackage{graphicx} % Required for inserting images
\usepackage[dvipsnames]{xcolor}
\usepackage{amssymb}
\usepackage{amsmath}
\usepackage{amsthm}
\usepackage{physics}
\usepackage{hyperref}
\usepackage{color,soul}
\usepackage{enumitem}
\usepackage{natbib}

\DeclareMathOperator*{\argmin}{arg\,min}

\title{Distribution-dependent Generalization Bounds for Tuning Linear Regression Across Tasks}
\author{
Maria-Florina Balcan\thanks{School of Computer Science, Carnegie Mellon University.}
\hspace{1cm}
Saumya Goyal\thanks{Machine Learning Department, Carnegie Mellon University.}
\hspace{1cm}
Dravyansh Sharma\thanks{Toyota Technological Institute at Chicago;
Northwestern University.}
}
\date{}

\begin{document}
\newcommand{\R}{\mathbb{R}}
\newcommand{\F}{\mathcal{F}}
\newcommand{\cE}{\mathcal{E}}
\newcommand{\Dx}{\mathcal{D}_X}
\newcommand{\Dyx}{\mathcal{D}_{Y|X}}
\newcommand{\E}[2]{\mathbb{E}_{#1}\left[#2\right]}
\newcommand{\Eml}[1]{\mathbb{E}_{#1}}
\newcommand{\wh}{\hat{w}}
\newcommand{\bh}{\hat{\beta}}
\newcommand{\whp}{\hat{w}'}
\newcommand{\ws}[1]{w^{*#1}}
\newcommand{\el}[1]{\tilde{#1}}
\newcommand{\etr}[1]{\tilde{#1}_{tr}}
\renewcommand{\eval}[1]{\tilde{#1}_{val}}
\newcommand{\lv}{l_v}
\newcommand{\lve}{\tilde{l}_v}
\newcommand{\lev}{l_{ev}}
\newcommand{\muh}{\hat{\mu}}
\newcommand{\lambmu}{{(\lambda,\mu)}}
\newcommand{\lambmuh}[1]{{(\lambda_{#1},\hat{\mu})}}
\newcommand{\lambmus}[1]{{(\lambda_{#1},\mu^*)}}
\newcommand{\ewtimes}{\times_{ew}}
\newtheorem{theorem}{Theorem}[section]
\newtheorem{lemma}[theorem]{Lemma}
\newtheorem{proposition}[theorem]{Proposition}
\newtheorem{remark}[theorem]{Remark}
\newtheorem{corollary}{Corollary}[theorem]
\theoremstyle{definition}
\newtheorem{definition}{Definition}
\newtheorem{assumption}{Assumption}
\DeclareRobustCommand{\hlcyan}[1]{{\sethlcolor{cyan}\hl{#1}}}

\maketitle
\begin{abstract}
Modern regression problems often involve high-dimensional data and a careful tuning of the regularization hyperparameters is crucial to avoid overly complex models that may overfit the training data while guaranteeing desirable properties like effective variable selection.
We study the recently introduced direction of tuning regularization hyperparameters in linear regression across multiple related tasks. We obtain distribution-dependent bounds on the generalization error for the validation loss when tuning the L1 and L2 coefficients, including ridge, lasso and the elastic net. In contrast, prior work develops bounds that apply uniformly to all distributions, but such bounds necessarily degrade with feature  dimension, $d$. While these bounds are shown to be tight for worst-case distributions, our bounds improve with the ``niceness'' of the data distribution. 
Concretely, we show that under additional assumptions that instances within each task are i.i.d.\ draws from broad well-studied classes of distributions including sub-Gaussians, our generalization bounds do not get worse with increasing $d$, and are much sharper than prior work for very large $d$. We also extend our results to a generalization of ridge regression, where we achieve tighter bounds that take into account an estimate of the mean of the ground truth distribution.
\end{abstract}

\section{Introduction}\label{sec:intro}
Hyperparameter tuning is a common problem in machine learning that typically involves a lot of experimentation and domain expertise, and commonly used approaches lack formal optimality guarantees. %resulting in long development cycles. 
In this work, we study hyperparameter tuning in regularized linear regression, which is a popular technique used in various applications. 
For a linear regression problem with $n$ inputs in $d$ dimensions arranged in an input matrix, $X\in \mathcal{X}^n\subseteq \R^{d \times n}$, and output vector $y \in \R^n$, a regularized least squares estimator is given by 
$\wh=\argmin_w \|X^\intercal w-y\|^2 + r(\lambda, w)$. Here $r(\lambda, w)$ can take several forms, including the L2 regularization for ridge regression \citep{hoerl1970ridge,tikhonov1977solutions}, $\wh_\lambda=\argmin_w \|X^\intercal w-y\|^2 + \lambda\|w\|^2$ and the elastic net, $\wh_{\lambda_1, \lambda_2}=\argmin_w \|X^\intercal w-y\|^2 + \lambda_1\|w\|_1 + \lambda_2\|w\|_2^2$~\citep{hastie2009elements}. 

Determining a good regularization coefficient $\lambda$  constitutes finding a balance between avoiding overfitting, allowing good generalization and variable selection. Popular methods for tuning hyperparameters involve finding the best parameter from a discrete set of values, also known as grid-search. These approaches either fail to give theoretical guarantees on optimality in the continuous space, or require strong data-dependent assumptions (see~\citealt{enet_2022} for a discussion).  
We further note that, there does not exist an optimal value of the regularization coefficients for a single task. In other words, for a fixed value of the regularization coefficients, there exists a ground truth value of $w$ and sampling of the training and validation tasks, such that the validation error is arbitrarily high (Appendix \ref{sec:single_task}). On the other hand, we can guarantee the optimality of regularization coefficients on average over a multi-task under certain circumstances (Appendix \ref{sec:equivalence}). 
Our work thus involves a data-driven approach to tuning the regularization hyperparameters in ridge regression, lasso and the elastic net which interpolates the two. We assume access to a set of related 
linear regression tasks where each task is assumed to be sampled similarly, that is, all inputs are sampled from the same distribution, and all ground truth functions are assumed to be sampled from the same distribution across tasks. We formalize this notion in Section \ref{sec:notation}.
This makes our setting similar to multi-task learning, since previously seen tasks inform the procedure for future unknown tasks. 

We study finding $\lambda$ by computing the Expected Risk Minimizer (ERM) estimate of $\lambda$ that minimizes the expected test error, estimated using given validation data for each task. 
Prior work on data-driven tuning of regularization hyperparameters for linear regression~\citep{enet_2022, regression_pdim} provides {\it distribution-independent} generalization bounds for the ERM that apply to worst-case distributions. 
Contrary to prior work, we give distribution-dependent generalization bounds for learning the regularization hyperparameters, assuming i.i.d.\ samples within each task. We show that, depending on the ``niceness" of the distribution, our bounds are much tighter than the worst-case bounds obtained in prior work when the feature dimension $d$ is large.  

In fact, much of the work in data-driven algorithm design (see Appendix \ref{sec:add_rl}) has focused on data-independent guarantees. Technically, the primary approach has been to bound the pseudo-dimension which implies generalization guarantees for worst-case distributions. Some prior work has given bounds on the Rademacher complexity for tuning parameters in data-driven algorithm design~\citep{balcan2018dispersion,balcan2018data,sharma2025offline,du2025tuning}, but there is no clear evidence of the advantage over data-independent techniques. Prior work has shown improved distribution-dependent guarantees in a few other contexts like pricing problems~\citep{balcan2018general,balcan2025generalization}. Here is it shown that generalization independent of the number of items is made possible via distribution-dependent techniques but this is not achievable using distribution-independent bounds. Our results involve a conceptually similar conclusion for tuning  regularization in linear regression  in terms of the feature dimension $d$, but using new and completely different techniques.

\paragraph{Summary of contributions.}%\label{sec:overview}
Our key results are summarized as follows:
\begin{itemize}[leftmargin=*,topsep=0pt,partopsep=1ex,parsep=1ex]\itemsep=-4pt
    \item We provide generalization guarantees for tuning the regularization parameter in ridge regression in Theorem \ref{thm:lips_tgt}. We show that the error term can be broken into an error induced from a finite sampling of validation examples, and from a finite sampling of tasks. We show how to bound both of these %terms 
    in terms of Rademacher complexities, and compute upper bounds on the Rademacher complexities. This differs from techniques in prior work \citep{mtl_maurer} that lead to looser bounds (Appendix \ref{sec:alt_bounds}).
    We also consider the special case assuming well-specified linear maps in Theorem \ref{thm:lips_ws}.
    We  show that our data-dependent bounds are 
    tighter 
    than the previously best known bounds from \citet{regression_pdim}  (Section \ref{sec:ridge}). 
    \item In Section \ref{sec:en}, we give distribution-dependent generalization error bounds for tuning the L1-penalty (lasso) as well as for tuning the L1 and L2 penalties simultaneously (the elastic net). The analysis extends our technique for ridge regression, by applying it to the piecewise structured solution of lasso and the elastic net. We show that our bounds are much tighter than worst-case bounds from prior work for data drawn according to the well-studied sub-Gaussian distribution. Roughly speaking, for number of training examples $n=\tilde{\Omega}(d+\log T)$, we show that the generalization error is at most $\tilde{O}\left(\frac{1}{\sqrt{nT}}+\sqrt{\frac{\log 1/\delta}{T}}\right)$, compared to the  $\tilde{O}\left(\frac{\sqrt{d}}{\sqrt{T}}\right)$ distribution-independent upper bound shown by \citet{regression_pdim} (which they show  cannot be improved for worst-case distributions).
    \item We propose a  generalized version of ridge regression, which we call the Re-centered Ridge Regression in Section \ref{sec:offset}, where the L2-norm penalty is measured w.r.t.\ to a parameter $\mu$ instead of the origin. We derive generalization bounds for this estimator in Theorem \ref{thm:lips_recentered} and show that they are tighter than the bounds derived in Section \ref{sec:ridge} depending on the error of a given estimate $\muh$ of the optimal value of the parameter, $\mu^*$.
\end{itemize}

\textbf{Technical Novelty.} We present several novel theoretical techniques for proving our bounds in addition to improving previously known bounds. % as summarized in Section \ref{sec:inf_results} and detailed in Sections \ref{sec:ridge} and \ref{sec:en}. 
%The technical novelties are summarized below, and detailed in the proofs in the Appendix.
\begin{enumerate}
[leftmargin=*,topsep=0pt,partopsep=1ex,parsep=1ex]\itemsep=-4pt
    \item Rademacher complexity bounds for linear regression in prior work tune the parameter $w$~\citep{trace_reg,lin_hyp_sets}, and thus the function class over which a supremum is taken involves varying $w$ over a finite radius ball in some space. Our problem is significantly different, where we tune regularization hyperparameters, which influence the value of $w$ depending on training data. There is no prior work on tuning regularization parameters with data-dependent guarantees using Rademacher complexity, and our problem needs fundamentally different and new analytical techniques than tuning $w$.\looseness-1
    \item Rademacher complexity bounds for multi-task learning have been discussed before such as in \citet{mtl_maurer}, but using existing proof techniques give weaker bounds as we show in Appendix \ref{sec:alt_bounds}. We introduce a new way to analyse the Rademacher complexity by “breaking the error term into an error induced from a finite sampling of validation examples, and from a finite sampling of tasks” instead of “breaking the error into an error induced from imperfect estimation of expected validation error (due to finiteness of validation data), and error from imperfect estimation of  due to finiteness of the number of tasks”, and show that it leads to better asymptotic bounds.\looseness-1
    \item We additionally introduce a new method of bounding Rademacher complexity by using Lipschitzness of the function class in another function class in Definition \ref{def:lips_general} and Theorem \ref{thm:rad_lips_general} which is used in Lemmas \ref{lem:S_v_2} and \ref{lem:S_v_1}.
    \item Lastly, we introduce a method that extends easily to different loss functions instead of the explicit structure demanded by the analysis of \citet{enet_2022, regression_pdim}. This is evident since we only require Lipschitzness and boundedness of the loss function, instead of requiring squared loss. Further, in Appendix \ref{sec:offset}, we prove bounds for generalisations of ridge regression.
\end{enumerate}

\subsection{Informal results and key insights}
We present informal versions of our main results in this Section. We denote the expected validation loss (on a future unknown task) by $l_v$, and denote the ERM parameters and the optimal values of the parameters by $\lambda_{ERM}$ and $\lambda^*$ respectively. These and other notation are described in detail in Section \ref{sec:notation}.
\begin{theorem}[Informal Theorem \ref{thm:lips_tgt}]\label{thm:lips_tgt_inf}
    Assume a set of $T$ tasks sampled from the same (unknown) distribution given as quadruples of training and validation data $(X^t,y^t,X_v^t,y_v^t)_{t\in[T]}$, where each sample within each task is drawn i.i.d. Further assume that we have a bounded and $L$-Lipschitz validation loss function $l$. With  probability $1-\delta$, the ERM estimator for validation loss satisfies,
    \begin{align}
    l_v(\lambda_{ERM}) - l_v(\lambda^*) \leq \frac{2ML\Lambda_D^T
    % \sqrt{\E{X,y}{\|Xy\|^2}}
    }{\sqrt{T}}\E{x_v}{\|x_v\|} + \tilde{O}\left(\frac{\sqrt{\ln(T/\delta)}}{\sqrt{T}}\right)\nonumber.
    \end{align}
    Here $M = \max\|Xy\|^2$, $\Lambda_D^T = \E{}{\max_t 1/V(X^tX^{t\intercal})}$, and $V(\cdot)$ denotes the smallest non-zero singular value of a matrix. We denote a single validation example by $x_v$.
\end{theorem}
Intuitively, the leading term is the dominant error term that depends on $\Lambda_D^T$. While it is non-trivial to compute $\Lambda_D^T$ for arbitrary distributions, we show that $\Lambda_D^T = O(\frac{d}{n} T^{2/d})$ for a very general class of distributions where each entry of each input $x$ is sampled independently from a distribution with a bounded probability density function. We note the following key insights from Section \ref{sec:ridge}.
\begin{itemize}[leftmargin=*,topsep=0pt,partopsep=1ex,parsep=1ex]\itemsep=-4pt
    \item For well-specified problems (as defined in Section \ref{sec:notation}), we are able to reduce our bounds to $l_v(\lambda_{ERM}) - l_v(\lambda^*) = O\left(\frac{1}{\sqrt{T}}(T^{2/d} + \sqrt{\log (T/\delta)})\right)$\footnote{Note that our bounds can be combined with results in prior work to give a bound on the generalization error in the above as  $l_v(\lambda_{ERM}) - l_v(\lambda^*) = O\left(\min\left\{\frac{1}{\sqrt{T}}(T^{2/d} + \sqrt{\log (T/\delta)}), \frac{\sqrt{\log d + \log(1/\delta)}}{\sqrt{T}}\right\}\right)$. So for most of our discussions we focus on giving examples and regimes where the new bounds developed in this work are sharper.\label{footnote:combined-bounds}} when $n \geq 6d$, for a general class of distributions where each entry of each input $x$ is sampled independently. The tightest known bound 
    for elastic net using
    squared loss functions from the literature is $O\left(\frac{\sqrt{d + \log(1/\delta)}}{\sqrt{T}}\right)$ from \cite{regression_pdim}. We also extend the distribution independent analysis of \cite{regression_pdim} for ridge regression using ideas from \cite{enet_2022} and \cite{bartlett2022generalization} to derive a bound of $O\left(\frac{\sqrt{\log d + \log(1/\delta)}}{\sqrt{T}}\right)$ in Appendix \ref{app:ridge}. Our bounds are better than the distribution-independent bounds proven in Appendix \ref{app:ridge} when $d=\Omega(T)$, although under the additional assumption that examples within each task are i.i.d. We also note  our bounds are better than the previously published bounds, specifically in \cite{regression_pdim}, for a larger regime $d=\Omega\left(\frac{\log T}{\log\log T}\right)$, because the previous distribution-independent bounds are weaker\footnote{Please note that the bounds in \cite{regression_pdim} applies to a larger class of problems beyond ridge regression.} than the distribution-independent bounds for ridge regression that we establish in Appendix \ref{app:ridge}.

    \item Our bounds suggest a way to determine a sufficient number of examples for training and validation for tuning ridge parameters: training examples reduce error from noise, while validation examples reduce error from variance in ground truth distribution. We explain this in more detail in Section \ref{sec:ridge}.\looseness-1
\end{itemize}

In Section \ref{sec:en}, we further establish generalization  bounds for tuning L1 and L2 penalties simultaneously in the elastic net under similar settings and assumptions. Unlike the ridge regression results, we additionally assume that the L2 coefficient $\lambda_2$ is bounded away from zero, which is a common assumption in prior work (e.g.~\citealt{enet_2022}). We get somewhat weaker generalization error bounds for elastic net than ridge regression under slightly stronger conditions, but for interesting ``nice'' distributions like sub-Gaussian data our elastic net bounds qualitatively match the ridge regression bounds.

\begin{theorem}[Informal Theorem \ref{thm:en}]\label{thm:informal-en}
    Consider the task of tuning $\lambda=(\lambda_1,\lambda_2)\in[0,\overline{\Lambda}_1]\times[\underline{\Lambda}_2,\infty)$. Assume a set of $T$ tasks sampled from the same (unknown) distribution given as quadruples of training and validation data $(X^t,y^t,X_v^t,y_v^t)_{t\in[T]}$, where each within-task sample is drawn i.i.d. Further assume that we have  a bounded and $L$-Lipschitz validation loss function $l$. With  probability $1-\delta$, the ERM estimator for validation loss satisfies,
    \begin{align}
    &l_v(\lambda_{ERM}) - l_v(\lambda^*) \leq \nonumber\\&\tilde{O}\left(\frac{L\overline{\Lambda}_1\sqrt{d\ln(T/\delta)}}{\sqrt{T}}\right)\left(\E{X}{\max_{t,\cE}  \frac{1}{V(X^t_\cE X_\cE^{t\intercal}) +\underline{\Lambda}_2 }}+\E{X,y}{ \max_{\cE} \frac{\|y\|\sqrt{V^*(X_\cE X_\cE^\intercal)}}{V^*(X_\cE X_\cE^\intercal)+\underline{\Lambda}_2}}\right)+O\left(\frac{\sqrt{\log \frac{1}{\delta}}}{\sqrt{T}}\right).\nonumber
    \end{align}
    Here 
    $V^*({A})$ is the non-zero singular value $\sigma_i$ of matrix $A$ that maximizes $\frac{\sqrt{\sigma_i(A)}}{\sigma_i(A)+\underline{\Lambda}_2}$, and $V(\cdot)$ is as in Theorem \ref{thm:lips_tgt_inf}. $X_{\cE}$ denotes the sub-matrix corresponding the subset $\cE$ of features (we take a maximum over all subsets). We have suppressed the dependence on norm of the validation examples and validation set size for simplicity here.
\end{theorem}

\noindent The above bound simplifies the more precise bounds in Theorem \ref{thm:en}, using somewhat looser upper bounds that could still be used to recover our asymptotic bounds for interesting cases (e.g.\ sub-Gaussian distributions). We have also suppressed some less interesting terms depending on $\|x_v\|$ for simplicity. The distribution-dependent terms in the bound involve expectation over the draw of the problem instances $(X,y)$ for the training set of tasks.
We further show in Proposition \ref{prop:en_isotropic_main} that for sub-Gaussian data distribution, the generalization error corresponding to the above bound is $\tilde{O}(1/\sqrt{nT}+\sqrt{\frac{\log 1/\delta}{T}})$ for sufficiently large $n \geq \Omega\left(d+\log\frac{T}{\underline{\Lambda}_2}\right)$, improving upon prior work~\citep{regression_pdim} that gives a bound of $O\left(\frac{\sqrt{d + \log(1/\delta)}}{\sqrt{T}}\right)$ which applies to worst-case distributions but has a polynomial dependence on the feature dimension $d$. Prior work, however, does not assume the samples within each task to be i.i.d.\ draws.

\subsection{Related work}
\textbf{Hyperparameter tuning for regularized linear regression.}  Several methods for tuning regularization parameters in linear regression have been suggested in the literature. Several of these approaches however, have been purely empirical with no theoretical guarantees~\citep{gibbons1981simulation}, or involve strong data-dependent assumptions~\citep{golub1979generalized}. A new line of work, proposed by \citet{enet_2022} seeks to find regularization parameters across several related tasks, as opposed to finding separate regularization parameters for each task. The best known bounds in this direction were given by \citet{regression_pdim}, where they use pseudo-dimension arguments to prove that $T = O(d/\epsilon^2)$ tasks are sufficient for learning up to an $\epsilon$ tolerance in the validation error, where $d$ is the feature space dimension. In this paper, we make the additional assumption that instances within each task are sampled i.i.d.\ and give data-dependent bounds that show a potentially tighter dependence of $T$ on $d$ depending on the data distribution. For example, as explained in Section \ref{sec:ridge}, we are able to get $T=O(1/\epsilon^{\frac{2d}{d-4}})$ dependence for a general class of distributions. We further note that our bounds provide additional insights into the error bound. While the bound in \citet{regression_pdim} was independent of the number of training and validation samples, our bounds decrease as the number of samples increase.

\textbf{Rademacher Complexity bounds for linear regression.} Using Rademacher complexities to show data-dependent generalization bounds for linear regression is well-studied in the literature \citep{mlt_book, trace_reg, lin_hyp_sets}. However, analyzing generalization error on multi-task learning is not common but has been done in some prior work \citep{trace_reg, mtl_maurer}. \citet{trace_reg} restrict their attention to finding regression parameters for a fixed set of tasks with a bounded trace norm on the matrix of ground truth parameters. In this paper we study finding regularization parameters for solving a future unknown task, and use some of their techniques to simplify computation of Rademacher complexities. \citet{mtl_maurer} discuss meta-learning optimal representations for learning for fixed, as well as unknown tasks using Gaussian complexities. Our approach of dividing generalization error into error from finite sampling of validation and tasks respectively is similar to their approach of dividing generalization error into error from learning from a representation and learning the representation respectively. 
Several tighter variants of Rademacher complexity such as the local Rademacher complexity \citep{localrad_bartlett} and offset Rademacher complexity \citep{offset_rad} have also been proposed in literature. It has been shown by several works that these techniques can possibly give tighter bounds than simple Rademacher complexities \citep{ebay_poisson_jana}. Analyzing our problem of finding regularization parameters through possibly tighter variants of Rademacher complexities remains an open question for future work. \cite{balcan2018dispersion} provide general  bounds on the Rademacher complexity based on certain dispersion parameters, which roughly correspond to smoothness of problem instances (similar to our assumptions in Proposition \ref{prop:ex_isotropic}), but their upper bounds for tuning regularized regression problems also degrade with $d$.

{Another related line of work studies multi-task learning for linear regression, but framed as an in-context learning problem for transformers~\citep{ahn2023transformers,zhang2024trained,wu2024many}. The assumptions on the tasks and examples within tasks needed for their theoretical results on sample complexity are typically stronger than our results. For example, Assumption 1 of \citet{wu2024many} states that the linear regression map $w$ in different tasks come from a Gaussian distribution, and the data vectors $(X^{(i)},y^{(i)})$ are i.i.d.\ draws from a Gaussian with the mean of $y$ depending on $w$.  We have results for general distributions (Theorems \ref{thm:lips_tgt}, \ref{thm:en}), as well as instantiations of our bounds for broader classes of distributions including bounded-density distributions (Proposition \ref{prop:ex_isotropic}) and sub-Gaussian distributions (Proposition \ref{prop:en_isotropic_main}). However, our bounds are not directly comparable as the goal is to learn different quantities from the multiple ``pre-training'' tasks. They learn a common $d\times d$ matrix $\Gamma$ using gradient descent which linearly maps $(X, y, X_v)$ for any unseen test task to predictions $y_v$. In contrast, we learn how to set the L1 and L2 penalties for predicting $y_v$ by regularized linear regression and give uniform convergence guarantees. Note that while their approach only achieves approximate Bayes optimality in certain restrictive regimes, we are always provably near Bayes-optimal.}

See Appendix \ref{sec:add_rl} for additional related work.

\section{Problem setting and notation}\label{sec:notation}
Throughout the paper, we will denote vectors by small case variables (e.g. $x$) and column-wise collection of vectors by large case variables (e.g. $X$).
We start with defining the typical linear regression setting, where each task is given with validation data as a quadruple $(X,y,X_v,y_v)$ of training and validation data. Here for each training input $x$, $x \in \mathcal{X} \subseteq \R^d$ and similarly for each validation input $x_v$, $x_v \in \mathcal{X} \subseteq \R^d$. We further denote the $i^{\text{th}}$ element of $X$ and $y$ as $X^{(i)}$ and $y^{(i)} $ respectively. 
We assume all training and validation examples are sampled i.i.d. (which is stronger than the assumptions of \cite{enet_2022, regression_pdim} where the tasks are assumed to be i.i.d.\ but  the examples within tasks may not be i.i.d.). We call a linear regression problem well-specified if the expected value of the output is a linear function of the input. This is a popular setting for linear regression studied in previous works such as \cite{offset_rad} and relevant in many practical situations. % like business forecasting, genetics and agriculture. 
Consequently, we denote a well-specified linear map by the feature vector $w\in\mathcal{W}\subseteq \R^d$ as: $f_w:\mathcal{X}\times E \rightarrow\R$ so that $f_w(x,\epsilon) = x^\intercal w + \epsilon$. Here $E\subseteq \R$ is the set of possible noise values that we can observe. 
We will denote the set of all well-specified linear maps by $\F_{ws} = \{f_w:w\in \mathcal{W}\}$. For the well-specified linear map setting, we will assume that for each task there exists $f_w\in \F_{ws}$ so that for any input $X^{(i)}$, there is $\epsilon^{(i)}$ such that $f_w(X^{(i)},\epsilon^{(i)}) = y^{(i)}$. 
We further assume that each training and validation input for each task is sampled from the same distribution denoted by $D_\mathcal{X}$. Thus, $x\sim D_\mathcal{X}$ and $X\sim D_\mathcal{X}^n$.
Similarly, we assume that all training and validation noise vectors for each task are sampled from the same distribution denoted by $D_E$, so that $\epsilon \sim D_E^n$. The ground truth feature vectors for each task are also assumed to be sampled i.i.d. from the distribution $D_{\mathcal{W}} $.
For notational convenience, we will denote an element wise operation on a collection of inputs as the function applied to the matrix of inputs. So for given $X,y$ there exists $\epsilon\in E^n, \text{s.t. } f(X,\epsilon) = y$. We will denote an ordered set of such tasks given with validation data (each with a possibly different input-output map) as a problem instance that we denote by $S$. Formally,\looseness-1
\begin{align}
    S = \{(X^t,y^t,X^t_{v}, y^t_{v})&: X^t\in \mathcal{X}^n, X^t_{v}\in \mathcal{X}^{n_v}, \exists w^{*t}\in \mathcal{W}, \epsilon^t \in E^n, \epsilon_v^t \in E^{n_v}\nonumber \\
    &\text{ s.t. } y^t = X^{t\intercal}w^{*t} + \epsilon^t, y_v^t = X_v^{t\intercal}w^{*t} + \epsilon_v^t, \forall t\in [T]\}.\nonumber
\end{align}
We denote different tasks using superscript. So if we have $T$ tasks, the training data will be denoted as $X^t\in \mathcal{X}^n$ and $y^t \in \R^n$ for $t\in [T]$ and validation data will be denoted as $X^t_{v}\in \mathcal{X}^{n_v}$ and $y^t_{v} \in \R^{n_v}$ for $t\in [T]$.

We also study a generalization of this setting. We denote the set of deterministic maps as $\F = \{f:\mathcal{X}\times E \rightarrow \R\}$ that takes an input in $\mathcal{X}\subseteq\R^d$ and random noise and returns the output. Here $E\subseteq\R^m$ is a possibly more general set of possible noise vectors. 
Similar to before, for given $X,y$ there exists $\epsilon\in E^n, \text{s.t. } f(X,\epsilon) = y$. We assume the ground truth map for each task is sampled i.i.d. from the distribution $D_{\mathcal{F}} $.
The problem instance in the general setting can then be denoted as:
\begin{align}
    S = \{(X^t,y^t,X^t_{v}, y^t_{v})&: X^t\in \mathcal{X}^n, X^t_{v}\in \mathcal{X}^{n_v}, \exists f^t\in \mathcal{F},\epsilon^t\in E^n, \epsilon_v^t\in E^{n_v},\nonumber\\
    &\text{s.t. } y^t=f^t(X^t, \epsilon^t), y_v^t=f^t(X_v^t,\epsilon_v^t)  \forall t\in [T]\}.\label{def:prob}
\end{align}

Assume we have an estimator as a function of $X,y$ that takes $\lambda$ as a hyperparameter. Denote this estimator as $\wh_\lambda(X,y)$. We define the empirical validation loss as:
\begin{align*}
    \lv(\lambda, S) = 
    \frac{1}{T} \sum_t \frac{1}{n_v} \sum_i l(X_{v}^{t(i)T}\wh_\lambda(X^t,y^t), y_{v}^{t(i)}).
\end{align*}
Intuitively, we compute the estimator for the given value of lambda for each training instance $(X^t,y^t)$. We then compute the empirical validation loss on each task using the respectively computed estimators, and average the loss across all tasks. For notational convenience, we will denote $\wh(X^t,y^t)$ as $\wh^t$ wherever obvious from context. 

The objective of finding hyperparameters in machine learning is often to minimize the expected validation loss given as $l_v(\lambda) = \E{S}{\lv(\lambda, S)}$\footnote{Note that $\E{A\sim D}{.}$ represents the expectation with respect to random variable $A$ when drawn from distribution $D$. In the subsequent parts of the paper, we will omit the distribution, and even the random variable when obvious from context.}. This is a popular setting studied in previous works such as \cite{regression_pdim}. % To formally define the expectation value, we assume distribution over variables $x,f,\epsilon$ as $x\sim D_\mathcal{X}, f\sim D_\mathcal{F}, \epsilon\sim D_\mathcal{E}$. We use $D$ to denote the joint distribution over all variables wherever appropriate. 
We can define the expected validation loss as:
\begin{align*}
    l_v(\lambda) = \E{X\sim D_\mathcal{X}^n, f\sim D_\mathcal{F}, \epsilon\sim D_E^n}{\E{x_v\sim D_\mathcal{X}, \epsilon_v\sim D_E}{l(x_v^\intercal \wh_\lambda(X, f(X,\epsilon)), f(x_v,\epsilon_v)}},
\end{align*}
which is the just expected value of $\lv(\lambda,S)$ over the the sampling of the problem instance $S$. 
If the tasks are linear well-specified, we can directly assume a distribution over the variables $\ws{}\sim D_\mathcal{W}$. We can then rewrite the expected validation loss as:
\begin{align*}
    l_v(\lambda) = \E{X\sim D_\mathcal{X}^n, w*\sim D_\mathcal{W}, \epsilon\sim D_E^n}{\E{x_v\sim D_\mathcal{X}, \epsilon_v\sim D_E}{l(x_v^\intercal \wh_\lambda(X, X^\intercal \ws{}+\epsilon), (x_v^\intercal \ws{} + \epsilon_v))}},
\end{align*}
In this paper, we study the problem of finding the optimal hyperparameters for the ridge regression estimator as defined in Section \ref{sec:ridge}, 
and a generalization of ridge regression defined in Section \ref{sec:offset}. 
Our bounds depend on the well-conditioned nature of the sample covariance matrix, and we will denote the smallest singular value of any matrix with the notation $V(.)$. 

\section{Sample complexity bounds for tuning Ridge Regularization}\label{sec:ridge}
In this section, we study generalization guarantees on the ERM estimate of $\lambda$ for the ridge estimator defined in Definition \ref{def:ridge_estimator}. We give our main result in Theorem \ref{thm:lips_tgt}, and study a slightly tighter variant for the well-specified case in Theorem \ref{thm:lips_ws}. Finally, we give Proposition \ref{prop:ex_isotropic}, which instantiates the bound for a general class of ``nice'' distributions. 

\begin{definition}[Ridge Estimator]\label{def:ridge_estimator}
    The ridge estimator for a linear regression task $(X,y)$ with regularization hyperparameter $\lambda \geq 0$ is given as:
\begin{align}
    \wh_\lambda(X,y) &=\argmin_w \|X^\intercal w-y\|^2 + \lambda\|w\|^2\nonumber\\
    \implies \wh_\lambda(X,y) &= (XX^\intercal + \lambda I)^{-1}Xy.\nonumber
\end{align}
\end{definition}

Denote the optimal $\lambda$ as $\lambda^*$ so that
\begin{equation}
    \lv(\lambda^*) = \min_\lambda \lv(\lambda).\nonumber
\end{equation}
We wish to estimate $\lambda^*$ using ERM on the empirical validation loss which satisfies:
\begin{equation}\label{def:lamb_erm}
    \lambda_{ERM} = \argmin_\lambda \lv(\lambda, S) = 
    \argmin_\lambda \frac{1}{T} \sum_t \frac{1}{n_v} \sum_i l(X_{v}^{t(i)\intercal}\wh_{\lambda}(X^t,y^t), y_{v}^{t(i)}).
\end{equation}
Thus $\lambda_{ERM}$ is the value of $\lambda$ that gives the least average validation loss over all of the tasks. We will make the following assumptions on the loss function $l(y_p,y_t)$, valid over all possible values of $X,y,x_v,y_v$ under the support of $D$, and for all possible estimators $\wh, \wh_1, \wh_2$:
\begin{assumption}[Boundedness]\label{assum:bound}
    $l(x_v^\intercal\wh(X,y),y_v) \leq C$.
\end{assumption}
\begin{assumption}[Lipschitzness]\label{assum:lips}
    $|l(x_v^\intercal\wh_1(X,y),y_v) - l(x_v^\intercal\wh_2(X,y),y_v)| \leq L |x_v^\intercal(\wh_1(X,y) - \wh_2(X,y))|$.
\end{assumption}
\begin{remark}
    Note that many popular loss functions, such as the squared loss $l(a,b) = (a-b)^2$, are not bounded on all inputs. We assume that we only receive inputs so that the assumptions hold for the chosen values of $C,L$. We briefly justify our assumptions below:
\begin{enumerate}[leftmargin=*,topsep=0pt,partopsep=1ex,parsep=1ex]\itemsep=-4pt
    \item \textbf{Boundedness:} Boundedness of the loss function is a common assumption made in the literature for proving generalization bounds \citep{mlt_book}. A lot of common loss functions, such as the squared loss are not bounded for all inputs. Prior work addresses this by assuming boundedness of the inputs \citep{regression_pdim}. Assuming boundedness, or well-behaved tail distributions is a common assumption that rely on the fact that real-world data typically has well-behaved tail distributions \citep{kontorovich13,galvez24}.\looseness-1
    \item \textbf{Lipschitzness:} Lipschitzness is another common assumption for proving generalization bounds in literature. For a lot of loss functions, such as the squared loss (Proposition \ref{prop:lipschitz}), hinge loss, etc., Lipschitzness follows directly from the boundedness of the loss function.
\end{enumerate}
Finally, note that, while we allow for any loss function that satisfies the above assumptions, we restrict our attention to regularised least-squares estimators.
\end{remark}

\begin{theorem}\label{thm:lips_tgt}
    Given a loss function that satisfies Assumptions \ref{assum:bound} and \ref{assum:lips} above, the expected validation loss error using the ERM estimator defined in Equation \ref{def:lamb_erm} is bounded with probability $\geq 1-\delta$ as:
    \begin{align}
        l_v(\lambda_{ERM}) - l_v(\lambda^*) &\leq \frac{2ML\Lambda_D^T
        }{\sqrt{T}}\E{}{\|x_v\|} + \frac{2L}{\sqrt{n_vT}} \sqrt{\E{x_v}{\|x_v\|^2}}\E{X,y}{\|y\|/\sqrt{V(XX^\intercal)}}\nonumber\\
       &\quad + \frac{2MLb_v\Lambda_D^T}{\sqrt{n_vT}}\sqrt{\frac{\log(4T/\delta)}{2}} + 5C\sqrt{\frac{\ln(16/\delta)}{2T}}.\nonumber
    \end{align}
    Here $M^2 = \max\|Xy\|^2$, $b_v^2 = \max \|x_v\|^2$ and $\Lambda_D^T = \E{X}{\max_t 1/V(X^tX^{t\intercal})}$. % and $V(M)$ is the smallest non-0 singular value of $M$.
\end{theorem}
\begin{proof} 
    We write $l_v(\lambda_{ERM}) - l_v(\lambda^*) = l_v(\lambda_{ERM}) - l_{v}(\lambda_{ERM},S) +  l_{v}(\lambda_{ERM},S) - l_{v}(\lambda^*,S) + l_{v}(\lambda^*,S) - l_v(\lambda^*)$. We note, as usual, that $l_{v}(\lambda_{ERM},S) - l_{v}(\lambda^*,S) \leq 0$ and $l_{v}(\lambda^*,S) - l_v(\lambda^*)$ is bounded by a Hoeffding bound (Theorem \ref{thm:Hoeffding}). Notably, with probability $\geq 1-\delta$,
\begin{align*}
    l_{v}(\lambda^*,S) - l_v(\lambda^*) \leq C\sqrt{\frac{\ln(1/\delta)}{2T}}.
\end{align*}
It remains to bound $l_v(\lambda_{ERM}) - l_{v}(\lambda_{ERM},S) \leq \sup_\lambda l_v(\lambda) - l_{v}(\lambda,S)$. Lemma \ref{lem:sup_to_rad_2} allows us to break this into error induced from a finite sampling of validation examples, and error induced from finite sampling of training data. We get that with probability at least $1-\delta$:
\begin{align}
        \sup_\lambda l_v(\lambda) - l_{v}(\lambda,S) &\leq 2\E{\sigma, \etr{S}}{\sup_\lambda \frac{1}{n_vT} \sum_{t,i} \sigma^t l(X_v^{t(i)\intercal}\wh_\lambda^t,y_v^{t(i)})}\nonumber\\
        &+ 2\E{\sigma, \eval{S}}{\sup_\lambda \frac{1}{n_vT} \sum_{t,i} \sigma^{t(i)}\E{X,f,\epsilon}{l(X_v^{t(i)\intercal}\wh_\lambda,y_v^{t(i)})}}\nonumber\\
        &+ 2C\sqrt{\frac{2\ln(4/\delta)}{T}}.\nonumber
    \end{align}
Where all $\sigma^{t} $ and $\sigma^{t(i)} $ are i.i.d. Rademacher variables.
We observe that the second term above is much similar to the Rademacher complexity of typical linear regression . We proceed similarly, and in Lemma \ref{lem:S_tr_2} we use Lipschitzness of the loss function to upper bound the second term above in terms of the distribution of outputs $y$.
\begin{align}
   \Eml{\sigma, \eval{S}}\bigg[\sup_\lambda \frac{1}{n_vT} \sum_{t,i} \sigma^{t(i)} &\E{X,f,\epsilon}{l(x_v^{t(i)\intercal}\wh_\lambda,y_v^{t(i)})}\bigg] \leq\nonumber\\
   &\frac{L}{\sqrt{n_vT}} \sqrt{\E{x_v}{\|x_v\|^2}}\E{X,y}{\|y\|/\sqrt{V(XX^\intercal)}}.
\end{align}
In order to upper bound the first term, which is the expected Rademacher complexity of validation loss with a fixed validation set, we show in Lemma \ref{lem:S_v_2} that $\sum_i l(X_v^{t(i)\intercal}\wh_\lambda^t,y_v^{t(i)})$ is Lipschitz in $\frac{1}{V^T + \lambda}$ (according to Definition \ref{def:lips_general}) for fixed $y_v^{t(i)}$. Here $V^T = \min_t V(X^tX^{t\intercal})$ and $V(.)$ is the smallest non-zero eigenvalue of the matrix. We use this Lipschitzness to bound the first term with probability $\geq 1-\delta$ as:
\begin{align}
       \Eml{\sigma, \etr{S}} \Biggl[\sup_\lambda \frac{1}{n_vT} \sum_{t,i} \sigma^t {l(X_v^{t(i)\intercal}\wh_\lambda^t,y_v^{t(i)})}\Biggr]\leq \frac{ML\Lambda_D^T
       }{\sqrt{T}}\E{}{\|x_v\|}%\E{\etr{S}}{\max_t \|(X^tX^{t\intercal})^{-1}\|_\infty}
        + \frac{MLb_v\Lambda_D^T}{\sqrt{n_vT}}\sqrt{\frac{\log(T/\delta)}{2}}.\nonumber
\end{align}
We now replace $\delta$ by $\delta/4$ in the 3 probabilistic bounds above so that the following holds with probability at least $1-\delta$:
\begin{align*}
    l_v(\lambda_{ERM}) - &l_v(\lambda^*) \leq\nonumber\\
        &\frac{2ML\Lambda_D^T
        }{\sqrt{T}}\E{}{\|x_v\|} + \frac{2L}{\sqrt{n_vT}} \sqrt{\E{x_v}{\|x_v\|^2}}\E{X,y}{\|y\|/\sqrt{V(XX^\intercal)}}\nonumber\\
        &+ \frac{2MLb_v\Lambda_D^T}{\sqrt{n_vT}}\sqrt{\frac{\log(4T/\delta)}{2}}
        + 2C\sqrt{\frac{2\ln(16/\delta)}{T}} + C\sqrt{\frac{\ln(4/\delta)}{2T}}.
\end{align*}
To get the desired result, we note that $C\sqrt{\frac{\ln(4/\delta)}{2T}} \leq C/2\sqrt{\frac{2\ln(16/\delta)}{T}}$.
\end{proof}

The above theorem is very generally applicable, only requiring mild assumptions on the regularity of the loss function. We present a couple different variants of the above theorem in this paper that can be more useful for different circumstances. In Theorem \ref{thm:lips_ws}, we give a slightly tighter version of Theorem \ref{thm:lips_tgt} for the well-specified case. We give a variant of Theorem \ref{thm:lips_tgt} that takes an estimate of the expected value of the ground truth to achieve tighter guarantees in Theorem \ref{thm:lips_recentered}. We also present an alternative to Theorem \ref{thm:lips_tgt} in Appendix \ref{sec:alt_bounds}, that proceeds similarly to previous proof techniques such as the ones presented in \cite{mtl_maurer}.

\begin{remark}[Simplifying to Theorem \ref{thm:lips_tgt_inf}]
    We note that $\E{}{\|x_v\|} \leq \sqrt{\E{}{\|x_v\|^2}}$, and further that we can replace the term $\E{X,y}{\|y\|/\sqrt{V(XX^\intercal)}}$ in Lemma \ref{lem:S_tr_2} with $M\Lambda_D^T$. This yields the simplifation of the first two terms in Theorem \ref{thm:lips_tgt} to the first term in Theorem \ref{thm:lips_tgt_inf}. For the latter terms, the reduction is more straight forward since we only focus on the dependence on $T$.  
\end{remark}

\subsection{Well-specified tasks}
The bound in Theorem \ref{thm:lips_tgt} depends on the joint distribution of $X,y$, which in turn depends on the distribution of the function space $D_\mathcal{F}$ and noise vectors $D_E$. In this Section, we present a slightly tighter version of the above bound where we refine the second term based on a well-specified assumption. This allows us to easily analyze the bounds using distributions of $\ws{}$ and $\epsilon$. We instantiate one such analysis in Proposition \ref{prop:ex_isotropic}.

\begin{theorem}\label{thm:lips_ws}
    Given a loss function that satisfies Assumptions \ref{assum:bound} and \ref{assum:lips} above, and tasks that are well-specified linear maps, the expected validation loss error using the ERM estimator defined in Equation \ref{def:lamb_erm} is bounded with probability $\geq 1-\delta$ as:
    \begin{align}
        l_v(\lambda_{ERM}) - l_v(\lambda^*) &\leq \frac{2ML\Lambda_D^T
        }{\sqrt{T}}\E{}{\|x_v\|} + \frac{2L}{\sqrt{n_vT}} \sqrt{\E{x_v}{\|x_v\|^2}}\E{}{\|\ws{}\| + \|\epsilon\|/\sqrt{V(XX^\intercal)}}\nonumber\\
       &\quad + \frac{2MLb_v\Lambda_D^T}{\sqrt{n_vT}}\sqrt{\frac{\log(4T/\delta)}{2}} + 5C\sqrt{\frac{\ln(16/\delta)}{2T}}.\nonumber
    \end{align}
    Here $M^2 = \max\|Xy\|^2$ and $b_v^2 = \max \|x_v\|^2$ $\Lambda_D^T = \E{}{\max_t 1/V(X^tX^{t\intercal})}$. 
\end{theorem}
\begin{proof}[Proof Sketch]
We proceed with this proof similarly to Theorem \ref{thm:lips_tgt} by breaking the error term into error induced from finite sampling of validation data, and error from finite sampling of tasks. 
The bound for the first term proceeds similarly.
For the second term, we use the well-specified assumption to modify Lemma \ref{lem:S_tr_2} by using Lipschitzness in $x_v^\intercal (\wh-\ws{})$. We do this in Lemma \ref{lem:S_tr_ws}, which allows us to bound the trace product in terms of the matrix of $(\wh_\lambda^t - \ws{t})$. This results in a potentially tighter bound in terms of the distributions of $\ws{},\epsilon$.
\end{proof}

In order to better understand the bound from the above theorem, we instantiate it for the case when each entry of each input $x$ is sampled i.i.d.\ in Proposition \ref{prop:ex_isotropic}. Under the mild smoothness assumptions, we obtain a bound that is much tighter than the best known bound from literature, as long as $d=\Omega\left(\frac{\log T}{\log\log T}\right)$, as we see later.

\begin{proposition}\label{prop:ex_isotropic}
    Under the conditions of Theorem \ref{thm:lips_ws}, assume that each entry in the input $x$ is sampled independently from a zero-mean distribution with density bounded by $C_0$ such that $\E{}{xx^\intercal} = \Sigma = \sigma_x^2/d I_d$. Further assume the covariance matrices of both $x,\ws{}$ to have constant trace as $d$ increases. So, $tr(\Sigma) = \sigma_x^2 = const$ and $tr(\E{}{\ws{}\ws{\intercal}}) = \sigma_w^2 = const$. If $n \geq 6d$, the generalization error bound given in Theorem \ref{thm:lips_ws} is $O\left(\frac{1}{\sqrt{T}}(T^{2/d} + \sqrt{\log (T/\delta)})\right)$.
\end{proposition}
\begin{proof}[Proof Sketch]
    The main challenge for instantiating the bound is computing $\Lambda_D^T$. We use results from \cite{cov_mat_mourtada_22} that give tight bounds on the behavior of eigenvalues of the matrix $\hat{\Sigma}_n = \frac{1}{n} XX^\intercal$. In particular, we find that $\E{}{1/V(XX^\intercal)} = O(d/n)$, and $\Lambda_D^T = O(\frac{d}{n}T^{2/d})$. 
    Thus, we can instantitate the bound in Theorem \ref{thm:lips_ws} as:
    \begin{align*}
        l_v(\lambda_{ERM}) - l_v(\lambda^*) =   O\left(\frac{d}{n}\frac{T^{2/d}}{\sqrt{T}} + \frac{\sqrt{\E{}{tr(\ws{}\ws{\intercal})}} + \sqrt{\E{}{\epsilon^2}O(d/n)}}{\sqrt{n_vT}}
        + \frac{\sqrt{\log(T/\delta)}}{\sqrt{T}}\right).
    \end{align*}
    Which gives the desired result using the assumptions of constant trace and $n\geq 6d$.
\end{proof}

\textbf{Discussion and comparison with previous work.} We make the following observations regarding the computed bounds in the above theorems:
\begin{enumerate}[leftmargin=*,topsep=0pt,partopsep=1ex,parsep=1ex]\itemsep=-4pt
    \item The quantity $\Lambda_D^T$ is closely related to $\E{}{1/V(XX^\intercal)}$, which is a common feature in many analyses for linear regression, and for which many techniques have been developed to find reliable upper bounds \citep{cov_mat_yaskov_14,cov_mat_mourtada_22}. We instantiate one such bound for the well-specified case when the inputs $x$ are sampled from an isotropic distribution, such that each element of $x$ is sampled independently from a distribution with bounded density in Proposition \ref{prop:ex_isotropic}.\looseness-1
    \item Compared to previous work, our bounds above are distribution-dependent and much tighter. Prior work from \cite{regression_pdim} give a bound of $O\left(\frac{\sqrt{d + \log(1/\delta)}}{\sqrt{T}}\right)$ for the squared loss, which is weaker than our bound as long as $d=\Omega\left(\frac{\log T}{\log\log T}\right)$, depending on the distribution. We additionally show a distribution independent bound of $O\left(\frac{\sqrt{\log d + \log(1/\delta)}}{\sqrt{T}}\right)$ in Appendix \ref{app:ridge} based on prior work \cite{regression_pdim, enet_2022}. Our bounds beat the new distribution independent analysis bounds when $d = \Omega(T)$, depending on the distribution. Note that prior work does not assume i.i.d. samples within each task as we do.
    We note further that our techniques are much more general, in that they don't rely on the specific nature of the loss function, and consequently work for any Lipschitz loss. As noted below, our bounds also get smaller with increasing number of training examples, which is a feature not present in previous bounds.

    \item 
    Our bounds decrease as the number of training examples ($n$) increase, which was not true in previous work. To see this, first note that the third term in the bounds of both theorems \ref{thm:lips_tgt} and \ref{thm:lips_ws} depend on $\Lambda_D^T$ which decreases with the number of examples following a discussion similar to point 1 above.
    The values of the dominant terms also decrease with the number of training examples up to a certain point. 

    To see a  clearer picture, we would again redirect the attention of the reader to Proposition \ref{prop:ex_isotropic} and its proof in Appendix \ref{sec:supp_res}, where we show that, under the assumptions of the Lemma, the generalization bound behaves as:
    \begin{align*}
        l_v(\lambda_{ERM}) - l_v(\lambda^*) = O_\delta\left(
        \frac{d}{n}\frac{T^{2/d}}{\sqrt{T}} + \frac{\sqrt{\E{}{tr(\ws{}\ws{\intercal})}} + \sqrt{\E{}{\epsilon^2}O(d/n)}}{\sqrt{n_vT}}
        \right).
    \end{align*}
    As $n\rightarrow \infty$, the bound does not get tighter than $O_\delta\left(\frac{\sqrt{\E{}{tr(\ws{}\ws{\intercal})}}}{\sqrt{n_vT}}
    \right)$. This makes practical sense, increasing the number of training examples helps deal with the variance in noise and increasing validation examples or tasks helps deal with the variance in ground truth values. Given a fixed number of supervised examples, if the ground truth varies very heavily, we would like to use a higher number of examples in the validation split.
\end{enumerate} 

\subsection{Convergence guarantees for cross-validation}\label{sec:cross-validation}

While we study the general multi-task setting introduced by \citet{enet_2022} throughout this work, as observed by \citet{enet_2022}, a special case where these guarantees apply is in establishing formal guarantees for the convergence of cross-validation over a single training dataset (single task setting) in terms of the number of iterations or ``folds'' of cross-validation used to tune the hyperparameter. For example, if one does leave-one-out cross-validation (LOOCV), then the number of folds or iterations needed is equal to $n$, the size of the training set of the task. This can be very inefficient, as one needs to solve $n$ regression problems for each value of the hyperparameter $\lambda$. Another related approach is Monte-Carlo cross-validation, where one does a random independent training-validation split in a fixed proportion (e.g.\ 80\% training + 20\% validation) to compute the validation loss of each hyperparameter, and sets the best hyperparameter. For this setting, we can apply the results of Appendix \ref{app:ridge} by assuming each task is a ``fold" of cross-validation, and is a randomly chosen training-validation split. We thus get a bound of $O(\log d/\epsilon^2)$ ``folds" being enough to estimate $\lambda$ up to an $\epsilon$ error in validation loss. Note that we can only give guarantee on the performance of the learned $\lambda$ with respect to the best $\lambda$ for the same training data as used for cross-validation. That is, we do not give guarantee on the performance of $\lambda$ on newly sampled training data (which as we argue in Appendix \ref{sec:single_task}, is not feasible).

If we instead define each fold of cross-validation to be a collection of $n+n_v$ examples sampled from the original training set (i.i.d. with repitition), then we can use the results of the above Section to give tighter bounds on cross-validation. Specifically, under the conditions of Proposition \ref{prop:ex_isotropic}, in the high-dimensional regime $d=\Omega(\log T)$, our bounds imply that  $O\left(\frac{\left(\log\frac{1}{\epsilon}\right)^2}{\epsilon^2}\right)$ folds are sufficient, which is an improvement if $d=\Omega\left(\left(\log\frac{1}{\epsilon}\right)^2\right)$ as well.

\section{Sample complexity of tuning  LASSO and Elastic Net}\label{sec:en}

We will now establish similar distribution-dependent bounds on the generalization error for tuning the regularization coefficient in LASSO.

\begin{definition}[LASSO Estimator]\label{def:lasso_estimator}
    The LASSO estimator for a linear regression task $(X,y)$ with regularization hyperparameter $\lambda_1 \in [\underline{\Lambda},\overline{\Lambda}]$ is given as:
\begin{align}
    \wh_{\lambda_1}(X,y) &=\argmin_w \|X^\intercal w-y\|^2 + \lambda_1\|w\|_1.\nonumber
\end{align}
\end{definition}

\noindent Under the same boundedness and Lipschitzness assumptions on the  loss function as above, along with a full rank assumption (Assumption \ref{asmp:lasso-full-rank} in the appendix), we have the following result.

\begin{theorem}\label{thm:lasso}
    The expected validation loss error using the ERM estimator for LASSO is bounded with probability $\geq 1-\delta$ as:
    \begin{align}
        l_v(\lambda_{ERM}) - l_v(\lambda^*) &\leq\frac{{2L\overline{\Lambda}\tilde{\Lambda}_D^T}\E{x_v}{\|x_v\|}\sqrt{d}}{\sqrt{T}} \nonumber\\
        &+\frac{2L\sqrt{\E{x_v}{\|x_v\|^2}}}{\sqrt{n_v T}}\E{X,y}{ \max_{\cE}\left( \frac{\|y\|}{\sqrt{V(X_\cE X_\cE^\intercal)}} + \overline{\Lambda}\frac{\sqrt{d}}{V(X_\cE X_\cE^\intercal)}\right)}\nonumber \\
        &%+ \frac{2LM\Lambda_D^T}{T^{3/4}}\E{x_v}{\|x_v\|}\sqrt[4]{\ln{(4/\delta)}/2}
         + \frac{Lb_v\overline{\Lambda}\tilde{\Lambda}_D^T}{\sqrt{n_vT}}\sqrt{2\ln(T/\delta)}\nonumber + 5C\sqrt{\frac{\ln(16/\delta)}{2T}}.\nonumber
    \end{align}
    Here $b_v^2 = \max \|x_v\|^2$ and $\tilde{\Lambda}_D^T = \E{\etr{S}}{ \max_{\cE,t} \frac{1}{V(X^t_\cE X_\cE^{t\intercal} ) }}$, where the maximum is over all feature subsets $\cE\subseteq [d]$ and tasks $t$ in the training set. 
\end{theorem}

\begin{proof}[Proof Sketch]
    The proof follows the same overall structure as the proof of Theorem \ref{thm:lips_ws}. The relevant lemmas for bounding the Rademacher complexity for LASSO are Lemmas \ref{lem:S_v_1} and \ref{lem:S_tr_1} established in Appendix \ref{app:lasso}. The key difference comes from the difference in the LASSO solution. Unlike ridge, there is no fixed closed form solution for all values of $\lambda_1$. The solution $\hat{w}_{\lambda_1}$ is a piecewise linear function of $\lambda_1$ and the closed form expressions for within fixed pieces is known. We use this to bound the relevant Rademacher complexity for the class of loss functions which express the validation loss as a function of $\lambda_1$.
\end{proof}

\noindent We also give the following bound on the generalization error of simultaneously tuning L1 and L2 penalties for $\lambda_1\in[\underline{\Lambda}_1,\overline{\Lambda}_1],\lambda_2\in[\underline{\Lambda}_2,\infty)$ (see Appendix \ref{app:elasticnet}).

\begin{theorem}\label{thm:en}
    The expected validation loss error using the ERM estimator for Elastic Net is bounded with probability $\geq 1-\delta$ as:
    \begin{align}
        l_v&(\lambda_{ERM}) - l_v(\lambda^*) \le \frac{2L\overline{\Lambda}\sqrt{d}}{\sqrt{T}}\left(\E{x_v}{\|x_v\|}+b_v\sqrt{\frac{\log(T/\delta)}{2n_v}}\right)\E{X}{\max_{t,\cE}  \frac{1}{V(X^t_\cE X_\cE^{t\intercal}) +\underline{\Lambda}_2 }} \nonumber\\
        &+\frac{2L\sqrt{\E{x_v}{\|x_v\|^2}}}{\sqrt{n_vT}}\E{X,y}{ \max_{\cE}\left( \frac{\|y\|\sqrt{V^*(X_\cE X_\cE^\intercal)}}{V^*(X_\cE X_\cE^\intercal)+\underline{\Lambda}_2}+\frac{\overline{\Lambda}_1\sqrt{d}}{V(X_\cE X_\cE^\intercal)+\underline{\Lambda}_2}\right)}+ 5C\sqrt{\frac{\ln(16/\delta)}{2T}}.\nonumber \label{eqn:en-main}\\
    \end{align}
    Here  %$V({M})$ is the smallest non-zero singular value of matrix $M$ and 
    $V^*({M})$ is the non-zero singular value of $M$ that maximizes $\frac{\sqrt{\sigma_i(M)}}{\sigma_i(M)+\underline{\Lambda}_2}$.
\end{theorem}

\noindent As above, we show that our bounds are much sharper than prior work for well-studied ``nice'' distributions. For sub-Gaussian data distribution we show that our bounds on the generalization error are independent of the feature dimension $d$. In contrast, prior work on worst-case distributions~\citep{regression_pdim} shows a tight $\Theta(d)$ bound on the pseudo-dimension for tuning the elastic net regularization coefficients.
Formally we have the following proposition (proof in Appendix \ref{app:elasticnet}).

\begin{proposition}\label{prop:en_isotropic_main}
    Consider the expected validation error of an ERM estimator for the Elastic Net hyperparameters over the  range $\lambda_1\in[\underline{\Lambda}_1,\overline{\Lambda}_1],\lambda_2\in[\underline{\Lambda}_2,\infty)$.
    Assume further that all tasks are well-specified such that all inputs $x$ are sampled from sub-Gaussian distributions with independent entries. Concretely, assume that each entry in the input $x$ is sampled independently from a zero-mean sub-Gaussian distribution  such that $\E{}{xx^\intercal} = \Sigma = (\sigma_x^2/d) I_d$. We further restrict the covariance matrices of both $x,\ws{}$ to have constant trace as $d$ increases. So, $tr(\Sigma) = \sigma_x^2 = const$ and $tr(\E{}{\ws{}\ws{\intercal}}) = \sigma_w^2 = const$. For sufficiently large $n \geq \Omega\left(d+\log\frac{T}{\underline{\Lambda}_2}\right)$, the generalization error bound given in Theorem \ref{thm:en} is $\tilde{O}\left(1/\sqrt{nT}+\sqrt{\frac{\log 1/\delta}{T}}\right)$, where the soft-O notation suppresses dependence on quantities apart from $T,n,\delta$ and $d$.
\end{proposition}

\noindent We conclude this section with a couple of remarks. We show how our Elastic Net bounds apply to cross-validation and how they compare with the worst-case bounds from prior work. We also show how to obtain the informal Theorem \ref{thm:informal-en} from the more precise bounds in Theorem \ref{thm:en} above.

\begin{remark}
    As in Section \ref{sec:cross-validation}, the results here apply to bounding the convergence rate for single-task cross-validation, under similar restrictions of sampling from repetition for each fold. For Monte-Carlo cross-validation, the bound on the sufficient number of folds is improved from $O(d/\epsilon^2)$ due to prior work \citep{regression_pdim} to $O(1/\epsilon^2)$ for sufficiently large $n$ under the conditions of Proposition \ref{prop:en_isotropic_main}, although the bounds from prior work apply even when sampling without repetition.
    Also, we note that the comment in Footnote \ref{footnote:combined-bounds} applies to our Elastic Net bounds as well, and in the above we have $l_v(\lambda_{ERM}) - l_v(\lambda^*) = \tilde{O}\left(\min\left\{\frac{1}{\sqrt{nT}} +\sqrt{\frac{\log 1/\delta}{T}}, \frac{\sqrt{d}}{\sqrt{T}}\right\}\right)$.
\end{remark}

\begin{remark}
    As noted in Theorem \ref{thm:informal-en}, we suppress the validation data terms including $\E{x_v}{\|x_v\|}, b_v, n_v$ and $\E{x_v}{\|x_v\|^2}$ in the $\tilde{O}$ notation, since these terms do no impact the bounds in the application of our result to the sub-Gaussian example in Proposition \ref{prop:en_isotropic_main}. The last term is asymptotically dominated by the $\sqrt{\frac{\log(T/\delta)}{T}}$ quantity in the first term (asymptotics in $T$), and the second term under $\E{X,y}{\max_\cE(\cdot)}$ in Equation (\ref{eqn:en-main}) is  also dominated by the first term.
\end{remark}

\section{Re-centered Ridge Regression}\label{sec:offset}
We note that the bounds above in the well-specified case in Theorem \ref{thm:lips_ws} depend on the quantity $\E{\ws{}}{\|\ws{}\|^2}$. This can be quite large if $\ws{}$ is not centered around 0. We thus suggest using the following estimator, and give generalization guarantees for ERM estimation of the regularization hyperparamter.

\begin{definition}[Re-centered Ridge Estimator \citep{ridge_wieringen}]\label{def:gen_ridge}
    The re-centered ridge estimator for a linear regression task $(X,y)$ with hyperparameters $\lambda,\mu$ is given as:
\begin{align}
    \wh_{(\lambda,\mu)}(X,y) &=\argmin_w \|X^\intercal w-y\|^2 + \lambda\|w - \mu\|^2\nonumber\\
    \implies \wh_\lambmu(X,y) &= (XX^\intercal  + \lambda I)^{-1}Xy + \lambda(XX^\intercal +\lambda I)^{-1}\mu.\nonumber%\label{eq:off_ridge_estimator}
\end{align}
\end{definition}
Intuitively, instead of penalizing the distance of $w$ from origin, this estimator penalizes its distance from a known, central point. 

In the following, we assume we have a fixed estimate of the optimal $\mu^*$ given as $\muh$, and bound the validation error on the ERM estimate of $\lambda$ using the MSE in $\muh$. We are able to get a tighter bound than in Theorem \ref{thm:lips_ws}, where we replace all $\E{}{\ws{}\ws{\intercal}}$ with the variance of $\ws{}$, and only incur an additional error term that depends on the closeness of the estimate $\muh$ to the actual $\mu^*$.
\begin{theorem}\label{thm:lips_recentered}
     For a validation loss function that satisfies Assumptions \ref{assum:bound} and \ref{assum:lips} given in Section \ref{sec:ridge}, and tasks that are well-specified linear maps, the expected validation loss error using the ERM estimator defined in Equation \ref{def:lamb_erm} using the re-centered ridge estimator for a given $\muh$ is bounded with probability $\geq 1-\delta$ as:
     \begin{align}
         l_v(\lambda_{ERM},\muh) - \lv(\lambda^*,\mu^*) &\leq L\E{x_v}{\|x_v\|}\|\muh - \mu^*\| + 
         \frac{2ML\Lambda_D^T
        }{\sqrt{T}}\E{}{\|x_v\|}\nonumber\\
        &+ \frac{2L}{\sqrt{n_vT}} \sqrt{\E{x_v}{\|x_v\|^2}}\E{}{\|\ws{}\| + \|\epsilon\|/\sqrt{V(XX^\intercal)}}\nonumber\\
       &\quad + \frac{2MLb_v\Lambda_D^T}{\sqrt{n_vT}}\sqrt{\frac{\log(4T/\delta)}{2}} + 5C\sqrt{\frac{\ln(16/\delta)}{2T}}.\nonumber    
     \end{align}
     Here $M^2 = \max\|Xy\|^2$ and $\Lambda_D^T = \E{}{\max_t 1/V(X^tX^{t\intercal})}$.
\end{theorem}
\begin{proof} We start by decomposing the excess risk on validation set as
\begin{align}
    l_v(\lambda_{ERM},\muh) - \lv(\lambda^*,\mu^*) &\leq l_v(\lambda_{ERM},\muh) - \lv(\lambda_{ERM},\mu^*)+ l_v(\lambda_{ERM},\mu^*) - \lv(\lambda^*,\mu^*)\label{eq:rec_break}
\end{align}
We bound the first term as follows:
\begin{align}
    l_v(\lambda_{ERM},\muh) - \lv(\lambda_{ERM},\mu^*) &\leq \sup_\lambda l_v(\lambda,\muh) - \lv(\lambda,\mu^*)\nonumber\\
    &= \sup_\lambda \E{}{l(x_v^\intercal \wh_\lambmuh{}, y_v) - l(x_v^\intercal \wh_\lambmus{}, y_v)}\nonumber\\
    &\leq \sup_\lambda \E{}{L(x_v^\intercal (\wh_\lambmuh{}-\wh_\lambmus{})} \nonumber\\
    &= L\sup_\lambda\E{}{|\lambda x_v^\intercal (XX^\intercal +\lambda I)^{-1}(\muh - \mu^*)|}\nonumber\\
    &\leq L\sup_\lambda \E{}{\|\lambda(XX^\intercal +\lambda I)^{-1}x_v\|}\|\muh - \mu^*\|\nonumber\\
    &\leq L\E{}{\|x_v\|}\|\muh - \mu^*\|.\nonumber
\end{align}
\\
For the second term, we see that generalization error in finding $\lambda$ is the same as generalization error in finding $\lambda$ for the well-specified case if we replace $\ws{}$ by $\ws{} - \mu^*$. To see this,
\begin{align*}
    \wh_{(\lambda,\mu^*)}(X,y) &=\argmin_w \|X^\intercal w-y\|^2 + \lambda\|w - \mu^*\|^2\\
    &=\argmin_w \|X^\intercal w-(X^\intercal \ws{} + \epsilon)\|^2 + \lambda\|w - \mu^*\|^2\\
    &= \argmin_w \|X^\intercal (w-\mu^*)-X^\intercal (\ws{} - \mu^*) + \epsilon\|^2 + \lambda\|w - \mu^*\|^2
\end{align*}
So that the optimization problem reduces to the same problem as in Definition \ref{def:ridge_estimator} with a shifting of the axes. Thus, we get a similar bound for the second term of Equation \ref{eq:rec_break} as in Theorem \ref{thm:lips_ws}, only requiring replacing $\ws{}$ with $\ws{} - \mu^*$, which was the intended effect.
\end{proof}

\section{Discussion and future work}
Distribution-dependent generalization guarantees are widely studied in statistical learning theory as an effective way to take into account the niceness of the distribution and give tighter learning guarantees. We study the fundamental problem of tuning the regularization parameter of linear regression across tasks. Our bounds improve upon previous distribution-independent results. In particular, we show that our bounds do not get worse with the feature dimension for various nice distributions, which is unavoidable for distribution-independent bounds. We also extend our results to generalizations including re-centered ridge regression. 
An interesting direction for future work is to show lower bounds to better understand the tightness of our results. Another exciting direction is to show distribution-dependent guarantees for sparse regression with the L0 penalty (also called best subset selection, see Appendix \ref{app:l0} for corresponding distribution-independent guarantees).

\bibliography{main}
\bibliographystyle{plainnat}   

\appendix

\newpage
\section{Additional Related Work}\label{sec:add_rl}
\textbf{Empirical Bayes.} Empirical Bayes (EB) involves finding the best Bayesian estimator for a set of parameters (say $\theta_i$) assumed to be sampled from an unknown prior, given samples (say $X_i\sim p(\theta_i)$) drawn from distributions that depend on parameters $\theta_i$. As a typical example, consider a Gaussian Sequence Model, where one observes $X_i=\theta_i+\epsilon_i$ for $i \in [n]$. Here $\epsilon_i\sim \mathcal{N}(0,\sigma^2)$. The idea, originally proposed by \citet{ebayes_robbins}, involves assuming $\theta_i$ being sampled from an unknown prior (distinguishing this from a purely Bayesian method where we assume a prior), and using the shared structure to find better estimates. Commonly, Empirical Bayes approaches are divided into $f$-modeling and $g$-modeling \citep{fg_efron, fg_shen}. While in $f$-modeling we explicitly find prior parameters, $g$-modeling works by directly finding the target variable without finding the prior explicitly. Empirical Bayes has been heavily studied in statistics literature, providing sample complexity bounds in certain circumstances such as the Poisson model \citep{ebay_poisson_jana}. Empirical Bayes approach to linear regression involves assuming an prior on ground truth vector with unknown parameters. From our results from Section \ref{sec:equivalence}, we see that EB estimation of linear regression parameters under a Gaussian prior is equivalent to our setting of learning ridge parameters from multiple tasks. While ours is a $g$-modeling approach where we don't estimate prior parameters directly, asymptotic optimality of $f$-modeling approaches have been shown previously \citep{eb_superiority_zhang05}. Though our generalization guarantees for ridge regression hold for all priors, including non-Gaussian priors, the best ridge estimator is not Bayes optimal for non-Gaussian priors. Several empirical \citep{blasso_park,flex_ebayes_kim} and theoretical \citep{lin_ebayes_wei} papers have studied EB methods for linear regression under other priors.

{\bf Data-driven algorithm design}. Data-driven algorithm design is a recently introduced paradigm for designing algorithms and provably tuning hyperparameters in machine learning~\citep{balcan2020data}. Apart from regression, the framework has been successfully used for designing several fundamental learning algorithms~\citep{balcan2018data,balcan2023analysis,balcan2024algorithm,balcan2025sample,blum2021learning,bartlett2022generalization,balcan2021data,balcan2024learning1,balcan2024trees,jinsample}, as well as solving  optimization problems including clustering (e.g.\ \citealt{balcan2017learning,balcan2020learning,balcan2024accelerating}), linear and integer programming (e.g.\ \citealt{balcan2022structural,balcan2024learning,khodak2024learning,cheng2024learning,sakaue2024generalization}). The techniques developed in this line of work allows selection of provably good hyperparameters over continuous domains, given access to multiple related tasks drawn from an unknown distribution~\citep{balcan2017learning,balcan2021much,balcan2024subsidy} or arriving online~\citep{balcan2018dispersion,sharma2020learning,balcan2021learning,sharma2023efficiently,sharma2024no,sharma2025offline}. Recent work shows that tuning discretized parameters can lead to much worse performance than competing with the best continuous parameter~\citep{balcan2024learning}.

\section{A distribution-independent bound for tuning ridge regularization based on prior work}\label{app:ridge}

While prior work~\citep{enet_2022,regression_pdim} establishes asymptotically tight bounds on the learning-theoretic complexity of simultaneously tuning L1 and L2 regularization coefficients in the elastic net, no direct bounds are given for just ridge regression. \citet{enet_2022} provide $\tilde{O}(\log d)$ on the pseudo-dimension of the 0-1 loss function class for tuning ridge-regularized {\it classification}, which is smaller than the $\Theta(d)$ bounds for elastic net. Here we provide a simple extension to their results and show that a similar  ${O}(\log d)$ upper bound can be shown for tuning the regularization in ridge {\it regression}.

We first recall some useful results from prior work. The following lemma is due to \citet{enet_2022}.

\begin{lemma}\label{lem:gram-inv}
Let $A$ be an $r\times s$ matrix. Consider the matrix $B(\lambda)=(A^\intercal A+\lambda I_s)^{-1}$ and $\lambda>0$. Then each entry of $B(\lambda)$ is a rational polynomial $P_{ij}(\lambda)/Q(\lambda)$ for $i,j\in[s]$ with each $P_{ij}$ of degree at most $s-1$, and $Q$ of degree $s$.
\end{lemma}

\noindent In addition, we will also need the definition of the refined GJ framework introduced by \cite{bartlett2022generalization}.

\begin{definition}[\cite{bartlett2022generalization}]
    A GJ algorithm $\Gamma$ operates on real-valued inputs, and can perform two types of operations:
    \begin{itemize}
        \item Arithmetic operations of the form $v=v_0 \odot v_1$, where $\odot \in \{+, -, \times, /\}$.
        \item Conditional statements of the form ``if $v_0 \ge 0$ then $\dots$ else $\dots$''.
    \end{itemize}
    In both cases, $v_0, v_1$ are either inputs or values previously computed by the algorithm (which are rational functions of the inputs). The {\it degree} of a GJ algorithm is defined as the maximum degree of any rational function  of the inputs that it computes. The {\it predicate complexity} of a GJ algorithm is the number of distinct rational functions that appear in its conditional statements.
\end{definition}

\noindent The following theorem  due to \cite{bartlett2022generalization}  is useful in obtaining a pseudodimension bound by showing a GJ algorithm that computes the loss for all values of the hyperparameters, on any fixed input instance.

\begin{theorem}[\cite{bartlett2022generalization}] \label{thm:gj}
    Suppose that each function $f \in \mathcal{F}$ is specified by $n$ real parameters. Suppose that for every $x \in \mathcal{X}$ and $r \in \R$, there is a GJ algorithm $\Gamma_{x, r}$ that given $f \in \mathcal{F}$, returns ``true" if $f(x) \geq r$ and ``false" otherwise. Assume that $\Gamma_{x, r}$ has degree $\Delta$ and predicate complexity $\Lambda$. Then, $\mathrm{Pdim}(\mathcal{F}) = O(n\log(\Delta\Lambda))$.
\end{theorem}

\noindent Let $\mathcal{H}_{\text{Ridge}}$ denote the loss function class that consists of functions (each function corresponds to a distinct value of $\lambda\in(0,\infty)$) computing the validation loss on any input instance $(X,y,X_v,y_v)$ for using a fixed  Ridge parameter $\lambda$ as in the notation of~\citep{enet_2022}. We have the following result, which implies distribution-independent sample complexity of $\tilde{O}\left(\frac{\log d}{\epsilon^2}\right)$ for tuning $\lambda$.

\begin{theorem}
    The pseudo-dimension of the function class $\mathcal{H}_{\text{Ridge}}$ is $O(\log d)$.
\end{theorem}

\begin{proof}
    For a fixed problem instance  $P=(X,y,X_v,y_v)$, the ridge solution is given by $\hat{w}_{\lambda}=(XX^{\intercal}+\lambda I)^{-1}Xy$ and the validation loss $\ell_\lambda(P)$ is $\|X_v^{\intercal}\hat{w}_{\lambda}-y_v\|^2$. By Lemma \ref{lem:gram-inv}, $\hat{w}_{\lambda}$ is a rational function of $\lambda$ with degree at most $d$, and the validation loss is also a rational function of $\lambda$ with degree at most $2d$. This gives us a GJ algorithm for computing whether $\ell_\lambda(P)\ge r$ for any instance $P$ and $r\in\R$,  with degree at most $2d$ and predicate complexity 1. Theorem \ref{thm:gj} now implies the claimed pseudo-dimension bound.
\end{proof}

\section{Tuning the L0 penalty for best subset selection}\label{app:l0}

Least squares regression with the L0 penalty, also known as {\it best subset selection}, is perhaps the most direct formulation for sparse regression~\citep{beale1967discarding,hocking1967selection}. Despite being a discrete optimization problem that is NP-hard (L1 and L2 regressions have efficient algorithms as they induce convex objectives), recent work develops efficient mixed integer optimization based approaches via branch-and-bound techniques for best subset selection and shows advantages over LASSO for variable selection~\citep{bertsimas2016best,hastie2020best}. The regression is given by the following optimization problem.

$$\min_w \|X^\intercal w-y\|_2^2 + \lambda_0\|w\|_0,$$

\noindent where $\lambda_0\in\R_{\ge0}$ is the L0 penalty. Let $l_v^{\lambda_0}(X,y,X_v,y_v)$ denote the validation loss $\|X^\intercal \hat{w}_0-y\|_2^2$ for $\hat{w}_0\in \argmin_w \|X^\intercal w-y\|_2^2 + \lambda_0\|w\|_0$, a solution to the L0 optimization on the training set $(X,y)$. Let $\mathcal{H}_0$ denote the class of validation loss functions $\{l_v^{\lambda_0}\mid \lambda_0\in\R_{\ge0}\}$. We can show a distribution-independent sample complexity of $\tilde{O}(d/\epsilon^2)$ by using the following result.

\begin{theorem}
    Suppose Assumption \ref{asmp:lasso-full-rank} holds. The pseudo-dimension of the function class $\mathcal{H}_{0}$ is $O(d)$.
\end{theorem}

\begin{proof}
    For a fixed problem instance, $P=(X,y,X_v,y_v)$, we will show that the dual loss function $l_v^*(\lambda_0)$ is a piecewise-constant function with at most $2^d$ pieces. Let $f(\lambda,w):=\|X^\intercal w-y\|_2^2 + \lambda_0\|w\|_0$ denote the L0 regression objective.

    For a fixed active set (the set of non-zero coefficients of $w$) $\cE$, we have $f(\lambda,w)=\|X^\intercal w-y\|_2^2 + \lambda_0\|w\|_0=\|X_{\cE}^\intercal w_{\cE}-y\|_2^2 + \lambda_0|\cE|$, and under Assumption \ref{asmp:lasso-full-rank}, the solution is unique and is given by $\hat{w}_{\cE}=(X_\cE X_\cE^{\intercal})^{-1}X_\cE y$. Thus, we have at most $2^d$ distinct solutions $\hat{w}_{\cE}$. Furthermore, we show that if $\hat{w}$ is a solution for two distinct L0 penalty values $\lambda_0^{(1)}<\lambda_0^{(2)}$, then it is also the solution for all intermediate values $\lambda_0\in(\lambda_0^{(1)},\lambda_0^{(2)})$. This claim, together with the previous observation, implies that $l_v^*(\lambda_0)$ has at most $2^d$ pieces.

    Indeed, let $\lambda_0^{(1)}<\lambda_0^{(2)}$, $w_0\in \argmin_w f(\lambda_0^{(1)},w)$ and $w_0\in \argmin_w f(\lambda_0^{(2)},w)$. Suppose $\lambda_0\in(\lambda_0^{(1)},\lambda_0^{(2)})$. If possible, let $f(\lambda_0,w')<f(\lambda_0,w_0)$ for some $w'$. We have two cases. Either, $\|w'\|_0\le\|w_0\|_0$, and
    \begin{align*}
        f(\lambda_0^{(2)},w')&=f(\lambda_0,w')+(\lambda_0^{(2)}-\lambda_0)\|w'\|_0\\
        &<f(\lambda_0,w_0)+(\lambda_0^{(2)}-\lambda_0)\|w_0\|_0\\
        &=f(\lambda_0^{(2)},w_0),
    \end{align*}
    \noindent contradicting that $w_0\in \argmin_w f(\lambda_0^{(2)},w)$. The only other possibility is $\|w'\|_0>\|w_0\|_0$. But in this case,
    \begin{align*}
        f(\lambda_0^{(1)},w')&=f(\lambda_0,w')-(\lambda_0-\lambda_0^{(1)})\|w'\|_0\\
        &<f(\lambda_0,w_0)-(\lambda_0-\lambda_0^{(1)})\|w_0\|_0\\
        &=f(\lambda_0^{(2)},w_0),
    \end{align*}
    contradicting that $w_0\in \argmin_w f(\lambda_0^{(1)},w)$. Therefore, $w_0\in \argmin_w f(\lambda_0,w)$ for all $\lambda_0\in(\lambda_0^{(1)},\lambda_0^{(2)})$.

    Finally, we use Lemma 2.3 of \cite{balcan2020data} to conclude that the pseudo-dimension of $\mathcal{H}_0$ is $O(d)$.
\end{proof}

\noindent An alternative proof may be given using the GJ framework based approach (see Appendix \ref{app:ridge}). While the above result gives a data-independent bound on the sample complexity for worst-case distributions, it is also possible to give a distribution-dependent bound that depends on the actual (expected) number of pieces in the dual loss function $l_v^*(\lambda_0)$. The key idea is to bound the Rademacher complexity using Massart's lemma~\citep{balcan2018data,sharma2025offline}. An interesting open question is to show a concrete gap between the sample complexity for ``nicer'' distributions and the bound for worst-case distributions above (similar to our results for ridge regression and the elastic net). Based on the remarkable recent work of~\citep{bertsimas2016best}, it would also be very interesting to tune branch-and-bound based techniques for best subset selection~\citep{balcan2018learning,balcan2022structural,cheng2025generalization}. Another interesting question is to study simultaneous tuning of the L0 penalized linear regression with L1 and/or L2 penalties.

\section{Motivation for multi-task learning}\label{sec:single_task}
Throughout the paper, we discuss the generalization error of finding a common regularization hyper-parameter for a multi-task. In this Section we briefly argue that there is no notion of an optimal estimator for a single task. We first note the following expression for validation loss that follows from simple linear algebra.

\begin{proposition}
    Consider a ridge regression task using regularization parameter $\lambda$ on a well-specified linear regression task with ground truth $\ws{}$ where training and validation data given as $(X,y)$ and $(x_v,y_v)$ respectively are such that $y = X^\intercal\ws{} + \epsilon$ and $y_v = x_v^\intercal \ws{} + \epsilon_v$ for zero mean noise $\epsilon, \epsilon_v$. The expected validation loss over noise is then given by:
    \begin{align*}
        l_v &= \E{\epsilon_v,\epsilon}{\|y_v - x_v^\intercal \wh_\lambda\|^2}\\
        &= \E{}{\|\epsilon_v\|^2} + \lambda^2 \|x_v^\intercal(XX^\intercal + \lambda I)^{-1}
        \ws{}\|^2 + \E{\epsilon}{\|x_v^\intercal(XX^\intercal + \lambda I)^{-1}X\epsilon\|^2}.
    \end{align*}
    where $\wh_\lambda = (XX^\intercal + \lambda I)^{-1}Xy $ is the learned ridge estimate.
\end{proposition}
\begin{proof}
    Given that the problem is well-specified,
    \begin{align*}
        \E{\epsilon_v,\epsilon}{\|y_v - x_v^\intercal \wh_\lambda\|^2} &= \E{\epsilon_v,\epsilon}{\|x_v^\intercal \ws{} + \epsilon_v - x_v^\intercal \wh_\lambda\|^2}\\
        &= \E{}{\|\epsilon_v\|^2} + \E{\epsilon}{\|x_v^\intercal(\ws{} - \wh_\lambda)\|^2}\\
        &= \E{}{\|\epsilon_v\|^2} + \E{\epsilon}{\|x_v^\intercal(\ws{} - (XX^\intercal + \lambda I)^{-1}Xy)\|^2}\\
        &= \E{}{\|\epsilon_v\|^2} + \E{\epsilon}{\|x_v^\intercal(\ws{} - (XX^\intercal + \lambda I)^{-1}(XX^\intercal \ws{} + X\epsilon))\|^2}\\
        &= \E{}{\|\epsilon_v\|^2} + \E{\epsilon}{\|x_v^\intercal(\lambda (XX^\intercal + \lambda I)^{-1}
        \ws{} - (XX^\intercal + \lambda I)^{-1}X\epsilon)\|^2}\\
        &= \E{}{\|\epsilon_v\|^2} + \lambda^2 \|x_v^\intercal(XX^\intercal + \lambda I)^{-1}
        \ws{}\|^2 + \E{\epsilon}{\|x_v^\intercal(XX^\intercal + \lambda I)^{-1}X\epsilon\|^2}
    \end{align*}
\end{proof}

Note from the expression in the above expression that for any fixed $\lambda$, we can adversarially chose a ground truth $\ws{}$ such that the expected validation loss is arbitrarily high. So, if we try to learn the ridge parameter for a single task, we are not guaranteed that the ridge parameter would perform well on a different sampling of the training and validation set. We can however, guarantee the existence of a single $\lambda$ that performs well \textit{on average} over a multi-task. In the subsequent Section, we discuss conditions when the best regularized estimator is also Bayes optimal.

\section{Bayes estimation using multi-task learning}\label{sec:equivalence}
% In this paper, we have given generalization guarantees on finding the optimal regularized estimator using a finite sample of tasks. 
In this section we are interested in conditions for when the optimal regularized estimator is also provably optimal for multi-task learning. To argue optimality of an estimator for a future, unknown task it is crucial to define the relationship between tasks already seen and the future task. Given any such relationship, we can re-formulate it to form a prior. Thus, the optimality of any estimator reduces to the case when the estimator is equal to the Bayesian estimator with the given prior. Of course, if the prior is known it is straight-forward to find such an estimator. The key challenge of this line of work, and for Empirical Bayes methods, is to find an approximate estimator from an unknown prior.

We show that the optimal regularized estimator is equal to the Bayesian estimator, and hence the optimal multi-task learning estimator when the regularization takes a form similar to the prior. For example, a re-centered ridge estimator is optimal if the prior is Gaussian. Note that our reduction of multi-task learning crucially depends on the (unknown) prior being \textit{frequentist} in the sense that the tasks are assumed to be sampled randomly from this prior distribution, as opposed to the prior being a \textit{belief} over the sampling of tasks.

It is well-known~\citep{all_of_stats} that for squared loss $l(\wh, \ws{}) = \|\wh - \ws{}\|^2$, the Bayesian estimator is given by $\wh = \E{}{\ws{}|X,y}$. The same estimator is the Bayesian estimator for the expected validation loss given as $l_v(\wh) = \E{X_v,y_v}{\|X_v^\intercal \wh - y_v\|^2}$ as shown in Theorem \ref{thm:bayes_estimator}.

\begin{theorem}\label{thm:bayes_estimator}
Given a linear problem $(X,y)$ such that $\exists \ws{}, \epsilon, y = X^\intercal \ws{} + \epsilon$. Given a prior over $\ws{} \sim \pi$, the Bayesian estimator corresponding to the validation loss $\lv = \E{X_v,y_v}{\|X_v^\intercal \wh-y_v\|^2}$, where $(X_v,y_v)$ are sampled from the same map as $(X,y)$, is given as:
\begin{align}
    \wh = \E{}{\ws{}|X,y}\nonumber
\end{align}
In other words, the Bayesian estimator is equal to the mean of the posterior.
\end{theorem}
\begin{proof}
    Define the Bayesian risk as:
\begin{align*}
    B_\pi(\hat{w}) = \int \E{X_v,y_v}{\|X_v^\intercal \hat{w} - y_v\|^2}\Pr(w^*|X,y)m(X,y)dXdydw^*.
\end{align*}
Here $m(X,y)$ denotes the marginal distribution on $X,y$.
We note that $y_v$ is sampled from the same ground truth $w^*$ as $y$ so we can re-write this as: 
\begin{align*}
    B_\pi(\hat{w}) &= \int \E{X_v,\epsilon_v}{\|X_v^\intercal (\hat{w} - w^*) - \epsilon_v\|^2}\Pr(w^*|X,y)m(X,y)dXdydw^*\\
    &= \int (\E{X_v}{\|X_v^\intercal (\hat{w} - w^*)\|^2} + n_v\E{}{\|\epsilon_v\|^2})\Pr(w^*|X,y)m(X,y)dXdydw^*\\
    &= n_v\E{}{\|\epsilon_v\|^2} + \int \E{X_v}{tr(X_vX_v^\intercal (\hat{w} - w^*)(\hat{w} - w^*)^\intercal )}\Pr(w^*|X,y)m(X,y)dXdydw^*\\
    &= n_v\E{}{\|\epsilon_v\|^2} + \int tr(\E{X_v}{X_vX_v^\intercal }(\hat{w} - w^*)(\hat{w} - w^*)^\intercal ))\Pr(w^*|X,y)m(X,y)dXdydw^*.
\end{align*}
Since $X_vX_v^\intercal $ is PSD, we can write $\E{X_v}{X_vX_v^\intercal } = AA^\intercal $ for some matrix A. We want to minimise the Bayesian risk with respect to $\hat{w}$:
\begin{align*}
    \nabla_{\hat{w}}B_\pi(\hat{w}) &= 2\int AA^\intercal (\hat{w} - w^*)\Pr(w^*|X,y)m(X,y)dXdydw^*\\
    &= 2AA^\intercal (\hat{w} - \E{}{w^*|X,y}).
\end{align*}
Since $AA^\intercal $ is PSD, the Hessian is PSD, so that the minimizer is obtained at
\begin{align*}
    \hat{w} = \E{}{w^*|X,y}.
\end{align*}
\end{proof}
For a Gaussian conjugate prior, we know that the mean of the posterior equals to the mode of the posterior. Thus, the Bayesian estimator equals the MAP estimator for a Gaussian prior. The following result states a slightly more general version of the statement.

\begin{theorem}\label{thm:map_suff}
    Given a well-specified task $(X,y)$ such that $\exists \ws{}, \epsilon, \text{s.t. } y = X^\intercal \ws{} + \epsilon$. Further assume that $\ws{}\sim \pi, \epsilon \sim N(0,\sigma^2 I)$, where $\pi$ is log-concave so that $\pi(w) = \exp(-f(w))$. The log-likelihood of $w$ given as $l(w)$ is then:
    \begin{align*}
        l(w) = -\frac{\|y-X^\intercal w\|^2}{2\sigma^2} - f(w).
    \end{align*}
    Define the MAP estimator as follows:
    \begin{align*}
        w_{MAP} = \wh = \max_w l(w).
    \end{align*}
    The MAP estimator is equal to the Bayesian estimator for expected validation loss, that is $\wh = w_{\text{Bayes}} = \E{}{\ws{}|X,y}$, if $f$ is convex and $\nabla^rf(\wh) = 0$ for $r>2$.
\end{theorem}
\begin{proof}
    Since $l(w)$ is concave,
    \begin{align}
        \nabla_w l(\wh) &=0\nonumber\\
        \implies \nabla f(\wh) + \frac{X(X^\intercal w-y)}{\sigma^2} &= 0.\label{step:wh_local_min}
    \end{align}
    We also know that for some normalization constant $Z$,
    \begin{align}
        \E{}{\ws{}|X,y} &= \frac{1}{Z}\int_{\R^d} w\exp\left(-\frac{\|y-X^\intercal w\|^2}{2\sigma^2} - f(w)\right)\dd{w}\nonumber\\
        &= \frac{1}{Z}\int_{\R^d} (\wh + t)\exp\left(-\frac{\|y-X^\intercal (\wh + t)\|^2}{2\sigma^2} - f(\wh + t)\right)\dd{t}\nonumber\\
        &= \wh + \frac{1}{Z}\int_{\R^d} t\exp\left(-\frac{\|y-X^\intercal (\wh + t)\|^2}{2\sigma^2} - f(\wh + t)\right)\dd{t}.\label{step:mean_int}
    \end{align}
    Where in the last step we use the fact that the likelihood integrates to $Z$. Now, expanding the first term inside the $\exp$,
    \begin{align}
        \frac{\|y-X^\intercal (\wh + t)\|^2}{2\sigma^2} &= \frac{\|X^\intercal \wh - y\|^2 + \|X^\intercal t\|^2 + 2t^\intercal X(X^\intercal \wh - y)}{2\sigma^2}\nonumber.
    \end{align}
    Using Taylor's expansion for the second term inside exp, and using the fact that $\nabla^rf(\wh) = 0$ for $r>2$:
    \begin{align}
        f(\wh + t) &= f(\wh) + t^\intercal \nabla f(\wh) + \frac{t^\intercal \nabla^2 f(\wh)t}{2}\nonumber.
    \end{align}
    We can now combine the above two equations with \ref{step:wh_local_min} to give us:
    \begin{align}
        \frac{\|y-X^\intercal (\wh + t)\|^2}{2\sigma^2} + f(\wh + t) &= \frac{\|X^\intercal \wh - y\|^2 + \|X^\intercal t\|^2}{2\sigma^2}\nonumber\\
        &+ f(\wh) + \frac{t^\intercal \nabla^2 f(\wh)t}{2}\nonumber.
    \end{align}
    Going back to Equation \ref{step:mean_int}, we can simplify using the above results as follows:
    \begin{align}
        \E{}{\ws{}|X,y} &= \wh + \exp\left(\frac{\|X^\intercal \wh - y\|^2}{2\sigma^2} + f(\wh)\right)\int_{\R^d} t\exp\left(-\frac{\|X^\intercal t\|^2}{2\sigma^2} - \frac{t^\intercal \nabla^2 f(\wh)t}{2}\right)\dd{t}\nonumber\\
        &= \wh.\nonumber
    \end{align}
    Where the last step follows from the symmetry of the integral around $0$.
\end{proof}
The following Corollary which is a direct consequence of the above theorem states that the optimal re-centered ridge regression parameters result in the Bayesian estimator. Thus, finding the optimal ridge regression parameter can be equivalently thought of as a $g$-modeling Empirical Bayes approach.

\begin{corollary}
    Given a prior on $\ws{} \sim Z\exp{\left(-\frac{\|\ws{}-\mu^*\|}{2\omega^2}\right)}$, for $\omega \in \R$, $\wh_{(\lambda,\mu)} = w_{\text{Bayes}} = \E{}{\ws{}|X,y}$ for $\lambda = \sigma^2/\omega^2, \mu = \mu^*$. Thus for appropriately chosen parameters, the re-centered ridge estimator corresponds to the Bayes estimator.
\end{corollary}
\textbf{Remark.} Note that since Theorem \ref{thm:map_suff} is valid for more general cases than just a Gaussian prior, we can derive similar results for other estimators that respect the form of the prior. For example, elastic net estimators when the prior is a mixture of a Gaussian and a Laplace distribution.

\section{Background}
In this section we cover some commonly known results on concentration of random numbers, as well as a common tool from learning theory, Rademacher Complexity.

We begin with Hoeffding's inequality, which shows that the mean of random variables concentrates exponentially fast around their mean.

\begin{theorem}[Hoeffding's inequality~\citep{all_of_stats}]\label{thm:Hoeffding}
For random numbers $X_1, \ldots, X_N$ sampled i.i.d., denote $\overline{X_N} = \frac{\sum X_i}{N}$ and $\E{}{X_i} = \mu$. The following hold given that $X_i\in [0,C]$:
\begin{enumerate}
    \item \begin{align*}
        \Pr(|\overline{X_N} - \mu| \geq t) \leq 2\exp{\left(\frac{-2Nt^2}{C^2}\right)}
    \end{align*}
    \item \begin{align*}
        \Pr(\overline{X_N} - \mu \geq t) \leq \exp{\left(\frac{-2Nt^2}{C^2}\right)}
    \end{align*}
    \item With probability $\geq 1-\delta$,
    \begin{align*}
        \overline{X_N} \leq \mu + C\sqrt{\frac{\ln{1/\delta}}{2N}}
    \end{align*}
\end{enumerate}
\end{theorem}
The following is used frequently in Rademacher complexity analyses, and shows that the value of a multi-variate function is concentrated around its expected value with a high probability.

\begin{theorem}[McDiarmid's Inequality~\citep{mlt_book}]\label{thm:mcd}
Given i.i.d.\ variables $X_1, \ldots, X_N$, such that $X_i\in \R \forall i \in [N]$, and a function $f:\R^N\rightarrow\R$ such that:
\begin{equation*}
    |f(x_1,\ldots,x_N) - f(x_1, \ldots, x_{k-1}, x_k',x_{k+1}, \ldots x_N)| \leq L_k.
\end{equation*}
That is, changing the $k$th element arbitrarily changes the value of the function by at most $L_k$. The following inequality holds:
\begin{align*}
    \Pr(|f(X_1,\ldots, X_N) - \E{X_1,\ldots,X_N}{f(X_1,\ldots, X_N)}| \geq t) \leq 2\exp{\left(\frac{-2t^2}{\sum L_k^2}\right)}.
\end{align*}
\end{theorem}

\begin{corollary}\label{cor:mcd}
    Given functions $l^i(\lambda, s), i\in[N]$ that take a parameter $\lambda$ and an input $s$ such that $l^i(\lambda, s) \leq C \forall \lambda,s,i$. For a set $S = \{S^{i}:i\in[N]\}$ of $N$ inputs, we define $l(\lambda,S) = \frac{1}{N}\sum l^i(\lambda,S^{(i)})$ as the average over $N$ inputs. If all inputs in $S$ are i.i.d., then with probability $\geq 1-\delta$,
    \begin{align}
        \sup_\lambda (\E{S'}{l(\lambda,S')} - l(\lambda,S)) &\leq \E{S}{\sup_\lambda (\E{S'}{l(\lambda,S')} - l(\lambda,S))} + C\sqrt{\frac{2\ln(2/\delta)}{N}}\nonumber\\
       &\leq \E{S,S'}{\sup_\lambda (l(\lambda,S') - l(\lambda,S))} + C\sqrt{\frac{2\ln(2/\delta)}{N}}\nonumber
    \end{align}
\end{corollary}
\begin{proof}
    Note that $\sup_\lambda (\E{S'}{l(\lambda,S')} - l(\lambda,S))$ is a function of $N$ i.i.d.\ variables by definition. Here $S'$ is a ``ghost sample" introduced to calculate expectation as is commonly done in literature. Further, changing one of these variables changes the function by at most $2C/N$. The statement follows from Theorem \ref{thm:mcd} by equating $f$ with $\sup_\lambda (\E{S'}{l(\lambda,S')} - l(\lambda,S))$ and $L_k = 2C/N$.
\end{proof}

The sample covariance matrix, $\hat{\Sigma}_n = n^{-1} XX^\intercal$ and its inverse $\hat{\Sigma}_n^{-1}$ are quantities that occur frequently in analyses of ridge regression. Below we give results from \cite{cov_mat_mourtada_22}, that allow us to bound the expected value of the inverse of the smallest singular value of $\hat{\Sigma}_n$ in terms of the distribution of the samples.
\begin{theorem}[Corollary 4 in \cite{cov_mat_mourtada_22}]\label{thm:exp_cov_inv}
    Consider the sample covariance matrix $\hat{\Sigma}_n = n^{-1}\sum XX^\intercal$, where $X\in\R^{d\times n}$. Assume that $X$ is sampled from a distribution such that $\E{}{xx^\intercal} = I_d$. Further assume that there exist constants $C\ge 1, \alpha\in(0,1]$ such that for any hyperplane $H$ in $\R^d$,
    \begin{align}
        \Pr(\mathrm{dist}(X,H)\le t)\le (Ct)^\alpha\; \forall t>0.\nonumber
    \end{align}
    Then, for $n\ge \max(6d/\alpha, 12/\alpha)$ and $1\leq q \leq \alpha n/12$,
    \begin{align}
        \E{}{|\max(1,\lambda_{min}(\hat{\Sigma}_n)^{-1})|^q}^{1/q} \leq 2^{1/q}C',\nonumber
    \end{align}
    where $C' = 3C^4e^{1+9/\alpha}$.
\end{theorem}

\begin{theorem}[Proposition 5 in \cite{cov_mat_mourtada_22}]\label{prop:ip_entry}
    If the entries $x^{(1)},\ldots, x^{(d)}$ are independent and have density bounded by $C_0$, and $\E{}{xx^\intercal} = I_d$, then for any hyperplane $H$ in $\R^d$,
    \begin{align*}
        \Pr(\mathrm{dist}(X,H)\le t)\le (Ct)^\alpha\; \forall t>0,
    \end{align*}
    for $\alpha = 1, C = 2\sqrt{2}C_0$.
\end{theorem}

\subsection{Rademacher Complexity}
In this section we will discuss Rademacher Complexity, which is a common tool from learning theory, and some important results used in our analysis. 

\begin{definition}[Empirical Rademacher Complexity]
    The Empirical Rademacher complexity of a function $l$ for given inputs $x_1,\ldots, x_n$ is given as:
    \begin{align}
        \mathcal{R} = \E{\sigma}{\frac{1}{n} \sum \sigma_i l(x_i)},\nonumber
    \end{align}
    where $\sigma_i$ are Rademacher random variables (i.e., they take values in $\{+1,-1\}$ with equal probability).
\end{definition}

The following is another popular result used to compute Empirical Rademacher Complexity of a fixed set of variables. 
\begin{theorem}[Khintchine's Inequality \citep{rademacher_series}]\label{thm:khintchine}
    For Rademacher random variables $\sigma^t$ and real numbers $x^t$, we have that
    \begin{align}
        \E{\sigma}{\Big\lvert\sum \sigma^tx^t\Big\rvert} \leq \left(\sum |x^t|^2\right)^{1/2}.\nonumber
    \end{align}
\end{theorem}

\section{A generalization of the Contraction Lemma}
The Contraction Lemma (Lemma \ref{thm:rad_lipschitz}) is a popular result used to simplify Rademacher complexity computations using the $L$-Lipschitzness of the loss function $l$. Below we present a more general result using a generalized version of Lipschitzness.

\begin{definition}[Lipschitzness in another function]\label{def:lips_general}
    A function $l:\mathcal{Z}\rightarrow\R$ is said to be Lipschitz in another function $g:\mathcal{Z}\rightarrow\R$ if:
    \begin{align*}
        l(a)-l(b) \leq L|g(a)-g(b)| \forall a,b \in \mathcal{Z}.
    \end{align*}
\end{definition}

\begin{theorem}\label{thm:rad_lips_general}
    Consider a class of functions $\mathcal{F}\subseteq\{f:\R^m\rightarrow\mathcal{Z}\}$ and two functions $l,g:\mathcal{Z}\rightarrow\R$ for some domain $\mathcal{Z}$ such that $l$ is $L$-Lipschitz in $g$. That is, $l(a)-l(b) \leq L|g(a)-g(b)|\;\forall a,b\in\mathcal{Z}$. We have the following bound on the empirical Rademacher complexity of $\{l\circ f:f\in\mathcal{F}\}$ for given inputs $x_1,\ldots,x_n$:
    \begin{align}
        \mathcal{R} = \E{\sigma}{\sup_{f\in\mathcal{F}}\frac{1}{n} \sum \sigma_i l(f(x_i))} \leq L \E{\sigma}{\sup_{f\in\mathcal{F}}\frac{1}{n} \sum \sigma_i g(f(x_i))}.\nonumber
    \end{align}
\end{theorem}
\begin{proof}
    \begin{align}
        n\mathcal{R} &= \E{\sigma}{\sup_{f\in\mathcal{F}} \sum \sigma_i l(f(x_i))} = \E{\sigma}{\sup_{f\in\mathcal{F}} (\sigma_1l(f(x_1)) + \sum_{i\neq1}\sigma_i l(f(x_i)))}\nonumber\\
        &= \frac{1}{2} \E{\sigma_2,\ldots,\sigma_n}{\sup_{f\in\mathcal{F}} (l(f(x_1)) + \sum_{i\neq1}\sigma_i l(f(x_i))) + \sup_{f\in\mathcal{F}} (-l(f(x_1)) + \sum_{i\neq1}\sigma_i l(f(x_i)))}\nonumber\\
        &= \frac{1}{2} \E{\sigma_2,\ldots,\sigma_n}{\sup_{f,f'\in\mathcal{F}} (l(f(x_1)) -l(f'(x_1)) + \sum_{i\neq1}\sigma_i l(f(x_i))  + \sum_{i\neq1}\sigma_i l(f'(x_i)))}\nonumber\\
        &\leq \frac{1}{2} \E{\sigma_2,\ldots,\sigma_n}{\sup_{f,f'\in\mathcal{F}} (L|g(f(x_1)) -g(f'(x_1))| + \sum_{i\neq1}\sigma_i l(f(x_i))  + \sum_{i\neq1}\sigma_i l(f'(x_i)))}.\label{step:use_lips}
    \end{align}
    In the last step we use the given expression, $l(a)-l(b) \leq L|g(a)-g(b)|\;\forall a,b\in\mathcal{Z}$. Note that we can now drop the absolute value surrounding $g(f(x_1)) -g(f'(x_1))$ so that:
    \begin{align}
        n\mathcal{R} &\leq \frac{1}{2} \E{\sigma_2,\ldots,\sigma_n}{\sup_{f,f'\in\mathcal{F}} (L(g(f(x_1)) -g(f'(x_1))) + \sum_{i\neq1}\sigma_i l(f(x_i))  + \sum_{i\neq1}\sigma_i l(f'(x_i)))}.\label{step:rem_lips_abs}
    \end{align}
    This is trivial if the $sup$ operator picks $f,f'$ such that $g(f(x_1)) \geq g(f'(x_1))$ in Equation \ref{step:use_lips}. If on the other hand the $sup$ operator picked $f,f'$ such that $g(f(x_1)) < g(f'(x_1))$ in Equation \ref{step:use_lips}, it can swap them in Equation \ref{step:rem_lips_abs}, resulting in the same value as replacing $f$ and $f'$ in the summation over $i\neq1$ does not change the expression. We can thus reduce this back to a Rademacher complexity computation as follows:
    \begin{align}
        n\mathcal{R} &\leq \E{\sigma}{\sup_{f} (\sigma_1Lg(f(x_1)) + \sum_{i\neq1}\sigma_i l(f(x_i)))}.
    \end{align}
    Proceeding similarly for all $i\neq 1$, we get the desired result.
\end{proof}

\begin{corollary}[Contraction Lemma~\citep{mlt_book}]\label{thm:rad_lipschitz}
    Let $l:\R\rightarrow\R$ be a Lipschitz function, that is, $l(a)-l(b) \leq L|a-b|\quad \forall a,b\in \R$. Let $\mathcal{F}$ be a class of functions $\mathcal{F}\subseteq\{f:\R^m\rightarrow\R\}$ that map into the domain of $l$. We have that the empirical Rademacher complexity of $\{l\circ f:f\in\mathcal{F}\}$ for given inputs $x_1,\ldots, x_n$ is upper bounded as:
    \begin{align}
        \mathcal{R} = \E{\sigma}{\sup_{f\in\mathcal{F}}\frac{1}{n} \sum \sigma_i l(f(x_i))} \leq L\E{\sigma}{\sup_{f\in\mathcal{F}}\frac{1}{n} \sum \sigma_i f(x_i)}.\nonumber
    \end{align}
\end{corollary}
\begin{proof}
    Follows from Theorem \ref{thm:rad_lips_general} by replacing $g$ with the identity function.
\end{proof}

Lipschitzness is a common assumption made for proving generalization bounds in literature. Lipschitzness usually follows from boundedness of the loss function, as we instantiate below for the squared loss.

\begin{proposition}\label{prop:lipschitz}
    For a squared loss function $l(y_p,y_t) = (y_p-y_t)^2$, boundedness implies Lipschitzness. That is, given that $l(x_v^\intercal \wh_\lambda,y_v) \leq C,  \forall x_v,y_v,X,y$,
    \begin{align*}
        |l(x_v^\intercal w_1,y_v) - l(x_v^\intercal w_2,y_v)| \leq 2\sqrt{C} |x_v^\intercal w_1 - x_v^\intercal w_2|.
    \end{align*}
\end{proposition}
\begin{proof}
    \begin{align}
        |l(x_v^\intercal w_1,y_v) - l(x_v^\intercal w_2,y_v)| &= |x_v^\intercal w_1 - x_v^\intercal w_2||x_v^\intercal w_1 - y_v + x_v^\intercal w_2 - y_v|\nonumber\\
        &\leq2\sqrt{C}|x_v^\intercal w_1 - x_v^\intercal w_2|.\nonumber
    \end{align}
    Since $|x_v^\intercal w_1 - y_v| \leq \sqrt{C}$.
\end{proof}

\section{Proofs for tuning Ridge Regression}\label{sec:proofs}
We start this Section by some definitions we will need for the proof. We will use the following elaborate definition of a problem instance (which sufficiently identifies a unique problem instance but not vice-versa). 
\begin{align}
    \el{S} = \{(X^t, f^t, \epsilon^t, X_v^t, \epsilon_v^t): X\in \R^{d \times n}, X^t_{v}\in \R^{d \times n_v}, \epsilon^t \in E^n, \epsilon_v^t \in E^{n_v}\}.\nonumber
\end{align}
This allows to define elaborate ordered set of training and validation examples as:
\begin{equation}\label{def:etr}
    \etr{S} = \{(X^t,f^t, \epsilon^t): X\in \R^{d \times n}, \epsilon^t \in E^n\},
\end{equation}
and,
\begin{equation}\label{def:eval}
    \eval{S} = \{(X_v^t, \epsilon_v^t): X^t_{v}\in \R^{d \times n_v}, \epsilon_v^t \in E^{n_v}\},
\end{equation}
Note that $\etr{S}\ewtimes \eval{S} = \el{S}$, where $\ewtimes$ takes the entry-wise composition of the ordered sets. Equivalent to the definition of the empirical validation loss we have:
\begin{align*}
    \lve(\lambda, \el{S}) = 
    \frac{1}{T} \sum_t \frac{1}{n_v} \sum_i l(X_{v}^{t(i)\intercal}\wh^t_{\lambda}(X^t,f^t(X^t, \epsilon^t)), f^t(X_v^t, \epsilon_v^t)).
\end{align*}
Note that $\lve(\lambda, \el{S}) = \lv(\lambda, S)$ as $\lve$ uses the elaborate form to find $y,y_v$ to compute empirical validation loss as in $\lv$. Further, $\E{S}{\lv(\lambda, S)} = \E{\el{S}}{\lve(\lambda, \el{S})}$. We will also be interested in the empirical expected validation loss, which for $y_v^{t(i)} = f(X_v^{\intercal t(i)},\epsilon_v^{t(i)})$ is given as:
\begin{align}
    \lev(\lambda, \eval{S})
    &= \frac{1}{n_vT} \sum_{t,i} \E{X,f,\epsilon}{l(X_{v}^{t(i)\intercal}\wh_{\lambda}, y_{v}^{t(i)})}\nonumber\\
    &= \E{\etr{S'}}{\lve(\lambda, \etr{S'}\ewtimes \eval{S})}\label{step:ev_exp_tr}.
\end{align}
Thus for a given $\eval{S}$, $\lev$ computes the expected validation loss over all possible sampling of training data.
 
In the well-specified linear case, we will overload the notation as follows. We will define the elaborate set of problem instances as:
\begin{align}
    \el{S} = \{(X^t,f_{\ws{t}}, \epsilon^t,X^t_{v}, \epsilon_v^t): X\in \R^{d \times n}, X^t_{v}\in \R^{d \times n_v}, w^{*t}\in \R^d, \epsilon^t \in \R^n, \epsilon_v^t \in \R^{n_v}\}.\nonumber
\end{align}
This allows to define elaborate set of training and validation examples as:
\begin{align*}
    \etr{S} = \{(X^t,f_{\ws{t}}, \epsilon^t): X\in \R^{d \times n}, w^{*t}\in \R^d, \epsilon^t \in \R^n\},
\end{align*}
and,
\begin{align*}
    \eval{S} = \{(X_v^t, \epsilon_v^t): X^t_{v}\in \R^{d \times n_v}, \epsilon_v^t \in \R^{n_v}\},
\end{align*}
Note again, that $\etr{S}\ewtimes \eval{S} = \el{S}$. Empirical validation loss can be re-written as:
\begin{align*}
    \lve(\lambda, \el{S}) = 
    \frac{1}{T} \sum_t \frac{1}{n_v} \sum_i l(X_{v}^{t(i)\intercal}\wh^t_{\lambda}(X^t,X^{t\intercal}\ws{t} + \epsilon^t), X_{v}^{t(i)\intercal}\ws{t} + \epsilon_{v}^{t(i)}).
\end{align*}

We present and prove the main lemmas used for proving Theorem \ref{thm:lips_tgt} below. We first start by upper-bounding the generalization error in terms of two different Rademacher complexities: Rademacher complexity of validation loss with fixed validation data and Rademacher complexity of expected validation loss over choice of training data. 

\begin{lemma}\label{lem:sup_to_rad_2}
    Given a bounded validation loss function, that is, given that $l(x_v^\intercal\wh_\lambda,y_v) \leq C,  \forall x_v,y_v,X,y,\lambda$. For any problem instance $S$ as defined in Equation \ref{def:prob}, with probability at least $1-\delta$, 
    \begin{align}
        \sup_\lambda l_v(\lambda) - l_{v}(\lambda,S) &\leq 2\E{\sigma, \etr{S}}{\sup_\lambda \frac{1}{n_vT} \sum_{t,i} \sigma^t l(X_v^{t(i)\intercal}\wh_\lambda^t,y_v^{t(i)})}\nonumber\\
        &+ 2\E{\sigma, \eval{S}}{\sup_\lambda \frac{1}{n_vT} \sum_{t,i} \sigma^{t(i)}\E{X,f,\epsilon}{l(X_v^{t(i)\intercal}\wh_\lambda,y_v^{t(i)})}}\nonumber\\
        &+ 2C\sqrt{\frac{2\ln(4/\delta)}{T}}.\nonumber
    \end{align}
    Where $y_v^{t(i)} = f^t(X_v^{\intercal t(i)},\epsilon_v^{t(i)})$, and $\sigma^{t} $ and $\sigma^{t(i)} $ are i.i.d. Rademacher variables.
\end{lemma}
\begin{proof}
    \begin{align}
    \sup_\lambda l_v(\lambda) - l_{v}(\lambda,S) = &\sup_\lambda (l_v(\lambda) - l_{ev}(\lambda, \eval{S}) + l_{ev}(\lambda, \eval{S}) - l_{v}(\lambda,S))\nonumber\\
    \leq &\sup_\lambda (l_v(\lambda) - l_{ev}(\lambda, \eval{S})) + \sup_\lambda (l_{ev}(\lambda, \eval{S}) - l_{v}(\lambda,S))\label{eq:sup_decompose_2}.
\end{align}
Note that $l_v(\lambda)$ is the expected value of $\lev(\lambda, \eval{S})$ over sampling of $\eval{S}$, whereas $l_{ev}(\lambda, \eval{S})$ is the average over $n_vT$ samples of the form $x_v,\epsilon_v$, where the $(tn_v+i)^\text{th}$ sample for $t\in[T], i\in[n_v]$ becomes the $i^\text{th}$ validation example for the $t^\text{th}$ task. Thus, by replacing each $l^i$ in Corollary \ref{cor:mcd} with $\lev(\lambda,\eval{S})$ we get that with probability $\geq 1-\delta$,
\begin{equation}\label{eq:mcd_tr_2}
    \sup_\lambda (l_v(\lambda) - \lev(\lambda, \eval{S})) \leq \E{\eval{S}, \eval{S'}}{\sup_\lambda (\lev(\lambda,\eval{S'}) - \lev(\lambda, \eval{S}))} +  C\sqrt{\frac{2\ln(4/\delta)}{n_vT}}.
\end{equation}

Similarly, for a fixed $\eval{S}$ we can view $l_v(\lambda, S)$ as an average over $T$ samples of training data. And we can view $\lev(\lambda,\eval{S})$ as the expected value of $l_v(\lambda, S)$ over the sampling of $\etr{S}$. Thus we can replace each $l^i(\lambda,.)$ in Corollary \ref{cor:mcd} with $\tilde{l}_v(\lambda, . \times \eval{S}^i)$, where $\eval{S}^i$ is the $i$th instance $\eval{S}$, to obtain that with probability $\geq 1-\delta/2$,
\begin{align}\label{eq:mcd_val_2}
    \sup_\lambda (l_{ev}(\lambda, \eval{S}) - l_{v}(\lambda,S)) &\leq \E{\etr{S}, \etr{S'}}{\sup_\lambda \lve(\lambda, \etr{S'}\times\eval{S}) - \lve(\lambda, \etr{S}\times\eval{S})}\nonumber\\
    &+ C\sqrt{\frac{2\ln(4/\delta)}{T}}.
\end{align}

In order to upper bound the unknown term in Equation \ref{eq:mcd_tr_2}, we note that we can arbitrarily swap the $(tn_v + i)^{\text{th}}$ validation instances between $\eval{S}$ and $\eval{S'}$ without changing the expectation. In fact, we can do this for all $(t,i)\in R \subseteq [T]\times [n_v]$ for any arbitrary set $R$. This allows us to reduce the term to a Rademacher complexity. We show this below where we denote $y_v^{t(i)} = f(X_v^{\intercal t(i)},\epsilon_v^{t(i)})$ and $y_v^{\prime t(i)} = f(X_v^{\prime \intercal t(i)},\epsilon_v^{\prime t(i)})$:

\begin{align}
    &\Eml{\eval{S}, \eval{S'}}\biggl[\sup_\lambda \lev(\lambda, \eval{S'}) - \lev(\lambda, \eval{S})\biggr]\nonumber\\
    &= \Eml{\eval{S}, \eval{S'}}\left[\sup_\lambda \frac{1}{n_vT} \sum_{t,i} \E{X,f,\epsilon}{l(X_v^{\prime t(i)\intercal}\wh_\lambda,y_v^{\prime t(i)})} - \frac{1}{n_vT} \sum_{t,i} \E{X,f,\epsilon}{l(X_v^{t(i) \intercal}\wh_\lambda,y_v^{ t(i)})}\right]\nonumber\\
    &= \mathbb{E}_{\eval{S}, \eval{S'}}\biggl[\sup_\lambda \frac{1}{n_vT} \sum_{t,i \notin R} \E{X,f,\epsilon}{l(X_v^{\prime t(i)\intercal}\wh_\lambda,y_v^{\prime t(i)})} + \frac{1}{n_vT} \sum_{t,i \in R} \E{X,f,\epsilon}{l(X_v^{\prime t(i)\intercal}\wh_\lambda,y_v^{\prime t(i)})}\nonumber\\
    &\quad - \frac{1}{n_vT} \sum_{t,i \notin R} \E{X,f,\epsilon}{l(X_v^{t(i)\intercal}\wh_\lambda,y_v^{t(i)})} - \frac{1}{n_vT} \sum_{t,i \in R} \E{X,f,\epsilon}{l(X_v^{t(i)\intercal}\wh_\lambda,y_v^{t(i)})}\biggr]\nonumber\\
    &= \mathbb{E}_{\eval{S}, \eval{S'}}\biggl[\sup_\lambda \frac{1}{n_vT} \sum_{t,i \notin R} \E{X,f,\epsilon}{l(X_v^{\prime t(i)\intercal}\wh_\lambda,y_v^{\prime t(i)})} + \frac{1}{n_vT} \sum_{t,i \in R} \E{X,f,\epsilon}{l(X_v^{t(i)\intercal}\wh_\lambda,y_v^{ t(i)})}\nonumber\\
    &\quad - \frac{1}{n_vT} \sum_{t,i \notin R} \E{X,f,\epsilon}{l(X_v^{t(i)\intercal}\wh_\lambda,y_v^{t(i)})} - \frac{1}{n_vT} \sum_{t,i \in R} \E{X,f,\epsilon}{l(X_v^{\prime t(i)\intercal}\wh_\lambda,y_v^{\prime t(i)})}\biggr]\nonumber\\
    &= 2\E{\sigma, \eval{S}}{\sup_\lambda \frac{1}{n_vT} \sum_{t,i} \sigma^{t(i)}\E{X,f,\epsilon}{l(X_v^{t(i)\intercal}\wh_\lambda,y_v^{t(i)})}}\label{eq:val_rad}.
\end{align}
The equality in the second last step holds because of symmetry due to expectation. In the last equation we introduce Rademachar variables for each value of $t$ and $i$. Thus we are able to upper bound the unknown term in Equation \ref{eq:mcd_tr_2} by the Rademacher complexity of the class of functions defined as the expected value of validation error over sampling of training tasks, across different values of $\lambda$.

Similarly, in Equation \ref{eq:mcd_val_2} we note that we can arbitrarily swap the $t^{th}$ training instances between $\etr{S}$ and $\etr{S'}$ without changing the expectation. In fact, we can do this for all $t\in R \subseteq [T]$ for any arbitrary set $R$. This allows us to reduce the term to a Rademacher complexity. We show this below where we denote $y_v^{t(i)} = f^t(X_v^{\intercal t(i)},\epsilon_v^{t(i)})$ and $y_v^{\prime t(i)} = f^t(X_v^{\prime \intercal t(i)},\epsilon_v^{\prime t(i)})$.
\begin{align}
    \Eml{\etr{S}, \etr{S'}}&\left[\sup_\lambda l_v(\lambda, \etr{S'}\times\eval{S}) - l_{v}(\lambda,\etr{S}\times\eval{S})\right] \nonumber\\
    &=\E{\etr{S}, \etr{S'}}{\sup_\lambda \frac{1}{T} \sum_t \frac{1}{n_v} \sum_i l(X_v^{t(i)\intercal}\wh_\lambda^{\prime t},y_v^{\prime t(i)}) - \frac{1}{T} \sum_t \frac{1}{n_v} \sum_i l(X_v^{t(i)\intercal}\wh_\lambda^{t},y_v^{t(i)})}\nonumber\\
    &= \Eml{\etr{S}, \etr{S'}}\left[\sup_\lambda \frac{1}{n_vT} \sum_{t\notin R}\sum_i l(X_v^{t(i)\intercal}\wh_\lambda^{\prime t},y_v^{\prime t(i)}) + \frac{1}{n_vT} \sum_{t\in R}\sum_i l(X_v^{t(i)\intercal}\wh_\lambda^{\prime t},y_v^{\prime t(i)})\right.\nonumber\\
    &\left.- \frac{1}{n_vT} \sum_{t\notin R}\sum_i l(X_v^{t(i)\intercal}\wh_\lambda^{t},y_v^{t(i)}) - \frac{1}{n_vT} \sum_{t \in R}\sum_i l(X_v^{t(i)\intercal}\wh_\lambda^{t},y_v^{t(i)})\right]\nonumber\\
    &= \Eml{\etr{S}, \etr{S'}}\left[\sup_\lambda \frac{1}{n_vT} \sum_{t\notin R}\sum_i l(X_v^{t(i)\intercal}\wh_\lambda^{\prime t},y_v^{\prime t(i)}) + \frac{1}{n_vT} \sum_{t\in R}\sum_i l(X_v^{t(i)\intercal}\wh_\lambda^{t},y_v^{t(i)})\right.\nonumber\\
    &\left.- \frac{1}{n_vT} \sum_{t\notin R}\sum_i l(X_v^{t(i)\intercal}\wh_\lambda^{t},y_v^{t(i)}) - \frac{1}{n_vT} \sum_{t \in R}\sum_i l(X_v^{t(i)\intercal}\wh_\lambda^{\prime t},y_v^{\prime t(i)})\right]\nonumber\\
    &= 2\E{\sigma, \etr{S}}{\sup_\lambda \frac{1}{n_vT} \sum_{t,i} \sigma^{t}l(X_v^{t(i)\intercal}\wh_\lambda^t,y_v^{t(i)})}\label{eq:task_rad_2}.
\end{align}
Similar to before, the equality in the second last step holds because of symmetry due to expectation. In the last step, we introduce Rademachar variables for each value of $t$. Thus we are able to upper bound the unknown term in Equation \ref{eq:mcd_val_2} by the Rademacher complexity of the class of functions defined as the empirical validation loss given fixed validation set, across different values of $\lambda$.

Since Equations \ref{eq:mcd_tr_2} and \ref{eq:mcd_val_2} hold with probability $\geq 1-\delta/2$ each, both equations hold with probability $\geq 1-\delta$ by a union bound. We get the desired result by combining equations \ref{eq:sup_decompose_2}, \ref{eq:mcd_tr_2}, \ref{eq:mcd_val_2}, \ref{eq:val_rad}, \ref{eq:task_rad_2}, and further noting that $C\sqrt{\frac{2\ln(4/\delta)}{T}} \geq C\sqrt{\frac{2\ln(4/\delta)}{n_vT}}$.
\end{proof}

In the following we give an upper bound on the expectation with respect to sampling of the validation set, of the Rademacher complexity of the expected value of validation error over sampling of training tasks in terms of distribution of the outputs $y$.

\begin{lemma}\label{lem:S_tr_2}
    Given a validation loss function that satisfies Assumptions \ref{assum:bound} and \ref{assum:lips} given in Section \ref{sec:ridge}, and $\eval{S}$ as defined in Equation \ref{def:eval} we get that (where we denote $y_v^{t(i)} = f(x_v^{t(i)},\epsilon_v^{t(i)})$ and $\sigma^{t(i)}$ are i.i.d. Rademacher variables): 
   \begin{align}
       \Eml{\sigma, \eval{S}}\Biggl[\sup_\lambda \frac{1}{n_vT} \sum_{t,i} \sigma^{t(i)} \Eml{X,f,\epsilon}\bigg[l(X_v^{t(i)\intercal}\wh_\lambda&,y_v^{t(i)})\bigg]\Biggr] \nonumber\\
       &\leq \frac{L}{\sqrt{n_vT}} \sqrt{\E{x_v}{\|x_v\|^2}}\E{X,y}{\|y\|/\sqrt{V(XX^\intercal)}}.\nonumber 
   \end{align}
\end{lemma}
\begin{proof}
We define $\mathcal{R}$ as below and use Lipschitzness to upper bound it as a simpler Rademacher complexity term:
\begin{align}
    \mathcal{R} &= \frac{1}{n_vT} \E{\sigma}{\sup_\lambda \sum_t\sum_i \sigma^{t(i)} \E{X,f,\epsilon}{l(X_v^{t(i)\intercal}\wh_\lambda,y_v^{t(i)})}} \nonumber\\
    &\leq \frac{L}{n_vT} \E{\sigma,X,f,\epsilon}{\sup_\lambda \sum_t\sum_i \sigma^{t(i)} X^{t(i)\intercal}_{v} \wh_\lambda}\nonumber\\
    &= \frac{L}{n_vT} \E{\sigma,X,f,\epsilon}{\sup_\lambda \left(\sum_t\sum_i \sigma^{t(i)} X^{t(i)}_{v}\right)^\intercal \wh_\lambda}\nonumber\\
    &\leq \frac{L}{n_vT} \E{\sigma,X,f,\epsilon}{\sup_\lambda \|\sum_t\sum_i \sigma^{t(i)} X^{t(i)}_{v}\| \|\wh_\lambda\|}\quad \text{(Cauchy-Schwartz inequality)}\nonumber\\
    &= \frac{L}{n_vT} \E{\sigma}{\|\sum_{t,i} \sigma^{t(i)} X^{t(i)}_{v}\|}\E{X,f,\epsilon}{\sup_\lambda \|\wh_\lambda\|}\nonumber
\end{align}
To bound the first Rademacher term in the product, we proceed as follows:
\begin{align}
    \E{\sigma}{\|\sum_{t,i} \sigma^{t(i)} X^{t(i)}_{v}\|} &\leq \sqrt{\E{\sigma}{\|\sum_{t,i} \sigma^{t(i)} X^{t(i)}_{v}\|^2}}\nonumber\\
    &= \sqrt{\E{\sigma}{\sum_{t,i} \|X_v^{t(i)}\|^2 + \sum_{(t_1,i_1) \neq (t_2,i_2)} \sigma^{t_1(i_1)}\sigma^{t_2(i_2)} X^{t_2(i_2)\intercal}_{v} X^{t_1(i_1)}_{v}}}\nonumber\\
    &= \sqrt{\sum_{t,i}\|X_v^{t(i)}\|^2}\nonumber.
\end{align}
Taking an expectation over validation set we find that,
\begin{align}
    \E{\eval{S}}{\sqrt{\sum_{t,i}\|X_v^{t(i)}\|^2}} &\leq \sqrt{\E{\eval{S}}{\sum_{t,i}\|X_v^{t(i)}\|^2}}\nonumber\\
    &= \sqrt{n_vT}\sqrt{\E{x_v}{\|x_v\|^2}}.\label{step:S_tr_2_xvub}
\end{align}
Since each validation example is sampled i.i.d.

It remains to upper bound the second term, which is $\E{X,f,\epsilon}{\sup_\lambda \|\wh_\lambda\|}$. Note that, if the singular values of $X$ are $s_1,\ldots, s_d$, then the singular values of $(XX^\intercal + \lambda I)^{-1}X$ are $\frac{s_i}{s_i^2 + \lambda}$ respectively for each $i$. Using the fact that $\left|\frac{s_i}{s_i^2 + \lambda}\right| \leq 1/|s_i|$ if $s_i \neq 0$, we obtain the following upper bound on $\|\wh_\lambda\|$:
\begin{align}
    \wh_\lambda &= (XX^\intercal + \lambda I)^{-1}Xy\nonumber\\
    \implies \|\wh_\lambda\|^2 &\leq \|(XX^\intercal + \lambda I)^{-1}X\|_\infty^2 \|y\|^2\quad \text{(definition of $\infty$-norm)}\nonumber\\
    &\leq \|(XX^\intercal)^{-1}X\|_\infty^2 \|y\|^2 \nonumber\\
    &= \|y\|^2/V(XX^\intercal)\quad \text{(since eigenvalues of $XX^\intercal$ are $s_1^2,\ldots,s_d^2$)}.\label{eq:w_norm_max}
\end{align}
Remember that $V(M)$ as the smallest non-zero singular value of $M$. This allows us to write,
\begin{align*}
    \E{X,f,\epsilon}{\sup_\lambda \|\wh_\lambda\|} \leq \E{}{\|y\|/\sqrt{V(XX^\intercal)}}.
\end{align*}

Combining these inequalities yields the desired result.
\end{proof}

We now show an upper bound on the expected Rademacher complexity of validation loss given fixed validation data in terms of the distribution of inputs $x$.

\begin{lemma}\label{lem:S_v_2}
   Given a validation loss function that satisfies Assumptions \ref{assum:bound} and \ref{assum:lips} given in Section \ref{sec:ridge}, and $\eval{S}$ as defined in Equation \ref{def:eval}, the following holds with probability at least $1-\delta$ (where we denote $y_v^{t(i)} = f^t(x_v^{t(i)},\epsilon_v^{t(i)})$, and $\sigma^t$ are i.i.d. Rademacher variables):
   \begin{align}
       \Eml{\sigma, \etr{S}} \Biggl[\sup_\lambda \frac{1}{n_vT} \sum_{t,i} \sigma^t {l(X_v^{t(i)\intercal}\wh_\lambda^t,y_v^{t(i)})}\Biggr]\leq \frac{ML\Lambda_D^T
       }{\sqrt{T}}\E{}{\|x_v\|}
        + \frac{MLb_v\Lambda_D^T}{\sqrt{n_vT}}\sqrt{\frac{\log(T/\delta)}{2}}.\nonumber
   \end{align}
   where $M^2 = \max\|Xy\|^2$, $b_v^2 = \max \|x_v\|^2$ and $\Lambda_D^T = \E{X}{\max_t 1/V(X^tX^{t\intercal})}$. 
\end{lemma}
\begin{proof}
We first note that Lipschitzness of the loss function implies Lipschitzness of the sum of the loss function over different examples:
\begin{align}
    l(x_v^\intercal\wh_1, y_v) - l(x_v^\intercal\wh_2, y_v) &\leq L|x_v^\intercal\wh_1 - x_v^\intercal\wh_2|\nonumber\\
    \implies \sum_i{l(X_v^{t(i)\intercal}\wh_{\lambda_1}, y_v^{t(i)}) - l(X_v^{t(i)\intercal}\wh_{\lambda_2}, y_v^{t(i)})} &\leq L\sum_{i}|X_v^{t(i)\intercal}(\wh_{\lambda_1} - \wh_{\lambda_2})|\nonumber
\end{align}
Using Lipschitzness (Theorem \ref{thm:rad_lips_general}):
\begin{align}
    \E{\sigma, \etr{S}}{\sup_\lambda \frac{1}{n_vT} \sum_{t,i} \sigma^t {l(X_v^{t(i)\intercal}\wh_\lambda^t,y_v^{t(i)})}}&\leq \frac{L}{n_vT}\E{\sigma, \etr{S}}{\sup_\lambda \sum_{t} \sigma^t (\sum_i X_v^{t(i)\intercal}\wh_\lambda^t)}.\label{step:task_rc_l2}
\end{align}
In order to derive a tight upper bound on the above Rademacher complexity, we introduce a new technique, where we argue lipschitzness of $x_v^\intercal\wh_\lambda$ in another function as follows.

If the SVD of $X = U_1PU_2^\intercal$, define $1_P$ as a diagonal matrix such that $(1_P)_{ii} = \mathbf{1}[P_{ii} \neq 0]$. So, $X = U_11_PU_1^\intercal X$. We can use this to argue Lipschitzness of  $x_{v}^{\intercal}\wh_{\lambda}$:
\begin{align*}
    x_v^\intercal\wh_{\lambda_1} - x_v^\intercal\wh_{\lambda_2} &= x_v^\intercal((XX^\intercal + \lambda_1 I)^{-1} - (XX^\intercal  + \lambda_2 I)^{-1})Xy\nonumber\\
    &= x_v^\intercal((XX^\intercal + \lambda_1 I)^{-1} - (XX^\intercal  + \lambda_2 I)^{-1})U_11_PU_1^\intercal Xy\nonumber\\
    &\leq\|x_v^\intercal((XX^\intercal + \lambda_1 I)^{-1} - (XX^\intercal + \lambda_2 I)^{-1})U_11_PU_1^\intercal\|\|Xy\|\nonumber\\
    &\leq\|x_v^\intercal\|\|((XX^\intercal + \lambda_1 I)^{-1} - (XX^\intercal + \lambda_2 I)^{-1})U_11_PU_1^\intercal\|_\infty\|Xy\|.
\end{align*}
Now we see that the SVD of $(XX^\intercal + \lambda I)^{-1}$ is $U_1MU_1^\intercal$ for some positive-definite diagonal matrix $M$. The non-zero singular values of $((XX^\intercal + \lambda_1 I)^{-1} - (XX^\intercal + \lambda_2 I)^{-1})U_11_PU_1^\intercal$ are $\frac{\lambda_2 - \lambda_1}{(e_i + \lambda_1)(e_i + \lambda_2)}$ if $e_i$ are the non-zero eigenvalues of $XX^\intercal$. Remember that  $V(XX^\intercal)$ is the smallest non-0 eigenvalue of $XX^\intercal$, and define $V^T = \min_t V(X^tX^{t\intercal})$ to see that $\|((XX^\intercal + \lambda_1 I)^{-1} - (XX^\intercal + \lambda_2 I)^{-1})U_11_PU_1^\intercal\|_\infty = \frac{1}{V(XX^\intercal) + \lambda_1} - \frac{1}{V(XX^\intercal) + \lambda_2}$. Thus,
\begin{align}
    x_v^\intercal\wh_{\lambda_1} - x_v^\intercal\wh_{\lambda_2} &\leq \|x_v^\intercal\|\|((XX^\intercal + \lambda_1 I)^{-1} - (XX^\intercal + \lambda_2 I)^{-1})U_11_PU_1^\intercal\|_\infty\|Xy\|\nonumber\\
    &= \|x_v\|\left|\frac{1}{V(XX^\intercal) + \lambda_1} - \frac{1}{V(XX^\intercal) + \lambda_2}\right|\|Xy\| \nonumber\\ 
    &\leq \|x_v\|\left|\frac{1}{V^T + \lambda_1} - \frac{1}{V^T + \lambda_2}\right|\|Xy\|.\label{eq:lamb_lips}
\end{align}
This shows the Lipschitzness of $x_v^\intercal\wh_\lambda$ in terms of $\frac{1}{V^T + \lambda}$ in line with Definition \ref{def:lips_general}. Using this Lipschitzness (Theorem \ref{thm:rad_lips_general}) in Equation \ref{step:task_rc_l2},
\begin{align}
    \Eml{\sigma, \etr{S}}\biggl[\sup_\lambda \frac{1}{n_vT} \sum_{t,i} \sigma^t &{l(X_v^{t(i)\intercal}\wh_\lambda^t,y_v^{t(i)})}\biggr]\leq \frac{L}{n_vT}\E{\sigma, \etr{S}}{\sup_\lambda \sum_{t} \sigma^t \left(\sum_i X_v^{t(i)\intercal}\wh_\lambda^t\right)}\nonumber\\
    &\leq \frac{L}{n_vT}\E{\sigma, \etr{S}}{\sup_\lambda \sum_{t} \sigma^t \frac{\sum_i\|X_v^{t(i)}\|\|X^ty^t\|}{V^T + \lambda}}\nonumber\\
    &\leq \frac{L}{n_vT}\E{\sigma, \etr{S}}{\left(\sup_\lambda \frac{1}{V^T + \lambda} \right) \left|\sum_{t} \sigma^t \left(\sum_i\|X_v^{t(i)}\|\|X^ty^t\|\right)\right|}\nonumber\\
    &= \frac{L}{n_vT}\E{\etr{S}}{\left(\sup_\lambda \frac{1}{V^T + \lambda} \right) \E{\sigma}{\left|\sum_{t} \sigma^t \left(\sum_i\|X_v^{t(i)}\|\|X^ty^t\|\right)\right|}}\nonumber\\
    &\leq \frac{L}{n_vT}\E{\etr{S}}{\frac{\sqrt{\sum_t(\sum_i \|X_v^{t(i)}\|)^2\|X^ty^t\|^2}}{V^T}}.\nonumber
\end{align}
We use Khintchine's inequality (Theorem \ref{thm:khintchine}) and set $\lambda = 0$ in the last step. To get the desired result we need to simplify the numerator. We note that for any $t\in[T]$, with probability $\geq 1-\delta$ by Hoeffding inequality (Theorem \ref{thm:Hoeffding}),
\begin{align*}
    \sum_i \|X_v^{t(i)}\| &\leq n_v \E{}{\|x_v\|} + b_v\sqrt{\frac{n_v\log(1/\delta)}{2}}.
\end{align*}
By a union bound over all tasks, we get that for all tasks $t\in[T]$, with probability $\geq 1-\delta$
\begin{align*}
    \sum_i \|X_v^{t(i)}\| &\leq n_v \E{}{\|x_v\|} + b_v\sqrt{\frac{n_v\log(T/\delta)}{2}}\\
    \implies \left(\sum_i \|X_v^{t(i)}\|\right)^2 &\leq \left(n_v \E{}{\|x_v\|} + b_v\sqrt{\frac{n_v\log(T/\delta)}{2}}\right)^2.
    %2n_v^2 \E{}{\|x_v\|}^2 + n_vb_v^2\log(T/\delta).
\end{align*}
We sum the above over all tasks and note from definition that $\|X^ty^t\|^2\leq M^2$ to get,
\begin{align*}
    \sum_t\left(\sum_i \|X_v^{t(i)}\|\right)^2\|X^ty^t\|^2 &\leq TM^2\left(n_v \E{}{\|x_v\|} + b_v\sqrt{\frac{n_v\log(T/\delta)}{2}}\right)^2\\
    \implies \sqrt{\sum_t\left(\sum_i \|X_v^{t(i)}\|\right)^2\|X^ty^t\|^2} &\leq n_v\sqrt{T} M\E{}{\|x_v\|} + b_vM\sqrt{\frac{n_vT\log(T/\delta)}{2}}.
\end{align*}
This gives the desired result.
\end{proof}

Below we present an additional lemma that is tighter than Lemma \ref{lem:S_tr_2} for the well-specified case. We then restate and prove Theorem \ref{thm:lips_ws} using this lemma. 

\begin{lemma}\label{lem:S_tr_ws}
    Given a validation loss function that satisfies Assumptions \ref{assum:bound} and \ref{assum:lips} given in Section \ref{sec:ridge}, $\eval{S}$ as defined in Equation \ref{def:eval}, the following holds for well-specified linear tasks (where we denote $y_v^{t(i)} = f^t(x_v^{t(i)},\epsilon_v^{t(i)})$, and $\sigma^{t(i)}$ are i.i.d. Rademacher variables):
   \begin{align}
       \Eml{\sigma, \eval{S}}\Biggl[\sup_\lambda \frac{1}{n_vT} \sum_{t,i} &\sigma^{t(i)} \Eml{X,f,\epsilon}\bigg[l(X_v^{t(i)\intercal}\wh_\lambda,y_v^{t(i)})\bigg]\Biggr]\leq\nonumber\\
       &\frac{L}{\sqrt{n_vT}} \sqrt{\E{x_v}{\|x_v\|^2}}\E{}{\|\ws{}\| + \|\epsilon\|/\sqrt{V(XX^\intercal)}}.
   \end{align}
\end{lemma}
\begin{proof}
We will proceed similarly to Lemma \ref{lem:S_tr_2} till Equation \ref{step:S_tr_2_xvub}. We now need to upper bound $\E{X,f,\epsilon}{\sup_\lambda \|\wh_\lambda\|}$ using the well-specified assumption. If we denote $y=X^\intercal\ws{} + \epsilon$, we see that,
\begin{align}
    \wh_\lambda &= (XX^\intercal + \lambda I)^{-1}Xy\nonumber\\
    &= (XX^\intercal + \lambda I)^{-1}(XX^\intercal \ws{} + X\epsilon)\nonumber\\
    &= (XX^\intercal + \lambda I)^{-1}XX^\intercal\ws{} + (XX^\intercal + \lambda I)^{-1}X\epsilon\nonumber\\
    \implies \|\wh_\lambda\| &\leq \|(XX^\intercal + \lambda I)^{-1}XX^\intercal\ws{}\| + \|(XX^\intercal + \lambda I)^{-1}X\epsilon\|\nonumber\\
    &\leq \|(XX^\intercal + \lambda I)^{-1}XX^\intercal\|_\infty\|\ws{}\| + \|(XX^\intercal + \lambda I)^{-1}X\|_\infty\|\epsilon\|\nonumber\\
    \implies \sup_\lambda \|\wh_\lambda\| &\leq \|\ws{}\| + \|\epsilon\|/\sqrt{V(XX^\intercal)}\nonumber\\
    \implies \E{X,\ws{},\epsilon}{\sup_\lambda \|\wh_\lambda\|} &\leq \E{}{\|\ws{}\| + \|\epsilon\|/\sqrt{V(XX^\intercal)}}.\nonumber
\end{align}
Where in the second last step, we set $\lambda \rightarrow 0$ using the fact that the eigenvalues of $(XX^\intercal + \lambda I)^{-1}XX^\intercal$ are $\frac{\lambda_i}{\lambda_i + \lambda}$ and the singular values of $(XX^\intercal + \lambda I)^{-1}X$ are $\frac{\sqrt{\lambda_i}}{\lambda_i + \lambda}$ if the eigenvalues of $XX^\intercal$ are $\lambda_i$ respectively.

Proceeding through the the rest of the steps similarly to Lemma \ref{lem:S_tr_2}, we obtain the desired result.
\end{proof}

\begin{theorem}[Proof of Theorem \ref{thm:lips_ws}]
    Given a loss function that satisfies Assumptions \ref{assum:bound} and \ref{assum:lips} in Section \ref{sec:ridge}, and tasks that are well-specified linear maps, the expected validation loss error using the ERM estimator defined in Equation \ref{def:lamb_erm} is bounded with probability $\geq 1-\delta$ as:
    \begin{align}
        l_v(\lambda_{ERM}) - l_v(\lambda^*) &\leq \frac{2ML\Lambda_D^T
        }{\sqrt{T}}\E{}{\|x_v\|} + \frac{2L}{\sqrt{n_vT}} \sqrt{\E{x_v}{\|x_v\|^2}}\E{}{\|\ws{}\| + \|\epsilon\|/\sqrt{V(XX^\intercal)}}\nonumber\\
       &\quad + \frac{2MLb_v\Lambda_D^T}{\sqrt{n_vT}}\sqrt{\frac{\log(4T/\delta)}{2}} + 5C\sqrt{\frac{\ln(16/\delta)}{2T}}.\nonumber
    \end{align}
    Here $M^2 = \max\|Xy\|^2$, $b_v^2 = \max \|x_v\|^2$ $\Lambda_D^T = \E{}{\max_t 1/V(X^tX^{t\intercal})}$. 
\end{theorem}
\begin{proof}
Proceeding similarly to Theorem \ref{thm:lips_tgt},
\begin{align}
    l_v(\lambda_{ERM}) - l_v(\lambda^*) \leq \sup_\lambda (l_v(\lambda) - l_{v}(\lambda,S))+ C\sqrt{\frac{\ln(1/\delta)}{2T}}\nonumber
\end{align}
with probability at least $1-\delta$. Using Lemma \ref{lem:sup_to_rad_2} again, we break the first error term into error induced from a finite sampling of validation examples, and error induced from finite sampling of training data to get that with probability $\geq 1-\delta$:
\begin{align}
        \sup_\lambda l_v(\lambda) - l_{v}(\lambda,S) &\leq 2\E{\sigma, \etr{S}}{\sup_\lambda \frac{1}{n_vT} \sum_{t,i} \sigma^t l(X_v^{t(i)\intercal}\wh_\lambda^t,y_v^{t(i)})}\nonumber\\
        &+ 2\E{\sigma, \eval{S}}{\sup_\lambda \frac{1}{n_vT} \sum_{t,i} \sigma^{t(i)}\E{X,f,\epsilon}{l(X_v^{t(i)\intercal}\wh_\lambda,y_v^{t(i)})}}\nonumber\\
        &+ 2C\sqrt{\frac{2\ln(4/\delta)}{T}}.\nonumber
\end{align}

In Lemma \ref{lem:S_v_2}, we see that $\sum_i l(X_v^{t(i)\intercal}\wh_\lambda^t,y_v^{t(i)})$ is Lipschitz in $\frac{1}{V^T + \lambda}$ for fixed $y_v^{t(i)}$. Here $V^T = \min_t V(X^tX^{t\intercal})$ and $V(.)$ is the smallest non-zero eigenvalue of the matrix. We use this Lipschitzness to bound the first term with probability $\geq 1-\delta$ as:
\begin{align}
       \Eml{\sigma, \etr{S}} \Biggl[\sup_\lambda \frac{1}{n_vT} \sum_{t,i} \sigma^t {l(X_v^{t(i)\intercal}\wh_\lambda^t,y_v^{t(i)})}\Biggr]\leq \frac{ML\Lambda_D^T
       }{\sqrt{T}}\E{}{\|x_v\|}
        + \frac{MLb_v\Lambda_D^T}{\sqrt{n_vT}}\sqrt{\frac{\log(T/\delta)}{2}}.\nonumber
\end{align}
    
Lemma \ref{lem:S_tr_ws} uses Lipschitzness of the loss function to upper bound the second term with probability $\geq 1-\delta$ as:
\begin{align*}
    \Eml{\sigma, \eval{S}}\Biggl[\sup_\lambda \frac{1}{n_vT} \sum_{t,i} &\sigma^{t(i)} \Eml{X,f,\epsilon}\bigg[l(X_v^{t(i)\intercal}\wh_\lambda,y_v^{t(i)})\bigg]\Biggr]\leq\nonumber\\
       &\frac{L}{\sqrt{n_vT}} \sqrt{\E{x_v}{\|x_v\|^2}}\E{}{\|\ws{}\| + \|\epsilon\|/\sqrt{V(XX^\intercal)}}.
\end{align*}
We now replace $\delta$ by $\delta/4$ in the 4 probabilistic bounds above so that the following holds with probability at least $1-\delta$:
\begin{align*}
    l_v(\lambda_{ERM}) &- l_v(\lambda^*) \leq\nonumber\\
        &\frac{2ML\Lambda_D^T
        }{\sqrt{T}}\E{x_v}{\|x_v\|} + \frac{2L}{\sqrt{n_vT}} \sqrt{\E{x_v}{\|x_v\|^2}}\E{}{\|\ws{}\| + \|\epsilon\|/\sqrt{V(XX^\intercal)}}\nonumber\\
        & 
        + \frac{2MLb_v\Lambda_D^T}{\sqrt{n_vT}}\sqrt{\frac{\log(4T/\delta)}{2}}
        + 2C\sqrt{\frac{2\ln(16/\delta)}{T}} + C\sqrt{\frac{\ln(4/\delta)}{2T}}.
\end{align*}
To get the desired result, we note that $C\sqrt{\frac{\ln(4/\delta)}{2T}} \leq C/2\sqrt{\frac{2\ln(16/\delta)}{T}}$.
\end{proof}

\section{Proof of Proposition \ref{prop:ex_isotropic}}\label{sec:supp_res}
\begin{proposition}[Proof of Proposition \ref{prop:ex_isotropic}]\label{lem:ex_isotropic_proof}
    Consider the expected validation error using an ERM estimator for the ridge parameter as defined in Equation \ref{def:lamb_erm}.
    Assume further that all tasks are well-specified such that all inputs $x$ are sampled from isotropic distributions with independent entries and bounded density. Concretely, assume that each entry in the input $x$ is sampled independently from a zero-mean distribution with density bounded by $C_0$ such that $\E{}{xx^\intercal} = \Sigma = \sigma_x^2/d I_d$. We further restrict the covariance matrices of both $x,\ws{}$ to have constant trace as $d$ increases. So, $tr(\Sigma) = \sigma_x^2 = const$ and $tr(\E{}{\ws{}\ws{\intercal}}) = \sigma_w^2 = const$. If $n \geq 6d$, the generalization error bound given in Theorem \ref{thm:lips_ws} is $O\left(\frac{1}{\sqrt{T}}(T^{2/d} + \sqrt{\log (T/\delta)})\right)$.
\end{proposition}
\begin{proof}
    To instantiate the bound in Theorem \ref{thm:lips_ws}, we want to use Theorem \ref{thm:exp_cov_inv}, and make the following manipulations to fit their assumptions.
    Consider the random variable $x' = (\sqrt{d}/\sigma_x) x$. The covariance matrix of $x'$ is $\E{}{x'x^{\prime \intercal}} = I_d$, and each entry is independent with density bounded by $C_0$. We can thus use Proposition \ref{prop:ip_entry} to satisfy the assumptions of Theorem \ref{thm:exp_cov_inv} for $\alpha = 1, C = 2\sqrt{2}C_0 = O(1)$. Thus, if $n\ge \max(6d/\alpha, 12/\alpha) = \max(6d,12)$ and $1\leq q \leq \alpha n /12 = n/12$,
    \begin{align}
        \E{}{|\max(1,\lambda_{min}(\hat{\Sigma}'_n)^{-1})|^q}^{1/q} \leq 2^{1/q}C',\nonumber
    \end{align}
    where $C' = O(1)$ and $\hat{\Sigma}'_n = n^{-1} X'X^{\prime\intercal}$ is the sample covariance matrix of $x'$. Now, since $\lambda_{\min}(\hat{\Sigma}'_n)^{-1} \leq \max(1,\lambda_{\min}(\hat{\Sigma}'_n)^{-1})$,
    \begin{align}
        \E{}{\lambda_{\min}(\hat{\Sigma}'_n)^{-q}}^{1/q} &\leq \E{}{|\max(1,\lambda_{\min}(\hat{\Sigma}'_n)^{-1})|q}^{1/q} \nonumber\\
        \implies \E{}{n^q\lambda_{\min}(X'X^{\prime \intercal})^{-1}}^{1/q} &= O(1)\nonumber\\
        \implies \E{}{d^{-q}\lambda_{\min}(XX^\intercal)^{-q}}^{1/q} &= O(1/n)\nonumber\\
        \implies \E{}{\lambda_{\min}(XX^\intercal)^{-q}}^{1/q} &= O(d/n)\nonumber\\
        \implies \E{}{\frac{1}{V(XX^\intercal)^q}}^{1/q} &= O(d/n).\nonumber
    \end{align}

    Now, for any sequence of i.i.d.\ random variables $Z_1, \ldots, Z_N$, if $\E{}{Z^q}^{1/q}\leq C$ for $q\geq 1$, then
    \begin{align*}
        \E{}{\max(Z_1,\ldots, Z_N)} &= \E{}{\max(Z_1,\ldots, Z_N)}^{q/q}\\
        &\leq \E{}{\max(Z_1,\ldots, Z_N)^q}^{1/q} \quad \text{(Jensen's inequality)}\\
        &\leq (\E{}{NZ^q})^{1/q}\\
        &= N^{1/q}C.
    \end{align*}
    Thus, since $\E{}{\frac{1}{V(XX^\intercal)^q}}^{1/q} = O(d/n) \implies \Lambda_D^T = \E{}{\max_t(1/V(X^tX^{t\intercal}))} = O\left(\frac{d}{n}T^{1/q}\right)$. This holds if $n\ge \max(6d, 12)$ and $q\leq n/12$. We substitute $q=d/2$ to get $\Lambda_D^T = O\left(\frac{d}{n}T^{2/d}\right)$
    
    The bound given in Theorem \ref{thm:lips_ws} is reproduced as follows:
    \begin{align}
        l_v(\lambda_{ERM}) - l_v(\lambda^*) &\leq \frac{2ML\Lambda_D^T}{\sqrt{T}}\E{}{\|x_v\|} + \frac{2L}{\sqrt{n_vT}} \sqrt{\E{x_v}{\|x_v\|^2}}\E{}{\|\ws{}\| + \|\epsilon\|/\sqrt{V(XX^\intercal)}}\nonumber\\
       &\quad + \frac{2MLb_v\Lambda_D^T}{\sqrt{n_vT}}\sqrt{\frac{\log(4T/\delta)}{2}} + 5C\sqrt{\frac{\ln(16/\delta)}{2T}}.\nonumber
    \end{align}

    For the second term,
    \begin{align}
        \E{X,\ws{},\epsilon}{\|\ws{}\| + \|\epsilon\|/\sqrt{V(XX^\intercal)}} &\leq \sqrt{\E{}{\|\ws{}\|^2}} + \sqrt{\E{}{\|\epsilon\|^2}O(d/n)}\nonumber\\
        &= \sqrt{\E{}{tr(\ws{}\ws{\intercal})}} + O(\sqrt{d/n}) = \sigma_w + O(\sqrt{d/n}).
    \end{align}
    If we substitute $n \geq 6d$, we get that $\Lambda_D^T = O(T^{2/d})$, and $\E{X,\ws{},\epsilon}{\|\ws{}\| + \|\epsilon\|/\sqrt{V(XX^\intercal)}} = O(1)$. Further, all terms involving $\delta$ are $O(\log (T/\delta))$, using which we can rewrite the bound as:
    \begin{align*}
        l_v(\lambda_{ERM}) - l_v(\lambda^*) &\leq O\left(\frac{T^{2/d}}{\sqrt{T}} + \frac{1}{\sqrt{T}} + \sqrt{\frac{\log (T/\delta)}{T}})\right)\\
        &\leq O\left(\frac{1}{\sqrt{T}}\left(T^{2/d} + \sqrt{\log (T/\delta)}\right)\right).
    \end{align*}    
    Note that we used the fact that $\E{}{\|x_v\|} \leq \sqrt{\E{}{\|x_v\|^2}} = \sqrt{\E{}{tr(\Sigma)}} = O(1).$
\end{proof}

\section{Proofs for tuning LASSO}\label{app:lasso}

We first present relevant properties of LASSO solutions from prior work. Let $(X,y)$ with $X=[x_1,\dots,x_d]\in\R^{d\times m}$ and $y\in\R^{m}$ denote a (training) dataset consisting of $m$ labeled examples with $d$ features. LASSO  is given by the following optimization problem.

\[\min_{w\in\R^d}\norm{X^\intercal w-y}_2^2+\lambda_1||w||_1,\]
\noindent where $\lambda_1\in[\underline{\Lambda},\overline{\Lambda}]\subset\R_+$ is the L1 regularization penalty. We will use the following well-known facts about the solution of the LASSO optimization problem~\citep{fuchs2005recovery,tibshirani2013lasso} which follow from the Karush-Kuhn-Tucker (KKT) optimality conditions.

\begin{lemma}[KKT Optimality Conditions for LASSO] $w^*\in \argmin_{w\in\R^d}\norm{X^\intercal w-y}_2^2+\lambda_1||w||_1$ iff for all $j\in[d]$,
\begin{align*}
    x_j(X^\intercal w^*-y)&=\lambda_1\mathrm{sign}(w^*), &\text{ if }w^*_j\ne 0,\\
    |x_j(X^\intercal w^*-y)|&\le\lambda_1, &\text{ otherwise. }
\end{align*}
\label{lem:lasso-kkt}
\end{lemma}
\noindent Here $x_j(X^\intercal w^*-y)$ is  the correlation of the the $j$-th covariate with the residual $X^\intercal w^*-y$. This motivates the definition of {\it equicorrelation sets} of covariates. For $S=\{s_1,\dots,s_k\}\subseteq [d]$, let $X_S=[x_{s_{1}},\dots,x_{s_{k}}]$.

\begin{definition}[Equicorrelation sets, \cite{tibshirani2013lasso}] Let $w^*\in \argmin_{w\in\R^d}\norm{X^\intercal w-y}_2^2+\lambda_1||w||_1$. The equicorrelation set corresponding to $w^*$, $\cE=\{j\in[d]\mid |x_j(X^\intercal w^*-y)|=\lambda_1\}$, is simply the set of covariates with maximum absolute correlation. We also define the equicorrelation sign vector for $w^*$ as $s=\mathrm{sign}(X_{\cE}(X^\intercal w^*-y))$.
\label{def:ec}
\end{definition}

The characterization of the LASSO solution in Lemma \ref{lem:lasso-kkt} can be restated more concisely using the equicorrelation sets and sign vectors as
\[X_{\cE}(X_{\cE}^\intercal w^*_{\cE}-y)=\lambda_1 s.\]

\noindent A necessary and sufficient condition for the uniqueness of the LASSO solution is that $X_{\cE}$ is full rank for all equicorrelation sets $\cE$~\citep{tibshirani2013lasso} (see \cite{enet_2022} for a  sufficient condition in terms of general position). 

\begin{assumption}\label{asmp:lasso-full-rank}
    For each task, $X_{\cE}$ is full rank for each  $\cE\subseteq[d]$.
\end{assumption}

\noindent Under this assumption, the unique solution to LASSO satisfies the following closed form within a fixed piece.

\begin{lemma}[\cite{tibshirani2013lasso}, Lemma 3] Let $\cE,s$ be the equicorrelation set and sign vector respectively (Definition \ref{def:ec}). Suppose Assumption \ref{asmp:lasso-full-rank} holds for $X$. Then for any $y$ and $\lambda_1>0$, the LASSO solution
is unique and is given by
\[w^*_{\cE}=(X_\cE X_\cE^\intercal )^{-1}(X_{\cE}y+\lambda_1 s), w^*_{[d]\setminus\cE}=0.\]
\label{lem:lasso}
\end{lemma}

\noindent While the above piecewise closed-form solution for LASSO involves similar terms to the ridge closed-form solution, there are some crucial differences. First, the $\lambda_1$ dependence is linear within each piece. In addition, it is known that the optimal solution $w^*_{\cE}$ is continuous in $\lambda_1 $ (even at the piece boundaries)~\citep{Mairal2012ComplexityAO}. Second, the slope and intercept for each linear piece depend on the submatrix $X_\cE$ instead of the full matrix $X$.

\subsection{Rademacher complexity lemmas for LASSO}
We will now present appropriate modifications of the lemmas for ridge regression above and use the above properties of LASSO solutions to establish bounds on the generalization error for tuning the L1 regularization coefficient. The following Lemma is the analogue of Lemma \ref{lem:S_v_2} for L1 regularization.

\begin{lemma}\label{lem:S_v_1}
   Consider the problem of tuning the LASSO regularization coefficient $\lambda_1$. Given a validation loss function that satisfies Assumptions \ref{assum:bound} and \ref{assum:lips} given in Section \ref{sec:ridge} and $\etr{S}$ as defined in Equation \ref{def:etr} the following holds (where we denote $y_v^{t(i)} = f^t(x_v^{t(i)},\epsilon_v^{t(i)})$, and $\sigma^t$ are i.i.d. Rademacher variables):
   \begin{align}
       \E{\sigma, \etr{S}}{\sup_{\lambda\in[\underline{\Lambda},\overline{\Lambda}]} \frac{1}{n_vT} \sum_{t,i} \sigma^t {l((x_v^{t(i)})^\intercal\wh_\lambda^t,y_v^{t(i)})}}\leq
\frac{L\overline{\Lambda}\tilde{\Lambda}_D^T\sqrt{d}\E{x_v}{\|x_v\|}}{\sqrt{T}}+ \frac{L\overline{\Lambda}\tilde{\Lambda}_D^T\sqrt{d}b_v}{\sqrt{n_vT}}\sqrt{\frac{\log\frac{T}{\delta}}{2}},
       \nonumber
   \end{align}
   where $b_v^2 = \max \|x_v\|^2$ and  $\tilde{\Lambda}_D^T = \E{\etr{S}}{ \max_{\cE,t} \frac{1}{V(X_\cE X_\cE^\intercal)} }$.
\end{lemma}
\begin{proof}
Using Lipschitzness (Corollary \ref{thm:rad_lipschitz}), as argued in the proof of Lemma \ref{lem:S_v_2}:
\begin{align*}
    \E{\sigma, \etr{S}}{\sup_{\lambda\in[\underline{\Lambda},\overline{\Lambda}]} \frac{1}{n_vT} \sum_{t,i} \sigma^t {l((x_v^{t(i)})^\intercal\wh_\lambda^t,y_v^{t(i)})}}
    \leq \frac{L}{n_vT}\E{\sigma, \etr{S}}{\sup_{\lambda\in[\underline{\Lambda},\overline{\Lambda}]}  \sum_{t} \sigma^t \sum_i (x_v^{t(i)})^\intercal\wh^t_{\lambda}}.
\end{align*}

\noindent Let $\Lambda_{\cE,s}$ denote the set of values of $\lambda\in[\underline{\Lambda},\overline{\Lambda}]$ for which the equicorrelation set and sign vectors are $\cE,s$ respectively (Definition \ref{def:ec}). We can rewrite the above as

\begin{align*}
    \frac{L}{n_vT}\E{\sigma, \etr{S}}{\sup_{\lambda\in[\underline{\Lambda},\overline{\Lambda}]}  \sum_{t} \sigma^t \sum_i (x_v^{t(i)})^\intercal\wh^t_{\lambda}} = \frac{L}{n_vT}\E{\sigma, \etr{S}}{\max_{\cE,s}\sup_{\lambda\in\Lambda_{\cE,s}}  \sum_{t} \sigma^t \sum_i (x_v^{t(i)})^\intercal\wh^t_{\lambda}}.
\end{align*}

\noindent  We next use Lemma \ref{lem:lasso} and H{\"o}lder's inequality to show Lipschitzness of  $x_{v}^{\intercal}\wh_{\lambda}$ in $\lambda$ for a fixed $\cE,s$:
\begin{align*}
    x_v^\intercal\wh_{\lambda_a} - x_v^\intercal\wh_{\lambda_b} &= (x_v)_{\cE}^\intercal((X_\cE X_\cE^\intercal )^{-1}(X_{\cE}y+\lambda_a s)-(X_\cE X_\cE^\intercal )^{-1}(X_{\cE}y+\lambda_b s))\nonumber\\
    &= (x_v)_{\cE}^\intercal((X_\cE X_\cE^\intercal )^{-1}s)(\lambda_a-\lambda_b)\nonumber\\
    % &\leq\|(x_v)_{\cE}^\intercal((X_\cE X_\cE^\intercal )^{-1}s)\||\lambda_1-\lambda_2|\nonumber\\
    &\leq\|(x_v)_{\cE}^\intercal\|\|(X_\cE X_\cE^\intercal )^{-1}s\||\lambda_a-\lambda_b|\\
    &\leq\|x_v\|\|(X_\cE X_\cE^\intercal )^{-1}s\||\lambda_a-\lambda_b|.
\end{align*}

\noindent We note that the above piecewise-Lipschitzness within a fixed piece corresponding to a fixed $\cE,s$ also implies a global Lipschitzness in terms of the worst-case piece, by using the fact that $\wh_{\lambda_1}$ is continuous in $\lambda_1$~\citep{Mairal2012ComplexityAO}. Indeed, for any pair of ${\lambda_1}$ values $\lambda,\lambda'$, the (signed) average slope of the slope between them has magnitude no more than the largest slope in any single fixed piece corresponding to the $\cE,s$ that maximize $\|(X_\cE X_\cE^\intercal )^{-1}s\|$.

We can  use this Lipschitzness (Theorem \ref{thm:rad_lips_general}) in above to get
\begin{align}
    &\E{\sigma, \etr{S}}{\sup_{\lambda\in[\underline{\Lambda},\overline{\Lambda}]}  \sum_{t} \sigma^t \sum_i (x_v^{t(i)})^\intercal\wh^t_{\lambda}} \\
    &\qquad\leq \E{ \etr{S}}{{\max_{t,\cE,s}\|(X^t_\cE X_\cE^{t\intercal} )^{-1}s\|}\E{\sigma}{\sup_{\lambda\in[\underline{\Lambda},\overline{\Lambda}]}  \sum_{t} \sigma^t \sum_i \|x_v^{t(i)}\|  \lambda }}\nonumber\\
    &\qquad\leq \E{\etr{S}}{\max_{t,\cE,s}\|(X^t_\cE X_\cE^{t\intercal} )^{-1}s\|}\left(\overline{\Lambda}\sqrt{\sum_t\left(\sum_i\|x_v^{t(i)}\|\right)^2}\right)\nonumber\\
    &\qquad\leq \overline{\Lambda}\E{\etr{S}}{\max_{t,\cE,s}\|(X^t_\cE X_\cE^{t\intercal} )^{-1}\|\|s\|}\left(\sqrt{\sum_t\left(\sum_i\|x_v^{t(i)}\|\right)^2}\right)\nonumber\\
    &\qquad\leq  {\overline{\Lambda}\sqrt{d}}\E{\etr{S}}{ \max_{t,\cE} \|(X^t_\cE X_\cE^{t\intercal} )^{-1}\| }\left(\sqrt{\sum_t\left(\sum_i\|x_v^{t(i)}\|\right)^2}\right).\nonumber
\end{align}
We use Khintchine's inequality, H\"{o}lder's inequality, and $||s||\le \sqrt{d}$ in the above steps. 
Substituting $\E{\etr{S}}{ \max_{t,\cE} \|(X^t_\cE X_\cE^{t\intercal} )^{-1}\| } =: \tilde{\Lambda}_D^T$, and simplifying the last term as in the proof of Lemma \ref{lem:S_v_2}, we get the desired bound.
\end{proof}

\noindent The following lemma is the LASSO analogue to Lemma \ref{lem:S_tr_2} for Ridge regularization.
\begin{lemma}\label{lem:S_tr_1}
    Given a validation loss function that satisfies Assumptions \ref{assum:bound} and \ref{assum:lips} given in Section \ref{sec:ridge}, and $\eval{S}$ as defined in Equation \ref{def:eval}, the following holds with probability at least $1-\delta$ (where we denote $y_v^{t(i)} = f(x_v^{t(i)},\epsilon_v^{t(i)})$ and $\sigma^{t(i)}$ are i.i.d. Rademacher variables):
   \begin{align}
       &\E{\sigma, \eval{S}}{\sup_\lambda \frac{1}{n_vT} \sum_{t,i} \sigma^{t(i)} \E{X,y}{l(x_v^{t(i)\intercal}\wh_\lambda^t,y_v^{t(i)})}} \leq\nonumber\\
       &\qquad\frac{L\sqrt{\E{x_v}{\|x_v\|^2}}}{\sqrt{n_v T}}\E{X,y}{ \max_{\cE}\left( \frac{\|y\|}{\sqrt{V(X_\cE X_\cE^\intercal)}} + \overline{\Lambda}\frac{\sqrt{d}}{V(X_\cE X_\cE^\intercal)}\right)}.\nonumber
   \end{align} 
\end{lemma}
\begin{proof}
We follow the arguments in the proof of Lemma \ref{lem:S_tr_2}. The main change is when giving the bound on $\|\wh_\lambda\|$.

For a fixed $\cE,s$ (Definition \ref{def:ec}), we have by Lemma \ref{lem:lasso},
\begin{align}
    \|\wh_\lambda\| &= \|(X_\cE X_\cE^\intercal )^{-1}(X_{\cE}y+\lambda_1 s)\|\nonumber\\
     &\leq \|(X_\cE X_\cE^\intercal )^{-1}X_{\cE}y\|+ \lambda_1\|(X_\cE X_\cE^\intercal )^{-1} s\|\quad \text{(triangle inequality)}\nonumber\\
    &\leq \frac{\|y\|}{\sqrt{V(X_\cE X_\cE^\intercal)}} + \overline{\Lambda}\frac{\|s\|}{V(X_\cE X_\cE^\intercal)} \nonumber\\
    &\le \frac{\|y\|}{\sqrt{V(X_\cE X_\cE^\intercal)}} + \overline{\Lambda}\frac{\sqrt{d}}{V(X_\cE X_\cE^\intercal)}.%\label{eq:w_norm_max}
\end{align}
Recall that here $V(M)$ denotes the smallest non-zero singular value of $M$.
This implies,
\begin{align}
     \E{X,y}{\sup_\lambda\|\wh_\lambda\|} \leq  \E{X,y}{ \max_{\cE}\left( \frac{\|y\|}{\sqrt{V(X_\cE X_\cE^\intercal)}} + \overline{\Lambda}\frac{\sqrt{d}}{V(X_\cE X_\cE^\intercal)}\right)}.%\label{eq:sup_w_lamb}
\end{align}

\end{proof}

\section{Proofs for Tuning the Elastic Net}\label{app:elasticnet}

We further extend the analysis for LASSO in Appendix \ref{app:lasso} to the Elastic Net which involves simultaneous tuning of L1 and L2 penalties. We use the same notation as in Appendix \ref{app:lasso}. The Elastic Net  is given by the following optimization problem.

\[\min_{w\in\R^d}\norm{X^\intercal w-y}_2^2+\lambda_1||w||_1+\lambda_2||w||_2^2,\]
\noindent where $\lambda_1\in[\underline{\Lambda}_1,\overline{\Lambda}_1]\subset\R_+$  and $\lambda_2\in[\underline{\Lambda}_2,\infty)\subset\R_+$. We will use the following generalization of Lemma \ref{lem:lasso}.

\begin{lemma}[\cite{enet_2022}, Lemma C.1]  Suppose Assumption \ref{asmp:lasso-full-rank} holds for $X$. Then for any $y$ and $\lambda_1,\lambda_2>0$, the Elastic Net solution
is unique and is given by
\[w^*_{\cE}=(X_\cE X_\cE^\intercal +\lambda_2 I_{|\cE|})^{-1}(X_{\cE}y+\lambda_1 s), w^*_{[d]\setminus\cE}=0,\]
for some $\cE\subseteq[d]$ and $s\in\{-1,1\}^{|\cE|}$.
\label{lem:enet}
\end{lemma}

\noindent We now extend the LASSO lemmas from Appendix \ref{app:lasso} to the Elastic Net. The following is a straightforward extension of Lemma \ref{lem:S_v_1} and gives an upper bound on the expectation with respect to sampling of the training set, of the Rademacher complexity of the average empirical validation loss, across different values of $\lambda_1,\lambda_2$.

\begin{lemma}\label{lem:S_v_12}
   Consider the problem of tuning the Elastic Net regularization coefficients $\lambda=(\lambda_1,\lambda_2)$. Given a validation loss function that satisfies Assumptions \ref{assum:bound} and \ref{assum:lips} given in Section \ref{sec:ridge} and $\etr{S}$ as defined in Equation \ref{def:etr}, the following holds (where we denote $y_v^{t(i)} = f(x_v^{t(i)},\epsilon_v^{t(i)})$ and $\sigma^{t}$ are i.i.d. Rademacher variables):
   \begin{align*}
       \mathbb{E}_{\sigma, \etr{S}}\Bigg[\sup_{\substack{\lambda_1\in[\underline{\Lambda}_1,\overline{\Lambda}_1]\\\lambda_2\in[\underline{\Lambda}_2,\infty)}}& \frac{1}{n_vT} \sum_{t,i} \sigma^t {l((x_v^{t(i)})^\intercal\wh_\lambda^t,y_v^{t(i)})}\Bigg]\leq\\
&\frac{L\overline{\Lambda}_1\sqrt{d}}{\sqrt{T}}\left(\E{x_v}{\|x_v\|}+b_v\sqrt{\frac{\log(T/\delta)}{2n_v}}\right)\E{X}{\max_{t,\cE}  \frac{1}{V(X^t_\cE X_\cE^{t\intercal}) +\underline{\Lambda}_2 }}.
       \nonumber
   \end{align*}
\end{lemma}

\begin{proof}
    The proof follows the same arguments as in the proof of Lemma \ref{lem:S_v_1}, but using Lemma \ref{lem:enet} and that $\lambda_2\ge \underline{\Lambda}_2 $.
\end{proof}
\noindent The following lemma is the Elastic Net analogue to Lemmas \ref{lem:S_tr} and \ref{lem:S_tr_1}. We give an upper bound on the expectation with respect to sampling of the validation set, of the Rademacher complexity of the average expected validation loss (w.r.t.\ sampling of the training set), across different values of $\lambda_1,\lambda_2$.
\begin{lemma}\label{lem:S_tr_12}
    Given a validation loss function that satisfies Assumptions \ref{assum:bound} and \ref{assum:lips} given in Section \ref{sec:ridge}, and $\eval{S}$ as defined in Equation \ref{def:eval}, the following holds with probability at least $1-\delta$ (where we denote $y_v^{t(i)} = f(x_v^{t(i)},\epsilon_v^{t(i)})$ and $\sigma^{t(i)}$ are i.i.d. Rademacher variables):
   \begin{align}
       &\E{\sigma, \eval{S}}{\sup_{\substack{\lambda_1\in[\underline{\Lambda}_1,\overline{\Lambda}_1]\\\lambda_2\in[\underline{\Lambda}_2,\infty)}} \frac{1}{n_vT} \sum_{t,i} \sigma^{t(i)} \E{X,y}{l(x_v^{t(i)\intercal}\wh_\lambda^t,y_v^{t(i)})}}\nonumber \\&\qquad\leq
       \frac{L\sqrt{\E{x_v}{\|x_v\|^2}}}{\sqrt{n_vT}}\E{X,y}{ \max_{\cE}\left( \frac{\|y\|\sqrt{V^*(X_\cE X_\cE^\intercal)}}{V^*(X_\cE X_\cE^\intercal)+\underline{\Lambda}_2}+\frac{\overline{\Lambda}_1\sqrt{d}}{V(X_\cE X_\cE^\intercal)+\underline{\Lambda}_2}\right)} .\nonumber
   \end{align}
Here %$V(M)$ denotes the smallest non-zero singular value of matrix $M$ and  
$V^*(M)$ is the non-zero singular value of $M$ that maximizes $\frac{\sqrt{\sigma_i(M)}}{\sigma_i(M)+\underline{\Lambda}_2}$.
\end{lemma}
\begin{proof}
We adapt the arguments in the proof of Lemma \ref{lem:S_tr_1}. 

For a fixed $\cE,s$ (Definition \ref{def:ec}), we have by Lemma \ref{lem:lasso},
\begin{align}
    \|\wh_\lambda\| &= \|(X_\cE X_\cE^\intercal +\lambda_2 I)^{-1}(X_{\cE}y+\lambda_1 s)\|\nonumber\\
     &\leq \|(X_\cE X_\cE^\intercal +\lambda_2 I)^{-1}X_{\cE}y\|+ \lambda_1\|(X_\cE X_\cE^\intercal +\lambda_2 I)^{-1} s\|\quad \text{(triangle inequality)}\nonumber\\
    &\leq \max_i \frac{\|y\|\sqrt{\sigma_i(X_\cE X_\cE^\intercal)}}{{\sigma_i(X_\cE X_\cE^\intercal)+\underline{\Lambda}_2}} + \overline{\Lambda}_1\frac{\|s\|}{V(X_\cE X_\cE^\intercal)+\underline{\Lambda}_2} \nonumber\\
    &\le \frac{\|y\|\sqrt{V^*(X_\cE X_\cE^\intercal)}}{V^*(X_\cE X_\cE^\intercal)+\underline{\Lambda}_2}+\frac{\overline{\Lambda}_1\sqrt{d}}{V(X_\cE X_\cE^\intercal)+\underline{\Lambda}_2}.%\label{eq:w_norm_max}
\end{align}

Recall that here $V^*(M)$ denotes the non-zero singular value of $M$ that maximizes $\frac{\sqrt{\sigma_i(M)}}{\sigma_i(M)+\underline{\Lambda}_2}$.
This implies,
\begin{align}
     \E{X,y}{\sup_\lambda\|\wh_\lambda\|} \leq  \E{X,y}{ \max_{\cE}\left( \frac{\|y\|\sqrt{V^*(X_\cE X_\cE^\intercal)}}{V^*(X_\cE X_\cE^\intercal)+\underline{\Lambda}_2}+\frac{\overline{\Lambda}_1\sqrt{d}}{V(X_\cE X_\cE^\intercal)+\underline{\Lambda}_2}\right)}.%\label{eq:sup_w_lamb}
\end{align}
\end{proof}

\subsection{Constructing Gramian matrices with lower bounded smallest eigenvalue}

Here we present a helper lemma for constructing an illustrative example distribution where our distribution-dependent bounds lead to improved generalization guarantees over prior work. We will need the following standard Theorem that gives a lower bound on the smallest singular value of a sub-Gaussian matrix.

\begin{theorem}[\citealt{vershynin2018high}] Let $A$ be a $d\times n$ random matrix with independent, mean zero, subgaussian with variance proxy $K^2$, and isotropic columns $A_i$. Then
for any $t \ge 0$ the smallest singular value of $A$ satisfies,

  $$\sigma_{\min}(A)\ge \sqrt{n} - CK^2(\sqrt{d} + t),$$

with probability at least $1 - 2 \exp(-t^2)$, where $C$ is an absolute constant. 
\end{theorem}

To construct our example for the Elastic Net, we need to extend this result to all sub-matrices and all tasks.
Roughly speaking, in the following lemma, we establish a uniform high-probability lower bound on the smallest singular value of sub-Gaussian submatrices for all tasks.

\begin{lemma}

Let $A^t\in\R^{d\times n}$ be i.i.d.\ random matrices for each $t\in [T]$, with independent, mean-zero, isotropic, sub-Gaussian columns with variance proxy $K^2$. 
Then there exist constants $C,C'$ depending only on $K$, such that the following holds: if $n \ge C \left(d+ \log \frac{T}{\delta} \right)$, then with probability at least $1-\delta$, 

$$\min_{\substack{E \subseteq [d] \\ t\in[T]}} \sigma_{\min}(A^t_E) \ge {\sqrt{C'n}}.$$

Equivalently,

$$\min_{\substack{E \subseteq [d] \\ t\in[T]}} \lambda_{\min}(A^t_E (A^t_E)^\intercal) \ge C'n.$$

    \label{lem:dd-gram}
\end{lemma}

\begin{proof}
For a fixed subset $E\subseteq[d]$ of size $|E|=s$ and fixed $t\in[T]$, note that the matrix $A^t_E$ has independent, isotropic, sub-Gaussian rows. By standard results (e.g., ~\citet{vershynin2018high}, Theorem 4.6.1 in the 2nd Edition), there exist constants $c_0,C_0$ such that

$$\Pr\left[ \sigma_{\min}(A^t_E) \le \sqrt{n} - C_0 \sqrt{s} - r \right] \le \exp(-4c_0 r^2), \quad \forall r \ge 0.$$

Set $r=\sqrt{n}/2$ to get

$$\Pr\left[ \sigma_{\min}(A^t_E) \le \frac{\sqrt{n}}{2} - C_0 \sqrt{s} \right] \le \exp(-c_0 n).$$

We now do a union bound over subsets $E$ and the tasks $t$.
There are $2^d$ subsets of $[d]$ and $T$ tasks. Applying a union bound, we get the probability of failure

\begin{align*}
    \Pr\left[ \exists E \subseteq [d], t\in[T] \mid \sigma_{\min}(A_E) \le \frac{\sqrt{n}}{2} - C_0 \sqrt{|E|} \right] 
&\le 2^d \cdot T \cdot \exp(-c_0 n)\\
&\le \exp\left( -c_0(n-c_1d-c_2\log T)\right),
\end{align*}

for constants $c_1,c_2$. Choose $n\ge c_1d+c_2\log T+\frac{1}{c_0}\log\frac{1}{\delta}$, to make this probability at most $\delta$. Thus, with probability at least $1-\delta$, we have for all $E,t$

$$\min_{\substack{E \subseteq [d] \\ t\in[T]}} \sigma_{\min}(A^t_E) \ge \frac{\sqrt{n}}{2}-C_0\sqrt{d}.$$

Choosing  $n\ge C(d+\log\frac{T}{\delta})$ with a sufficiently large constant $C$ completes the proof.
\end{proof}

\subsection{Proof of Proposition \ref{prop:en_isotropic_main}}

Finally, we show an example where our bounds improve over the distribution independent bounds from prior work~\citep{regression_pdim}.

\begin{proposition}\label{prop:en_isotropic}
    Consider the expected validation error of an ERM estimator for the Elastic Net hyperparameters over the  range $\lambda_1\in[\underline{\Lambda}_1,\overline{\Lambda}_1],\lambda_2\in[\underline{\Lambda}_2,\infty)$.
    Assume further that all tasks are well-specified such that all inputs $x$ are sampled from sub-Gaussian distributions with independent entries. Concretely, assume that each entry in the input $x$ is sampled independently from a zero-mean sub-Gaussian distribution  such that $\E{}{xx^\intercal} = \Sigma = (\sigma_x^2/d) I_d$. We further restrict the covariance matrices of both $x,\ws{}$ to have constant trace as $d$ increases. So, $tr(\Sigma) = \sigma_x^2 = const$ and $tr(\E{}{\ws{}\ws{\intercal}}) = \sigma_w^2 = const$. For sufficiently large $n \geq \Omega\left(d+\log\frac{T}{\underline{\Lambda}_2}\right)$, the generalization error bound given in Theorem \ref{thm:en} is $\tilde{O}\left(1/\sqrt{nT}+\sqrt{\frac{\log 1/\delta}{T}}\right)$, where the soft-O notation suppresses dependence on quantities apart from $T,n, \delta$ and $d$.
\end{proposition}
\begin{proof}
The generalization error bound in Theorem \ref{thm:en} is

\begin{align*}
    l_v&(\lambda_{ERM}) - l_v(\lambda^*)=\\&\tilde{O}\left(\frac{L\overline{\Lambda}\tilde{\Lambda}_D^T\sqrt{d}}{\sqrt{T}}+\frac{L}{\sqrt{n_vT}}\E{X,y}{ \max_{\cE}\left( \frac{\|y\|\sqrt{V^*(X_\cE X_\cE^\intercal)}}{V^*(X_\cE X_\cE^\intercal)+\underline{\Lambda}_2}+\frac{\overline{\Lambda}_1\sqrt{d}}{V(X_\cE X_\cE^\intercal)+\underline{\Lambda}_2}\right)} +\sqrt{\frac{\log 1/\delta}{T}} \right).
\end{align*}

Define $G^{t,\cE}:=X^t_\cE X_\cE^{t\intercal}$. We have,
$$\|(X^t_\cE X_\cE^{t\intercal} +\underline{\Lambda}_2 I)^{-1}\|= \frac{1}{\underline{\Lambda}_2+ \lambda_{\min}(G^{t,\cE})}.$$ 

Now by Lemma \ref{lem:dd-gram}, if $n=\Omega(d+\log (T/\delta))$ with probability at least $1-\delta$,

$$\max_{\cE,t}\|(X^t_\cE X_\cE^{t\intercal} +\underline{\Lambda}_2 I)^{-1}\|\le\frac{1}{\underline{\Lambda}_2+C{n}}.$$

    Setting $\delta=\frac{\underline{\Lambda}_2}{n}$, we get that for $n=\Omega\left(d+\log\frac{T}{\underline{\Lambda}_2}\right)$,
\begin{align*}
\tilde{\Lambda}_D^T=\E{X}{\max_{t,\cE}  \|(X^t_\cE X_\cE^{t\intercal} +\underline{\Lambda}_2 I)^{-1}\|}\le  \frac{1}{\underline{\Lambda}_2+C{n}} +\frac{\underline{\Lambda}_2}{n}\cdot \frac{1}{\underline{\Lambda}_2}=O\left(\frac{1}{n}\right).
\end{align*}
    A similar argument shows that
    \begin{align*}
    \E{X,y}{\max_{\cE}\left( \frac{\|y\|\sqrt{V^*(X_\cE X_\cE^\intercal)}}{V^*(X_\cE X_\cE^\intercal)+\underline{\Lambda}_2}+\frac{\overline{\Lambda}_1\sqrt{d}}{V(X_\cE X_\cE^\intercal)+\underline{\Lambda}_2}\right)}\\
    =O\left(\frac{1}{\sqrt{n}}+\frac{\overline{\Lambda}\sqrt{d}}{n}\right)
    \end{align*}
    Therefore,
    \begin{align*}
        l_v(\lambda_{ERM}) - l_v(\lambda^*)=O\left(\frac{L\overline{\Lambda}\sqrt{d}}{\sqrt{T}}\cdot \frac{1}{n}+\frac{L}{n_vT}\cdot\frac{\overline{\Lambda}}{\sqrt{n}}+\sqrt{\frac{\log 1/\delta}{T}}\right)=O\left(\frac{L\overline{\Lambda}}{\sqrt{nT}}+\sqrt{\frac{\log 1/\delta}{T}}\right).
        \end{align*}
\end{proof}

\section{Alternative Bounds Based on Prior Work}\label{sec:alt_bounds}
In this section, we present an alternative to Theorem \ref{thm:lips_tgt} that uses previous techniques like the ones used in \cite{mtl_maurer}. In particular, \cite{mtl_maurer} address learning optimal representations from multiple tasks. They give generalization error bounds using Rademacher complexities by dividing the error into an error induced from learning imperfect representations and from imperfect learning given a representation. This section proceeds similarly, by dividing the generalization error into an error induced from imperfect estimation of expected validation error (due to finiteness of validation data), and error from imperfect estimation of $\lambda$ due to finiteness of the number of tasks.

The main distinction of this section from the proof of Theorem \ref{thm:lips_tgt} is the difference in the decomposition of error in Lemmas  \ref{lem:sup_to_rad} and \ref{lem:sup_to_rad_2}. While the decomposition in Lemma \ref{lem:sup_to_rad} is more intuitive and similar to a decomposition done in \cite{mtl_maurer}, the decomposition in \ref{lem:sup_to_rad_2} led to an asymptotically tighter analysis.

Before we state the main theorem, we start with an overloaded definition of the empirical expected validation loss which takes $\etr{S}$ as input:
\begin{align}
    \lev(\lambda, \etr{S})  
    &= \frac{1}{T} \sum_t \E{x_v^t,\epsilon_v^t}{l(x_{v}^{t\intercal}\wh^t_{\lambda}, y_{v}^t)}\nonumber\\
    &= \E{\eval{S'}}{\lve(\lambda, \etr{S}\ewtimes \eval{S'})}\label{step:ev_exp}.
\end{align}
Where $y_v^t = f^t(x_v^t,\epsilon_v^t)$.
Thus for a given $\etr{S}$, $\lev$ computes the expectation of the empirical validation loss over all possible sampling of the validation data.
We state the main theorem of this section below.

\begin{theorem}\label{thm:lips}
    Given a loss function that satisfies Assumptions \ref{assum:bound} and \ref{assum:lips} above, the expected validation loss error using the ERM estimator defined in Equation \ref{def:lamb_erm} is bounded with probability $\geq 1-\delta$ as:
    \begin{align}
        l_v(\lambda_{ERM}) - l_v(\lambda^*) &\leq\frac{2ML\Lambda_D^T
        }{\sqrt{T}}\E{x_v}{\|x_v\|} + \frac{2L}{\sqrt{n_v}} \sqrt{\E{x_v}{\|x_v\|^2}}\sqrt{\E{X,y}{ \|y\|^2/V(XX^\intercal)}}\nonumber\\
        &
         + \frac{2L\tilde{M}}{\sqrt{n_v}\sqrt[4]{T}}\sqrt{\E{x_v}{\|x_v\|^2}}\sqrt[4]{\frac{\ln(4/\delta)}{2}}\nonumber + 5C\sqrt{\frac{\ln(16/\delta)}{2T}}.\nonumber
    \end{align}
    Here $M^2 = \max\|Xy\|^2$, $\tilde{M}^2 = \max \|y\|^2/V(XX^\intercal)$, $\Lambda_D^T = \E{}{\max_t 1/V(X^tX^{t\intercal})}$. 
\end{theorem}
\begin{proof}
The proof proceeds similar to the proof for Theorem \ref{thm:lips_tgt}.
We write $l_v(\lambda_{ERM}) - l_v(\lambda^*) = l_v(\lambda_{ERM}) - l_{v}(\lambda_{ERM},S) +  l_{v}(\lambda_{ERM},S) - l_{v}(\lambda^*,S) + l_{v}(\lambda^*,S) - l_v(\lambda^*)$. We note, as usual, that $l_{v}(\lambda_{ERM},S) - l_{v}(\lambda^*,S) \leq 0$ and $l_{v}(\lambda^*,S) - l_v(\lambda^*)$ is bounded by a Hoeffding bound (Theorem \ref{thm:Hoeffding}). Notably, with probability $\geq 1-\delta$,
\begin{align*}
    l_{v}(\lambda^*,S) - l_v(\lambda^*) \leq C\sqrt{\frac{\ln(1/\delta)}{2T}}.
\end{align*}

It remains to bound $l_v(\lambda_{ERM}) - l_{v}(\lambda_{ERM},S) \leq \sup_\lambda l_v(\lambda) - l_{v}(\lambda,S)$. We observe that this is error between the empirical loss, and expected loss over sampling of validation examples and tasks. We break this error into error induced from a finite sampling of validation examples, and error from a finite sampling of tasks in Lemma \ref{lem:sup_to_rad}. We get that with probability $\geq 1-\delta$:
\begin{align}
    \sup_\lambda l_v(\lambda) - &l_{v}(\lambda,S) \leq 2\E{\sigma, (X^t,f^t, \epsilon^t \forall t)}{\sup_\lambda \frac{1}{T} \sum_{t} \sigma^t \E{x_v,\epsilon_v}{l(x_v^\intercal\wh_\lambda^t,y_v)}}\nonumber\\
    &+ 2\E{\sigma, (X_v^t, \epsilon_v^t \forall t)}{\sup_\lambda \frac{1}{n_vT} \sum_{t,i} \sigma^{t(i)}l(X_v^{t(i)\intercal}\wh_\lambda^t,y_v^{t(i)})}+ 2C\sqrt{\frac{2\ln(4/\delta)}{T}}.\label{eq:from_sup_rad}
\end{align}
In Lemma \ref{lem:S_v}, we show that $\E{\epsilon_v}{l(x_v^\intercal\wh_\lambda^t,y_v)}$ is Lipschitz in $\frac{1}{V(X^tX^{t\intercal}) + \lambda}$ with Lipschitz constant $\|X^ty^t\|\|x_v\|$ for fixed $y_v$. We use this Lipschitzness to bound the first term as: 
\begin{align*}
    \Eml{\sigma, (X^t,f^t, \epsilon^t \forall t)}\biggl[\sup_\lambda \frac{1}{T} &\sum_{t} \sigma^t \E{x_v,\epsilon_v}{l(x_v^\intercal\wh_\lambda^t,y_v^t)}\biggr]\leq \frac{ML\Lambda_D^T}{\sqrt{T}}.
\end{align*}
We can use Lipschitzness of the loss function in the second term of Equation \ref{eq:from_sup_rad} to get 
\begin{align*}
    \E{\sigma, (X_v^t, \epsilon_v^t \forall t)}{\sup_\lambda \frac{1}{n_vT} \sum_{t,i} \sigma^{t(i)}l(X_v^{t(i)\intercal}\wh_\lambda^t,y_v^{t(i)})} \leq \frac{L}{n_vT}\E{\sigma, (X_v^t, \epsilon_v^t \forall t)}{\sup_\lambda \sum_{t,i} \sigma^{t(i)}X_v^{t(i)\intercal}\wh_\lambda^t}.
\end{align*}
This term can be viewed as a trace product between validation examples and a matrix of predictions $\wh_\lambda^t$. We use this in Lemma \ref{lem:S_tr} to show that with probability $\geq 1-\delta$:
\begin{align*}
    &\E{\sigma, (X_v^t, \epsilon_v^t \forall t)}{\sup_\lambda \frac{1}{n_vT} \sum_{t,i} \sigma^{t(i)} l(x_v^{t(i)\intercal}\wh_\lambda^t,y_v^{t(i)})} \leq\nonumber\\
       &\frac{L}{\sqrt{n_v}} \sqrt{\E{x_v}{\|x_v\|^2}}\sqrt{\E{X,y}{\|y\|^2/V(XX^\intercal)}} + \frac{L\tilde{M}}{\sqrt[4]{n_v^2T}}\sqrt{\E{x_v}{\|x_v\|^2}}\sqrt[4]{\frac{\ln(1/\delta)}{2}}.
\end{align*}
We can now replace $\delta$ by $\delta/4$ in the three probabilistic bounds above so that the following holds with probability at least $1-\delta$:
\begin{align*}
    l_v(\lambda_{ERM}) - l_v(\lambda^*) &\leq\frac{2ML\Lambda_D^T
        }{\sqrt{T}}\E{x_v}{\|x_v\|} + \frac{2L}{\sqrt{n_v}} \sqrt{\E{x_v}{\|x_v\|^2}}\sqrt{\E{X,y}{\|y\|^2/V(XX^\intercal)}}\nonumber\\
        &
         + \frac{2L\tilde{M}}{\sqrt[4]{n_v^2T}}\sqrt{\E{x_v}{\|x_v\|^2}}\sqrt[4]{\frac{\ln(4/\delta)}{2}}\nonumber\\
       & + 2C\sqrt{\frac{2\ln(16/\delta)}{T}} + C\sqrt{\frac{\ln(4/\delta)}{2T}}.
\end{align*}
To get the desired result, we note that $C\sqrt{\frac{\ln(4/\delta)}{2T}} \leq C/2\sqrt{\frac{2\ln(16/\delta)}{T}}$.
\end{proof}

Below we present and prove the main Lemmas used in the above theorem. We first start by upper-bounding the generalization error in terms of two different Rademacher complexities: Rademacher complexity of validation loss with fixed training data and Rademacher complexity of expected validation loss over choice of validation data. 

\begin{lemma}\label{lem:sup_to_rad}
    Given a bounded validation loss function, that is, given that $l(x_v^\intercal\wh_\lambda(X,y),y_v) \leq C,  \forall x_v,y_v,X,y,\lambda$. For any problem instance $S$ as defined in Equation \ref{def:prob}, with probability at least $1-\delta$, 
    \begin{align}
        \sup_\lambda (l_v(\lambda) - l_{v}(\lambda,S)) &\leq 2\E{\sigma, \etr{S}}{\sup_\lambda \frac{1}{T} \sum_{t} \sigma^t \E{x_v,\epsilon_v}{l(x_v^\intercal\wh_\lambda^t,y_v^t)}}\nonumber\\
        &+ 2\E{\sigma, \eval{S}}{\sup_\lambda \frac{1}{n_vT} \sum_{t,i} \sigma^{t(i)}l(X_v^{t(i)\intercal}\wh_\lambda^t,y_v^{t(i)})}\nonumber\\
        &+ 2C\sqrt{\frac{2\ln(4/\delta)}{T}}.\nonumber
    \end{align}
    Where $y_v^t = f^t(x_v^\intercal,\epsilon_v)$, $y_v^{t(i)} = f^t(X_v^{\intercal t(i)},\epsilon_v^{t(i)})$.
    % Here $M_{ev} = \max_t\E{x_v^t,\epsilon_v^t}{(x_{v}^{tT}\wh^t_{\lambda} - y_{v}^t)^2} \leq C$.
\end{lemma}
\begin{proof}
We begin by breaking the generalization error into error induced from a finite sampling of tasks, and error from a finite sampling of validation examples. This is similar to the approach of \cite{mtl_maurer}, where the authors break the generalization error into error induced from learning a representation, and error from learning given a representation.
    \begin{align}
    \sup_\lambda (l_v(\lambda) - l_{v}(\lambda,S)) = &\sup_\lambda (l_v(\lambda) - l_{ev}(\lambda, \etr{S}) + l_{ev}(\lambda, \etr{S}) - l_{v}(\lambda,S))\nonumber\\
    \leq &\sup_\lambda (l_v(\lambda) - l_{ev}(\lambda, \etr{S})) + \sup_\lambda (l_{ev}(\lambda, \etr{S}) - l_{v}(\lambda,S))\label{eq:sup_decompose}.
\end{align}
Note that $l_v(\lambda)$ is the expected value of $l_{ev}(\lambda, \etr{S})$ over sampling of $\etr{S}$, whereas $l_{ev}(\lambda, \etr{S})$ is the average over $T$ samples of training data. We can use Corollary \ref{cor:mcd} by replacing each $l^i$ by $l_{ev}(\lambda, \etr{S})$ to get that with probability $\geq 1-\delta/2$,
\begin{equation}\label{eq:mcd_tr}
    \sup_\lambda (l_v(\lambda) - \lev(\lambda, \etr{S})) \leq \E{\etr{S}, \etr{S'}}{\sup_\lambda (\lev(\lambda,\etr{S}) - \lev(\lambda, \etr{S'}))} +  C\sqrt{\frac{2\ln(4/\delta)}{T}}.%M_{ev}\sqrt{\frac{2\ln(2/\delta)}{T}},
\end{equation}

Again note that, for a fixed $\etr{S}$, we can view $l_{v}(\lambda,S)$ as an average over $n_vT$ i.i.d.\ samples of the form $x_v,\epsilon_v$, where the $(tn_v+i)^\text{th}$ sample for $t\in[T], i\in[n_v]$ becomes the $i^\text{th}$ validation example for the $t^\text{th}$ task. We can then view $l_{ev}(\lambda, \etr{S})$ as the expected value of $l_{v}(\lambda,S)$ over the sampling of $\eval{S}$. Thus, replacing each $l^i$ in Corollary \ref{cor:mcd} by $l_{v}(\lambda,S)$, we get that with probability $\geq 1-\delta/2$,
\begin{align}\label{eq:mcd_val}
    \sup_\lambda (l_{ev}(\lambda, \etr{S}) - &l_{v}(\lambda,S)) \leq\nonumber\\
    &\E{\eval{S}, \eval{S'}}{\sup_\lambda (\lve(\lambda, \etr{S}\ewtimes\eval{S'}) - \lve(\lambda, \etr{S}\ewtimes\eval{S}))}+ C\sqrt{\frac{2\ln(4/\delta)}{n_vT}}.
\end{align}

In order to upper bound the unknown term in Equation \ref{eq:mcd_tr}, we note that we can arbitrarily swap the $i^{\text{th}}$ training instances between $\etr{S}$ and $\etr{S'}$ without changing the expectation. In fact, we can do this for all $i\in R \subseteq [T]$ for any arbitrary set $R$. This allows us to reduce the term to a Rademacher complexity. We show this below where we denote $y_v^t = f^t(x_v^\intercal,\epsilon_v)$:
\begin{align}
    \Eml{\etr{S}, \etr{S'}}\biggl[\sup_\lambda \lev(\lambda,& \etr{S'}) - \lev(\lambda, \etr{S})\biggr]\nonumber\\
    &=\Eml{\etr{S}, \etr{S'}}\left[\sup_\lambda \frac{1}{T} \sum_t \E{x_v,\epsilon_v}{l(x_v^\intercal\wh_\lambda^{\prime t},y_v^{\prime t})} - \frac{1}{T} \sum_t \E{x_v,\epsilon_v}{l(x_v^\intercal\wh^{t}_\lambda,y_v^t)}\right]\nonumber\\
    &= \Eml{\etr{S}, \etr{S'}}\left[\sup_\lambda \frac{1}{T} \sum_{t\in R} \E{x_v,\epsilon_v}{l(x_v^\intercal\wh_\lambda^t,y_v^t)} + \frac{1}{T} \sum_{t\notin R} \E{x_v,\epsilon_v}{l(x_v^\intercal\whp^{t}_\lambda,y_v^{\prime t})}\right.\nonumber\\
    &\left.- \frac{1}{T} \sum_{t\notin R} \E{x_v,\epsilon_v}{l(x_v^\intercal\wh_\lambda^t,y_v^t)} - \frac{1}{T} \sum_{t\in R} \E{x_v,\epsilon_v}{l(x_v^\intercal\whp^{t}_\lambda,y_v^{\prime t})}\right]\nonumber\\
    &= 2\E{\sigma, \etr{S}}{\sup_\lambda \frac{1}{T} \sum_{t} \sigma^t \E{x_v,\epsilon_v}{l(x_v^\intercal\wh_\lambda^t,y_v^t)}}\label{eq:task_rad}.
\end{align}
In the last step we introduce rademachar variables for each task.

Similarly, in Equation \ref{eq:mcd_val}, we note that we can arbitrarily swap the $(tn_v + i)^{\text{th}}$ validation instances between $\eval{S}$ and $\eval{S'}$ without changing the expectation. In fact, we can do this for all $(t,i)\in R \subseteq [T]\times [n_v]$ for any arbitrary set $R$. This allows us to reduce the term to a Rademacher complexity. We show this below where we denote $y_v^{t(i)} = f^t(X_v^{\intercal t(i)},\epsilon_v^{t(i)})$ and $y_v^{\prime t(i)} = f^t(X_v^{\prime \intercal t(i)},\epsilon_v^{\prime t(i)})$:
\begin{align}
    \Eml{\eval{S}, \eval{S'}}&\biggl[\sup_\lambda(l_v(\lambda, \etr{S}\ewtimes\eval{S}) - l_{v}(\lambda,\etr{S}\ewtimes\eval{S'}))\biggr]\nonumber\\
    &=\Eml{\eval{S}, \eval{S'}}\left[\sup_\lambda \frac{1}{T} \sum_t \frac{1}{n_v} \sum_i l(X_v^{t(i)\intercal}\wh_\lambda^t,y_v^{t(i)}) - \frac{1}{T} \sum_t \frac{1}{n_v} \sum_i l(X_v^{\prime t(i)\intercal}\wh_\lambda^t,y_v^{\prime t(i)})\right]\nonumber\\
    &= \Eml{\eval{S}, \eval{S'}}\left[\sup_\lambda \frac{1}{n_vT} \sum_{(t,i) \notin R} l(X_v^{t(i)\intercal}\wh_\lambda^t,y_v^{t(i)}) + \frac{1}{n_vT} \sum_{(t,i) \in R} l(X_v^{\prime t(i)\intercal}\wh_\lambda^t,y_v^{\prime t(i)})\right.\nonumber\\
    &\left.- \frac{1}{n_vT} \sum_{(t,i) \notin R} l(X_v^{\prime t(i)\intercal}\wh_\lambda^t,y_v^{\prime t(i)}) - \frac{1}{n_vT} \sum_{(t,i) \in R} l(X_v^{t(i)\intercal}\wh_\lambda^t,y_v^{t(i)})\right]\nonumber\\
    &= 2\E{\sigma, \eval{S}}{\sup_\lambda \frac{1}{n_vT} \sum_{t,i} \sigma^{t(i)}l(X_v^{t(i)\intercal}\wh_\lambda^t,y_v^{t(i)})}\label{eq:tr_rad}.
\end{align}
In the last step, we introduce rademachar variables for each value of $(t,i)$.

Since Equations \ref{eq:mcd_tr} and \ref{eq:mcd_val} hold with probability $\geq 1-\delta/2$ each, both equations hold with probability $\geq 1-\delta$ by a union bound. We get the desired result by combining Equations \ref{eq:sup_decompose}, \ref{eq:mcd_tr}, \ref{eq:mcd_val}, \ref{eq:task_rad}, \ref{eq:tr_rad} and further noting that $C\sqrt{\frac{2\ln(4/\delta)}{T}} \geq C\sqrt{\frac{2\ln(4/\delta)}{n_vT}}$.
\end{proof}

In the following we give an upper bound on the expectation with respect to sampling of the training set, of the Rademacher complexity of the expected value of validation error over sampling of validation tasks in terms of the distribution of inputs $x$.

\begin{lemma}\label{lem:S_v}
   Given a validation loss function that satisfies Assumptions \ref{assum:bound} and \ref{assum:lips} given in Section \ref{sec:ridge} and $\etr{S}$ as defined in Equation \ref{def:etr}, the following holds with probability at least $1-\delta$ (where we denote $y_v^t = f^t(x_v,\epsilon_v)$):
   \begin{align}
       \Eml{\sigma, \etr{S}}\biggl[\sup_\lambda \frac{1}{T} \sum_{t}& \sigma^t \E{x_v,\epsilon_v}{l(x_v^\intercal\wh_\lambda^t,y_v^t)}\biggr]\leq \frac{ML\Lambda_D^T}{\sqrt{T}}
   \end{align}
   where $M^2 = \max\|Xy\|^2$ and $\Lambda_D^T = \E{X}{\max_t 1/V(X^tX^{t\intercal})}$. 
\end{lemma}
\begin{proof}
We proceed with the proof much similar to Lemma \ref{lem:S_v_2}. We first note that if $y_v = f(x_v,\epsilon_v)$ for a deterministic function $f$, Lipschitzness of the loss function implies Lipschitzness in expectation over $\epsilon_v$:
\begin{align}
    l(x_v^\intercal\wh_1, y_v) - l(x_v^\intercal\wh_2, y_v) &\leq L|x_v^\intercal\wh_1 - x_v^\intercal\wh_2|\nonumber\\
    \implies \E{\epsilon_v}{l(x_v^\intercal\wh_1, y_v) - l(x_v^\intercal\wh_2, y_v)} &\leq L|x_v^\intercal\wh_1 - x_v^\intercal\wh_2|.\nonumber
\end{align}
Using Lipschitzness (Corollary \ref{thm:rad_lipschitz}):
\begin{align}
    \E{\sigma, \etr{S}}{\sup_\lambda \frac{1}{T} \sum_{t} \sigma^t \E{x_v,\epsilon_v}{l(x_v^\intercal\wh_\lambda^t,y_v^t)}} &\leq \E{\sigma, \etr{S},x_v}{\sup_\lambda \frac{1}{T} \sum_{t} \sigma^t \E{\epsilon_v}{l(x_v^\intercal\wh_\lambda^t,y_v^t)}}\nonumber\\
    &\leq \frac{L}{T}\E{\sigma, \etr{S},x_v}{\sup_\lambda  \sum_{t} \sigma^t x_{v}^{\intercal}\wh^t_{\lambda}}.\label{step:task_rc_l1}
\end{align}
This expression is similar to a one-sample variant of the Rademacher complexity in Lemma \ref{lem:S_tr} as well as that in \cite{trace_reg}. However, we cannot use the techniques used there since that would result in a constant upper bound. We instead use Equation \ref{eq:lamb_lips} and Theorem \ref{thm:rad_lips_general} to conclude that,
\begin{align}
    \frac{L}{T}\E{\sigma, \etr{S},x_v}{\sup_\lambda  \sum_{t} \sigma^t x_{v}^{\intercal}\wh^t_{\lambda}} &\leq \frac{L}{T}\E{x_v}{\|x_v\|\E{\sigma, \etr{S}}{\sup_\lambda  \sum_{t} \sigma^t \frac{\|X^ty^t\|}{V^T+\lambda}}}\nonumber\\
    &= \frac{L}{T}\E{x_v}{\|x_v\|}\E{\sigma, \etr{S}}{\sup_\lambda \sum \sigma^t \frac{\|X^ty^t\|}{V^T+\lambda}}\nonumber\\
    &\leq \frac{L}{T}\E{x_v}{\|x_v\|}\E{\etr{S}}{(\sup_\lambda\frac{1}{V^T + \lambda}) \E{\sigma}{\sum \sigma^t \|X^ty^t\|}}\nonumber\\
    &\leq \frac{L}{T}\E{x_v}{\|x_v\|}\E{\etr{S}}{\frac{\sqrt{\sum \|X^ty^t\|^2}}{V^T}}.\nonumber
\end{align}
We use Khintchine's inequality (Theorem \ref{thm:khintchine}) and set $\lambda = 0$ in the last step. 
To get the desired result, we note from assumption that $M^2 = \max \|Xy\|^2 \implies \sqrt{\sum \|X^ty^t\|^2} \leq M\sqrt{T}$. Thus,
\begin{align}
    \frac{L}{T}\Eml{\sigma, \etr{S},x_v}\biggl[\sup_\lambda & \sum_{t} \sigma^t x_{v}^{\intercal}\wh^t_{\lambda}\biggr]\nonumber\\
    &\leq \frac{L}{T}\E{x_v}{\|x_v\|}\E{\etr{S}}{\frac{\sqrt{\sum \|X^ty^t\|^2}}{V^T}}\nonumber\\
    &\leq \frac{L}{T}\E{x_v}{\|x_v\|}\E{\etr{S}}{\frac{M\sqrt{T}}{V^T}}\nonumber\\
    &\leq \frac{ML\Lambda_D^T}{\sqrt{T}}\E{x_v}{\|x_v\|}\nonumber
\end{align}
\end{proof}

We now show an upper bound on the expected Rademacher complexity of validation loss given fixed training data in terms of the distribution of outputs $y$.

\begin{lemma}\label{lem:S_tr}
    Given a validation loss function that satisfies Assumptions \ref{assum:bound} and \ref{assum:lips} given in Section \ref{sec:ridge}, and $\eval{S}$ as defined in Equation \ref{def:eval}, the following holds with probability at least $1-\delta$ (where we denote $y_v^{t(i)} = f^t(x_v^{t(i)},\epsilon_v^{t(i)})$):
   \begin{align}
       &\E{\sigma, \eval{S}}{\sup_\lambda \frac{1}{n_vT} \sum_{t,i} \sigma^{t(i)} l(x_v^{t(i)\intercal}\wh_\lambda^t,y_v^{t(i)})} \leq\nonumber\\
       &\frac{L}{\sqrt{n_v}} \sqrt{\E{x_v}{\|x_v\|^2}}\sqrt{\E{X,y}{\|y\|^2/V(XX^\intercal)}} + \frac{L\tilde{M}}{\sqrt[4]{n_v^2T}}\sqrt{\E{x_v}{\|x_v\|^2}}\sqrt[4]{\frac{\ln(1/\delta)}{2}}.\nonumber
   \end{align}
Here $\tilde{M}^2 = \max \|y\|^2/V(XX^\intercal)$. 
\end{lemma}
\begin{proof}
We define $\mathcal{R}$ as below and use Lipschitzness to upper bound it as a simpler Rademacher complexity term:
\begin{align}
    \mathcal{R} &= \frac{1}{n_vT} \E{\sigma}{\sup_\lambda \sum_t\sum_i \sigma^{t(i)} l(x_v^{t(i)\intercal}\wh_\lambda^t,y_v^{t(i)})} \nonumber\\
    &\leq \frac{L}{n_vT} \E{\sigma}{\sup_\lambda \sum_t\sum_i \sigma^{t(i)} x^{t(i)\intercal}_{v} \wh^t_\lambda}.\nonumber
\end{align}

To compute the above quantity, we use a manipulation similar to one in \cite{trace_reg}. We define two matrices $X_\sigma \in \R^{T\times d}$ and $W_\lambda \in \R^{d \times T}$: the $t$-th row of $X_\sigma$ is defined as $X_{\sigma_{(t)}} = \sum_i \sigma^{t(i)} x^{t(i)\intercal}_{v}$ and the $t$-th column of $W_\lambda$ is defined as $W_\lambda^{(t)} = \wh^t_\lambda$
. By this definition we see that,
\begin{align}
    \frac{L}{n_vT} \E{\sigma}{\sup_\lambda \sum_t\sum_i \sigma^{t(i)} x^{t(i)\intercal}_{v} \wh^t_\lambda} &=  \frac{L}{n_vT} \E{\sigma}{\sup_\lambda tr(X_\sigma W_\lambda)}\nonumber\\
    \implies \mathcal{R} &\leq \frac{L}{n_vT} \E{\sigma}{\sup_\lambda \|X_\sigma\|_2 \|W_\lambda\|_2}\nonumber\\%\label{eq:rad_holder_tight}\\
     \implies \E{\eval{S}}{\mathcal{R}} &\leq \frac{L}{n_vT} \E{\sigma,x^{t(i)}_{v}}{\|X_\sigma\|_2}{\sup_\lambda\|W_\lambda\|_2}.\label{eq:rad_holder}
\end{align}
Note that $\E{\eval{S}}{\mathcal{R}}$ corresponds to the left hand side in the statement of the Lemma.

\begin{align}
    \|X_\sigma\|_2 &= \sqrt{tr(X_\sigma X_\sigma^\intercal )}\nonumber\\
    &= \sqrt{\sum_t (\sum_i \sigma^{t(i)} x^{t(i)\intercal}_{v})(\sum_j \sigma^{t(j)} x^{t(j)\intercal}_{v})^\intercal} \quad \text{(from definition)}\nonumber\\
    \implies \E{\sigma,x^{t(i)}_{v}}{\|X_\sigma\|_2} &\leq \sqrt{\sum_t \E{}{(\sum_i \sigma^{t(i)} x^{t(i)\intercal}_{v})(\sum_j \sigma^{t(j)} x^{t(j)\intercal}_{v})^\intercal}}\nonumber\\
    &= \sqrt{\sum_t \E{}{\sum_ix^{t(i)\intercal}_{v}x^{t(i)}_{v}}}\nonumber\\
    &= \sqrt{\sum_t \E{}{\sum_i \|x^{t(i)\intercal}_{v}\|^2}}\nonumber\\
    &= \sqrt{n_vT\E{x_v}{\|x_v\|^2}}.\label{eq:X_2}
\end{align}

To compute $\|W_\lambda\|_2$:
\begin{align}
    \|W_\lambda\|_2 &= \sqrt{tr(W_\lambda^\intercal W_\lambda)}\nonumber\\
    &= \sqrt{\sum_t \wh_\lambda^{t\intercal}\wh_\lambda^t} = \sqrt{\sum_t \|\wh_\lambda^t\|^2}\nonumber.
\end{align}

It remains to compute bounds on $\E{}{\sup_\lambda \sqrt{\sum_t \|\wh^t_\lambda\|^2}}$.
Using Hoeffding inequality (Theorem \ref{thm:Hoeffding}), if $\|\wh_\lambda^t\|^2\leq \tilde{M}^2 \forall t,\lambda$, we can say the following with probability $\geq 1-\delta$:
\begin{align}
    \frac{1}{T} \sum_t \|\wh^t_\lambda\|^2 \leq \E{X,y}{\|\wh_\lambda\|^2} + \tilde{M}^2\sqrt{\frac{\ln(1/\delta)}{2T}}.\nonumber
\end{align}
This gives us that with probability $\geq 1-\delta$,
\begin{align}
    \sup_\lambda\|W_\lambda\|_2 &\leq \sup_\lambda \sqrt{T(\E{X,y}{\|\wh_\lambda\|^2}) + \tilde{M}^2\sqrt{\frac{T\ln(1/\delta)}{2}}}\nonumber\\
    &= \sqrt{T}\sqrt{\sup_\lambda (\E{X,y}{\|\wh_\lambda\|^2}) + \tilde{M}^2\sqrt{\frac{\ln(1/\delta)}{2T}}}\nonumber\\
    &\leq \sqrt{T\sup_\lambda (\E{X,y}{\|\wh_\lambda\|^2})} + \tilde{M}\sqrt[4]{\frac{T\ln(1/\delta)}{2}}.\label{eq:w_lamb_sum}
\end{align}

Finally, note that from Equation \ref{eq:w_norm_max},
\begin{align}
    \|\wh_\lambda\|^2 &\leq \|y\|^2/V(XX^\intercal),\nonumber
\end{align}
where we defined $V(.)$ as the smallest non-0 singular value of the matrix. Thus,
\begin{align}
    \sup_\lambda \E{X,y}{\|\wh_\lambda\|^2} &\leq \E{X,y}{ \|y\|^2/V(XX^\intercal)}\label{eq:sup_w_lamb}
\end{align}
and,
\begin{align}
    \max \|\wh_\lambda\|^2 &\leq \max \|y\|^2/V(XX^\intercal).\nonumber
\end{align}
{So that $\tilde{M}^2 = \max\|y\|^2/V(XX^\intercal)$ satisfies $\|\wh_\lambda^t\|^2\leq \tilde{M}^2\forall t,\lambda$}.

Combining Equations \ref{eq:rad_holder}, \ref{eq:X_2}, \ref{eq:w_lamb_sum} and \ref{eq:sup_w_lamb},
\begin{align}
    \E{\eval{S}}{\mathcal{R}} \leq \frac{L}{\sqrt{n_v}}\sqrt{\E{x_v}{\|x_v\|^2}}\sqrt{\E{X,y}{ \|y\|^2/V(XX^\intercal)}} + \frac{L\tilde{M}}{\sqrt[4]{n_v^2T}}\sqrt{\E{x_v}{\|x_v\|^2}}\sqrt[4]{\frac{\ln(1/\delta)}{2}}\nonumber.
\end{align}

\end{proof}

\end{document}